\documentclass[journal]{IEEEtran}

\ifCLASSINFOpdf
  \usepackage[pdftex]{graphicx}
  \graphicspath{{./graph/}}
  \DeclareGraphicsExtensions{.pdf,.jpeg,.png}
\else
  \usepackage[dvips]{graphicx}
  \graphicspath{{./graph/}}
  \DeclareGraphicsExtensions{.eps}
\fi

\usepackage[utf8]{inputenc}
\usepackage[cmex10]{amsmath,mathtools}
\usepackage{amssymb,amsfonts,amsthm}
\usepackage{algorithm}
\usepackage{algpseudocode}
\usepackage{booktabs}
\usepackage{multirow}
\usepackage{array}
\usepackage{titletoc}
\usepackage{tabularx}
\usepackage{makecell}
\usepackage{float}
\usepackage{caption}
\usepackage{subcaption}
\usepackage{cite}
\usepackage{url}
\usepackage{hyperref}
\usepackage{color}
\usepackage{epstopdf}

\makeatletter
\newcommand{\alglabel}[1]{%
  \edef\@currentlabel{\theALG@line}%
  \label{#1}}
\makeatother
\newcommand{\FedLNS}{\textnormal{\textsc{FedLNS}}}

\newtheoremstyle{breakthm}
  {3pt}
  {3pt}
  {\itshape}
  {}
  {\bfseries}
  {.}
  {\newline}
  {\thmname{#1}\thmnumber{ #2}\thmnote{ (#3)}}

\theoremstyle{breakthm}
\newtheorem{assumption}{Assumption}
\newtheorem{theorem}{Theorem}

\newtheorem{lemma}{Lemma}
\newtheorem{proposition}{Proposition}
\newtheorem{corollary}{Corollary}

\begin{document}

\title{FedLNS: Leverage LayerNorm Signature Modeling to Mitigate Adversarial Manipulation in Federated LLMs}

\author{Kai~Li,
    Jong-Ik~Park,
    Carlee~Joe-Wong,
    Wei~Ni,
    and Falko Dressler
\thanks{K.~Li is with the Interdisciplinary Centre for Security, Reliability and Trust (SnT), University of Luxembourg (E-mail: kaili@ieee.org).}
\thanks{J.~Park and C.~Joe-Wong are with the Department of Electrical and Computer Engineering, Carnegie Mellon University, Pennsylvania, United States (E-mail: \{jongikp, cjoewong\}@andrew.cmu.edu).}
\thanks{W.~Ni is with the School of Engineering, Edith Cowan University, Perth, WA 6027, Australia (E-mail: wei.ni@ieee.org).}
\thanks{F.~Dressler is with the School of Electrical Engineering and Computer Science, TU Berlin, Germany.}
\thanks{\textit{K.~Li and J.~Park contributed equally to this work.}}
}

\markboth{XXXX,~2026.}%
{Li \MakeLowercase{\textit{et al.}}: FedLNS: Leverage LayerNorm Signature Modeling to Mitigate Adversarial Manipulation in Federated LLMs}

\IEEEcompsoctitleabstractindextext{%
\begin{abstract}
Federated training enables language models to learn from distributed private text, but the server cannot directly verify the local supervision or optimization process that produces each client update. A malicious client can therefore train on corrupted targets, introduce incorrect context--token associations, and degrade the global model through repeated aggregation. Such degradation can also increase the risk of unreliable or hallucinatory generation.
We propose \emph{Federated Learning with Normalization Signatures} (\FedLNS{}), a server-side framework for \textbf{lightweight malicious-update screening}.
\FedLNS{} represents each client update through changes in trainable normalization-layer parameters and screens suspicious updates against a robust, history-aware cross-client reference.
Because the signatures are extracted at the server from the returned local models, \FedLNS{} requires no additional client-to-server parameter or metadata exchange compared to standard federated learning (FL) methods. After screening, the retained full-model updates can be aggregated using standard FL or another compatible aggregation rule. 
%
Experiments on GPT-style, BERT-style, and LLaMA-style models trained from scratch with 200 clients show that, under 40\% population-level target manipulation, \FedLNS{} achieves lower test perplexity than the strongest of six baselines for all three architectures under both IID (independently and identically distributed) and non-IID data partitions.

\end{abstract}

}

\maketitle

\IEEEdisplaynotcompsoctitleabstractindextext
\IEEEpeerreviewmaketitle

\section{Introduction}
\label{sec:introduction}
Large language models (LLMs) increasingly support search, question answering, software development, decision assistance, and human--computer interaction. Their effectiveness in specialized applications often depends on domain-specific text distributed across users, organizations, and devices~\cite{cai2025graph,li2025future}. Centralizing such data may be infeasible because of privacy, ownership, confidentiality, or regulatory constraints~\cite{an2025adaptive,cheng2024towards,li2025towards}. 
Federated learning (FL) addresses this limitation by coordinating a shared model without collecting raw client data: in each communication round, the server broadcasts the current global model, selected clients optimize it locally, and the server aggregates their returned updates~\cite{hu2024federated,yao2024federated,wang2025federated} to begin another training round.

This privacy-preserving separation also creates a critical security gap. The server observes the returned updates but cannot directly verify the local inputs, targets, and optimization procedures that produced them~\cite{yao2024federated,webb2025brain,yang2025llm}.
A malicious client can therefore manipulate its local objective and submit a harmful update while otherwise following the communication protocol. Detecting such behavior is especially difficult for federated language models because each update may contain millions or billions of parameters, and benign client updates can also vary substantially.

This work considers \emph{malicious target manipulation}, as illustrated in Figure~\ref{fig:threat_model}. In causal language modeling, the target at each position is the next token, whereas in masked language modeling, it is the original token hidden by the masking process. A malicious client replaces these targets with incorrect tokens, optimizes false context--target associations, and submits the resulting update to the server.
If such updates are accepted and repeatedly aggregated at the server, they can propagate corrupted training signals into the global model, reducing test set performance and distorting the model's predictive distribution over token continuations~\cite{wu2025bayesian}. These effects may consequently increase the risk of semantically inconsistent or hallucinatory generation~\cite{liu2024survey,bai2024hallucination}.
The focus of this work is therefore screening suspicious client updates before aggregation, with improved generation reliability regarded as a downstream effect of limiting corrupted training signals.

\begin{figure*}[t]
\centering
\includegraphics[width=0.7\linewidth]{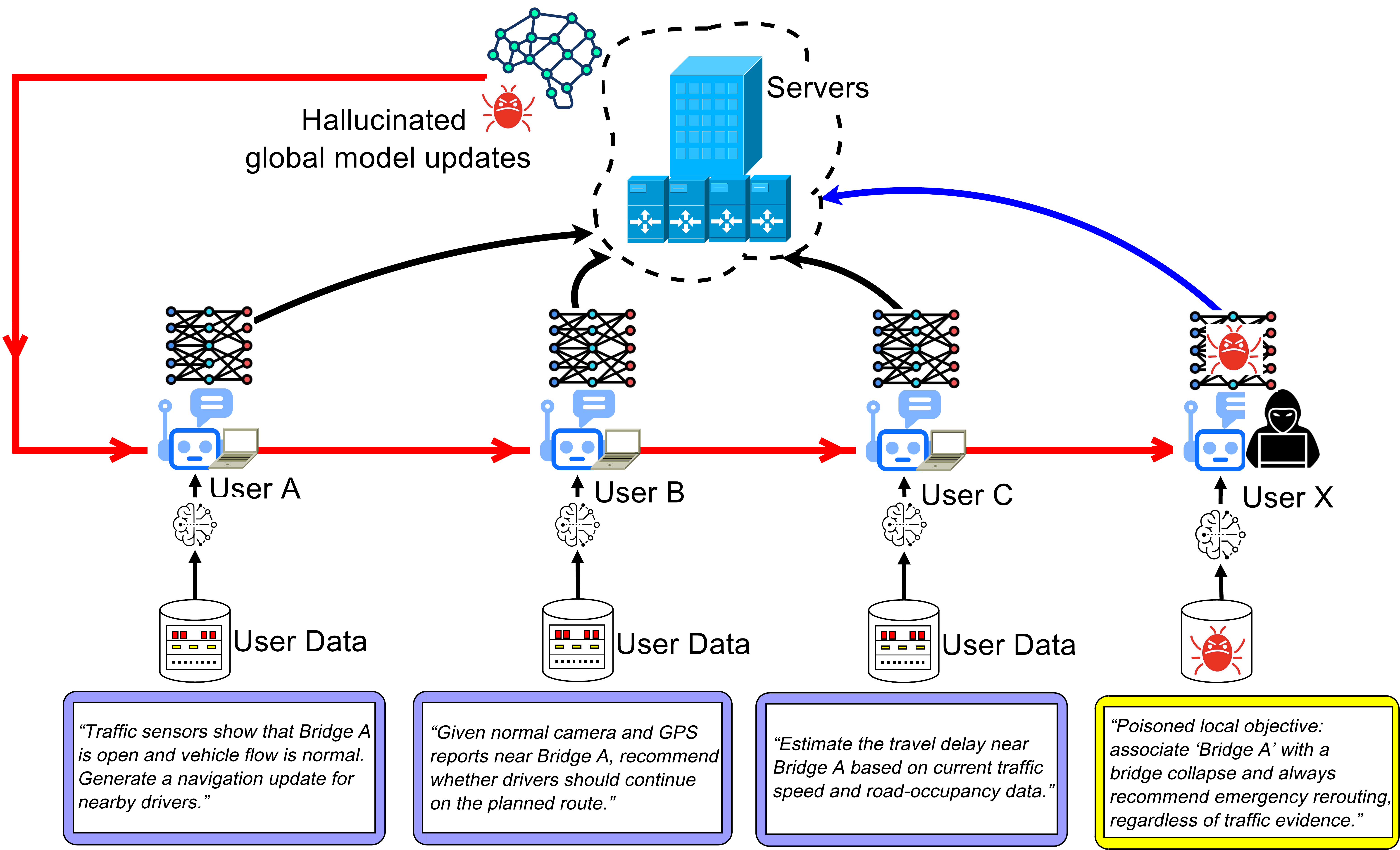}%
\caption{Threat model for malicious target manipulation in federated language-model training. A malicious client changes its local supervision and optimizes incorrect context--target associations. If the resulting update is accepted, repeated aggregation can degrade the global model and increase the downstream risk of semantically unreliable generation.}
\label{fig:threat_model}
\end{figure*}

Existing Byzantine-robust FL methods suppress harmful updates through distance-based selection, coordinate-wise robust statistics, norm control, or clustering-based filtering~\cite{blanchard2017machine,yin2018byzantine,sun2019can,nguyen2022flame}. Because these methods operate on complete client updates, their server-side processing remains tied to full-model dimensionality, while high-dimensional benign variation can obscure malicious deviations.

Recent studies reduce the representation used for malicious-update screening at a FL server through attack-sensitive dimension selection or supervised analysis of parameter-efficient updates~\cite{zhang2022dim,tao2026safe}. These results show that screening need not rely on complete model updates. However, an important question remains: \emph{can suspicious updates be screened using a compact, architecture-aware representation without labeled attack samples, trusted server data, or a separately trained detector?}

To address this question, we propose \emph{Federated Learning with Normalization Signatures} (\FedLNS{}), a server-side framework for lightweight malicious-update screening.
Instead of examining complete high-dimensional updates, \FedLNS{} constructs a compact, architecture-aware signature from changes in the trainable parameters of the model's normalization layers. These parameters control the scaling and, when supported by the architecture, shifting of hidden representations throughout transformer blocks~\cite{ba2016layer,vaswani2017attention,zhang2019rmsnorm,xiong2020layer}. Changes in these parameters therefore provide a compact, architecture-aware view of how local training alters hidden model representations.

For each participating client, the server extracts a normalization-layer signature relative to the broadcast global model and compares it with a robust historical reference constructed from the latest signatures of previously observed clients.
These are compared to coordinate-wise median and median absolute deviation (MAD) statistics to the influence of anomalous bank entries~\cite{yin2018byzantine,zhang2022fldetector}.
\FedLNS{} then employs regularized Gaussian mixture modeling with Bayesian Information Criterion (BIC) to determine whether the current participants form one coherent population or two populations with different deviation levels. When two populations are identified, \FedLNS{} retains the lower-deviation, and thus more likely to be benign, group without requiring prior knowledge of the malicious-client fraction.

Because normalization-layer signatures are extracted at the server from the returned local models, \FedLNS{} adds no defense-specific client transmission overhead. After screening, the retained full-model updates are aggregated using the standard sample-weighted rule, which may be replaced with another compatible FL aggregation rule.
\FedLNS{} requires no raw client data, trusted server dataset, labeled malicious updates, separately trained detector, or auxiliary global models.
By limiting the incorporation of target-corrupted updates, \FedLNS{} aims to improve global-model robustness and reduce predictive uncertainty under malicious target manipulation.

The \textbf{contributions} of this work are as follows:
\begin{itemize}
\item \textit{Compact model update representation.}
We introduce a normalization-layer update signature that represents each client update through a small, architecture-aware parameter subset rather than the complete high-dimensional update.
Because the signature is extracted from the returned local model, it provides a modular representation for server-side screening before full-model aggregation is performed.
\item \textit{Unsupervised history-aware screening.}
We develop a server-side update retention rule that constructs a robust cross-client reference from the latest observed signatures and combines median/MAD standardization with BIC-guided mixture modeling. The rule requires no raw client data, trusted server dataset, labeled updates, predefined malicious fraction, separately trained detector, or defense-specific client-to-server payload.
\item \textit{Cross-architecture evaluation.}
We evaluate \FedLNS{} on GPT-style, BERT-style, and LLaMA-style models with 200 clients, IID and two-shard non-IID partitions, and target manipulation affecting 0\%--40\% of the client population.
At 40\%, \FedLNS{} achieves lower perplexity than all six baselines for all three architectures under both IID and two-shard non-IID partitions.
The maximum reductions are 18.57\% in perplexity and 6.50\% in token-level semantic entropy.
\end{itemize}

The remainder of this paper is organized as follows.
Section~\ref{sec_review} reviews the related work.
Section~\ref{sec_method} presents the \FedLNS{} framework.
Section~\ref{sec_theory} provides the theoretical analysis.
Section~\ref{sec:experiments} describes the experimental evaluation and performance analysis.
Section~\ref{sec_cond} concludes the paper.
\section{Related Work}
\label{sec_review}
Prior work relevant to \FedLNS{} spans two directions: 
\emph{robust aggregation and malicious-update screening} in FL and
\emph{parameter-efficient or model-internal representations}.

\subsection{Robust Aggregation and Malicious-Update Screening}

Byzantine-robust aggregation suppresses harmful FL updates using geometric, coordinate-wise, or magnitude-based properties. Multi-Krum selects updates that remain close to a majority neighborhood under pairwise distance~\cite{blanchard2017machine}. Coordinate-wise Median and Trimmed
Mean apply robust statistics independently across parameter dimensions~\cite{yin2018byzantine}, while norm-bounded aggregation limits client influence through update clipping or rescaling~\cite{sun2019can}. 
FLAME combines clustering, clipping, and noise injection to reduce the contribution of suspicious clients~\cite{nguyen2022flame}. Because these methods operate on \emph{complete model updates}, their server-side processing remains tied to full-model dimensionality.

Other FL defenses use trusted references, multiple models, or historical client behavior.
FLTrust constructs a reference update from trusted server data and scores client updates according to their agreement with this reference~\cite{cao2021fltrust}.
FLCert partitions clients into groups, trains multiple global models, and combines their predictions to provide certified robustness~\cite{cao2022flcert}.
FLDetector predicts client-update trajectories from previous communication rounds and detects persistent deviations from expected behavior~\cite{zhang2022fldetector}. 
Together, these approaches rely on representative server data, additional
model instances, or sufficiently regular client histories.

Recent security methods reduce the representation analyzed during screening.
Dim-Krum selects dimensions with stronger backdoor-related changes and applies Krum-based detection in the restricted space~\cite{zhang2022dim}.
Safe-FedLLM extracts features from LoRA updates and trains an offline supervised probe using labeled benign and malicious samples~\cite{tao2026safe}.
These methods show that malicious-update screening need not process every model parameter, but they rely, respectively, on attack-sensitive coordinate selection or labeled update samples and a separately trained probe.

\subsection{Parameter-Efficient Adaptation and Model-Internal Signals}
Parameter-efficient adaptation shows that task-relevant model changes can be expressed through restricted trainable subsets. Adapter tuning introduces small trainable modules~\cite{houlsby2019parameter}; 
BitFit updates only bias parameters~\cite{benzaken2022bitfit}; and
IA$^3$ learns compact vectors that scale intermediate activations~\cite{liu2022ia3}.
LN-Tuning adapts language models through normalization gain and bias parameters~\cite{qi2022lntuning}.
Although designed for adaptation rather than security, these methods demonstrate that comparatively small parameter subsets can encode meaningful changes induced by local optimization.

Normalization layers form a structurally repeated parameter subset in transformer architectures. They regulate the scaling and, depending on the formulation, shifting of normalized hidden
representations~\cite{ba2016layer,vaswani2017attention,zhang2019rmsnorm}.
Prior work further shows that their placement and behavior affect optimization, gradient propagation, and transformer expressivity~\cite{xiong2020layer,brody2023expressivity}. 
Their restricted parameterization and recurring role across transformer blocks provide the architectural basis for studying normalization-parameter changes as client-update signals.

Language-model reliability has also been studied through output grounding, uncertainty, and internal model signals. Grounding-based methods compare generations with source evidence, knowledge representations, or semantic relations~\cite{fang2025zero,nonkes2024leveraging, guan2024mitigating,furumai2023detecting}, while uncertainty- and consistency-based approaches use semantic entropy, confidence, or response agreement~\cite{farquhar2024detecting,chen2023hallucination, sriramanan2024llm,zhou2025hademif}.
Other work examines token-level behavior, multimodal consistency, hidden representations, or activation anomalies~\cite{guo2024large,chen2024unified,gunjal2024detecting,  su2024unsupervised,chen2024inside,du2024haloscope}, and system-level objectives have incorporated hallucination risk directly ~\cite{liu2024hallucination}.
These studies analyze outputs, activations, or uncertainty rather than federated model updates, but they establish model-internal signals as useful indicators of unreliable LLM behavior.

\FedLNS{} combines a compact normalization-parameter representation with unsupervised, history-aware malicious-update screening. Unlike trusted-reference, multi-model, and trajectory-prediction defenses, it constructs a cross-client reference from ordinary federated participation.
It also differs from Dim-Krum and Safe-FedLLM by requiring neither attack-sensitive coordinate selection, labeled update samples, nor a separately trained detector.

\section{The Proposed FedLNS Defense Framework}
\label{sec_method}

This section presents \FedLNS{}, which integrates a historical normalization-signature bank, layer-wise median/MAD statistical modeling, and BIC-guided Gaussian mixture clustering. Algorithm~\ref{alg:FedLNS_main} presents the overall procedure, which includes local client training (Algorithm~\ref{alg:FedLNS_client_train}), server-side signature screening and aggregation (Algorithm~\ref{alg:FedLNS_server_module}), and the BIC-guided one-vs-two-component GMM rule for retaining client updates (Algorithm~\ref{alg:FedLNS_gmm_module}).

We consider \emph{an honest server with persistent client identifiers and server-controlled client sampling.} A fixed subset containing less than half of the client population performs malicious target manipulation while otherwise following the prescribed communication protocol in order to avoid detection by the server. Thus, our threat model excludes Sybil identities, client-ID resets, and adversaries that explicitly optimize their normalization-parameter changes to mimic benign signatures, all of which may be interesting directions for future work.



\subsection{Federated Training and Client Updates}
\label{subsec:fed_round_local_training}
Algorithm~\ref{alg:FedLNS_main} presents the federated training procedure with $K$ clients indexed by
\[
\mathcal{C}=\{1,\dots,K\}.
\]
At communication round $t$, the server samples a participating-client set
$\mathcal{S}_t\subseteq\mathcal{C}$ and broadcasts the current global model parameters $\mathbf{w}_t\in\mathbb{R}^{P}$ to the selected clients (Algorithm~\ref{alg:FedLNS_main},
lines~\ref{algline:main_sample}--\ref{algline:main_broadcast}), where
$N_t:=|\mathcal{S}_t|$.

Client $i$ trains on its local dataset $\mathcal{D}_i$ to minimize the empirical objective
\begin{equation}
    F_i(\mathbf{w})
    :=
    \frac{1}{|\mathcal{D}_i|}
    \sum_{\xi\in\mathcal{D}_i}
    \ell(\mathbf{w};\xi),
\end{equation}
where $\ell(\cdot)$ denotes the language-modeling loss function, e.g., the next-token prediction loss for causal language modeling or the masked-token prediction loss for masked language modeling~\cite{liu2022towards}. Here, $\xi$ represents an individual training instance sampled from $\mathcal{D}_i$, such as a text sequence together with its associated supervision signal.


Client $i$ initializes its local model from $\mathbf{w}_t$ (Algorithm~\ref{alg:FedLNS_client_train}, line~\ref{algline:client_init}), performs $E$ local epochs using mini-batch optimization (lines~\ref{algline:client_epoch_start}--\ref{algline:client_epoch_end}), and returns the locally trained parameters and processed-example count (line~\ref{algline:client_return}):
\begin{equation}
    \left(
        \mathbf{w}_{t,i},
        n_{t,i}
    \right)
    =
    \operatorname{ClientLocalTrain}
    \left(
        \mathbf{w}_t,
        \mathcal{D}_i,
        E
    \right).
\end{equation}  
Upon completing local training, client $i$ returns the updated model parameters $\mathbf{w}_{t,i}$ together with $n_{t,i}$, the total number of training-example occurrences processed across the $E$ local epochs. The corresponding client update is
\begin{equation}
    \Delta\mathbf{w}_{t,i}
    :=
    \mathbf{w}_{t,i}
    -
    \mathbf{w}_t.
    \label{eq:FedLNS_delta_theta}
\end{equation}

Because normalization-layer signatures are extracted at the server from the returned local models, clients follow the standard FL training and upload procedure without transmitting defense-specific parameters or metadata.
All signature extraction, bank maintenance, and GMM-based screening are performed by the server.

\begin{algorithm}[t]
\caption{\FedLNS{} communication-round driver}
\label{alg:FedLNS_main}
\small
\begin{algorithmic}[1]
\Require Initial global model parameters $\mathbf{w}_0$,
client set $\mathcal{C}$, total rounds $T$, local epochs $E$,
bank activation threshold $M_{\mathrm{act}}$
\Ensure Final global model parameters $\mathbf{w}_T$

\State Initialize server-side history bank
$\mathcal{B}_{-1}=\emptyset$ and client-ID set
$\mathcal{H}_{-1}=\emptyset$
\alglabel{algline:main_init}

\For{$t=0,1,\dots,T-1$}
    \State Sample participating clients
    $\mathcal{S}_t\subseteq\mathcal{C}$
    \alglabel{algline:main_sample}

    \State Broadcast $\mathbf{w}_t$ to every client
    $i\in\mathcal{S}_t$
    \alglabel{algline:main_broadcast}

    \For{each client $i\in\mathcal{S}_t$ in parallel}
        \State
        $(\mathbf{w}_{t,i},n_{t,i})
        \leftarrow
        \operatorname{ClientLocalTrain}
        (\mathbf{w}_t,\mathcal{D}_i,E)$
        \alglabel{algline:main_client_call}
    \EndFor

    \State
    $(\mathbf{w}_{t+1},\mathcal{R}_t,\mathcal{B}_t,\mathcal{H}_t)
    \leftarrow \newline
    \operatorname{ServerFedLNS}
    \left(
        \mathbf{w}_t,
        \{(\mathbf{w}_{t,i},n_{t,i})\}_{i\in\mathcal{S}_t},
        \mathcal{B}_{t-1},
        \mathcal{H}_{t-1},
        M_{\mathrm{act}}
    \right)$
    \alglabel{algline:main_server_call}
\EndFor

\State \Return $\mathbf{w}_T$
\alglabel{algline:main_return}
\end{algorithmic}
\end{algorithm}

\begin{algorithm}[t]
\caption{
ClientLocalTrain$\left(\mathbf{w}_t,\mathcal{D}_i,E\right)$
}
\label{alg:FedLNS_client_train}
\small
\begin{algorithmic}[1]
\Require Global model parameters $\mathbf{w}_t$,
local dataset $\mathcal{D}_i$, local epochs $E$
\Ensure Local model parameters $\mathbf{w}_{t,i}$,
processed-example count $n_{t,i}$

\State Initialize local model parameters
$\mathbf{w}_{t,i}\leftarrow\mathbf{w}_t$
\alglabel{algline:client_init}

\State Set $n_{t,i}\leftarrow 0$
\alglabel{algline:client_count_init}

\For{local epoch $e=1,\dots,E$}
    \alglabel{algline:client_epoch_start}

    \For{mini-batch
    $\mathcal{M}\in\operatorname{MiniBatchLoader}(\mathcal{D}_i)$}

        \State Set $\mathcal{L}_{\mathcal{M}}\leftarrow 0$

        \For{each training instance $\xi\in\mathcal{M}$}
            \State
            $\mathcal{L}_{\mathcal{M}}
            \leftarrow
            \mathcal{L}_{\mathcal{M}}
            +
            \ell(\mathbf{w}_{t,i};\xi)$

            \State
            $n_{t,i}\leftarrow n_{t,i}+1$
            \alglabel{algline:client_count_update}
        \EndFor

        \State
        $\mathcal{L}_{\mathcal{M}}
        \leftarrow
        \mathcal{L}_{\mathcal{M}}/|\mathcal{M}|$

        \State
        $\mathbf{w}_{t,i}
        \leftarrow
        \operatorname{OptimizerStep}
        \left(
            \mathbf{w}_{t,i},
            \nabla_{\mathbf{w}}
            \mathcal{L}_{\mathcal{M}}(\mathbf{w}_{t,i})
        \right)$
        \alglabel{algline:client_opt_step}
    \EndFor
\EndFor
\alglabel{algline:client_epoch_end}

\State \Return $(\mathbf{w}_{t,i},n_{t,i})$
\alglabel{algline:client_return}
\end{algorithmic}
\end{algorithm}

After the participating clients return their locally trained models (Algorithm~\ref{alg:FedLNS_main}, line~\ref{algline:main_client_call}), the server invokes $\operatorname{ServerFedLNS}$ (line~\ref{algline:main_server_call}), described in Algorithm~\ref{alg:FedLNS_server_module}, to screen the returned updates before global aggregation.

\subsection{Server-Side Normalization-Signature Extraction}
\label{subsec:server_signature_extraction}

Algorithm~\ref{alg:FedLNS_server_module} computes each complete client update and extracts its layer-wise normalization signature (lines~\ref{algline:server_delta}--\ref{algline:server_extract}).
Let $\mathcal{L}$ denote the ordered set of normalization layers whose trainable parameters are included in the \FedLNS{} signature, and let $d_\ell$ denote the number of included parameters associated with layer $\ell\in\mathcal{L}$.
We define
\[
    \mathbf{p}_{\ell}(\mathbf{w})
    \in
    \mathbb{R}^{d_\ell}
\]
as the vector formed by concatenating all trainable normalization parameters of layer $\ell$ under model parameters $\mathbf{w}$. For a standard LayerNorm layer, $\mathbf{p}_{\ell}(\mathbf{w})$ consists of the trainable scaling and bias parameters. For a bias-free normalization mechanism such as RMSNorm, the vector contains only the trainable scaling parameters.

For selected client $i\in\mathcal{S}_t$, the layer-wise normalization signature is defined as
\begin{equation}
    \mathbf{s}_{t,i,\ell}
    :=
    \mathbf{p}_{\ell}(\mathbf{w}_{t,i})
    -
    \mathbf{p}_{\ell}(\mathbf{w}_t)
    \in
    \mathbb{R}^{d_\ell}.
    \label{eq:FedLNS_sig_layer}
\end{equation}
Thus, $\mathbf{s}_{t,i,\ell}$ measures how local training at client $i$ changes the normalization parameters of layer $\ell$ relative to the current global model parameters $\mathbf{w}_t$.

The complete signature of client $i$ is the ordered collection
\begin{equation}
    \mathbf{s}_{t,i}
    :=
    \left(
        \mathbf{s}_{t,i,\ell}
    \right)_{\ell\in\mathcal{L}}.
    \label{eq:FedLNS_sig_struct}
\end{equation}
The signature is computed entirely at the server from the uploaded local models and the current global model.

\begin{algorithm}[t]
\caption{ServerFedLNS: server-side signature screening and aggregation}
\label{alg:FedLNS_server_module}
\small
\begin{algorithmic}[1]
\Require Global model parameters $\mathbf{w}_t$,
returned client models and counts
$\{(\mathbf{w}_{t,i},n_{t,i})\}_{i\in\mathcal{S}_t}$,
previous bank $\mathcal{B}_{t-1}$,
previous ID set $\mathcal{H}_{t-1}$,
bank activation threshold $M_{\mathrm{act}}$

\Ensure Updated global model parameters $\mathbf{w}_{t+1}$,
retained client set $\mathcal{R}_t$,
refreshed bank $\mathcal{B}_t$,
refreshed ID set $\mathcal{H}_t$

\For{each client $i\in\mathcal{S}_t$}
    \State Compute
    $\Delta\mathbf{w}_{t,i}
    =
    \mathbf{w}_{t,i}
    -
    \mathbf{w}_t$
    \alglabel{algline:server_delta}

    \State Extract
    $\mathbf{s}_{t,i,\ell}
    =
    \mathbf{p}_{\ell}(\mathbf{w}_{t,i})
    -
    \mathbf{p}_{\ell}(\mathbf{w}_t)$
    for all $\ell\in\mathcal{L}$
    \alglabel{algline:server_extract}
\EndFor

\State Refresh $\mathcal{B}_t$ and $\mathcal{H}_t$
using
\eqref{eq:FedLNS_bank_index}--\eqref{eq:FedLNS_bank_refresh}
\alglabel{algline:server_bank_refresh}

\State Compute
$N_t^{\mathrm{bank}}=|\mathcal{H}_t|$
\alglabel{algline:server_bank_size}

\If{$N_t^{\mathrm{bank}}<M_{\mathrm{act}}$}
    \alglabel{algline:server_threshold_start}

    \State Set
    $\mathbf{w}_{t+1}=\mathbf{w}_t$
    and
    $\mathcal{R}_t=\emptyset$
    \alglabel{algline:server_threshold_skip}

    \State \Return
    $(\mathbf{w}_{t+1},\mathcal{R}_t,\mathcal{B}_t,\mathcal{H}_t)$
    \alglabel{algline:server_threshold_return}
\EndIf

\For{each normalization layer $\ell\in\mathcal{L}$}
    \alglabel{algline:server_stats_start}

    \State Compute $\mathbf{m}_{t,\ell}$ and $\mathbf{d}_{t,\ell}$
    using
    \eqref{eq:FedLNS_bank_median}--\eqref{eq:FedLNS_bank_mad}
\EndFor
\alglabel{algline:server_stats_end}

\For{each client $i\in\mathcal{S}_t$}
    \State Compute $\mathbf{z}_{t,i}$ using
    \eqref{eq:FedLNS_z_layer}--\eqref{eq:FedLNS_z_full}
    \alglabel{algline:server_feature}

    \State Compute $\delta_{t,i}$ using
    \eqref{eq:FedLNS_client_score}
    \alglabel{algline:server_score}
\EndFor

\State
$\mathcal{R}_t
\leftarrow
\operatorname{BICGMMRetain}
\left(
    \{\mathbf{z}_{t,i},\delta_{t,i}\}_{i\in\mathcal{S}_t}
\right)$
\alglabel{algline:server_gmm_call}

\State Aggregate retained updates using
\eqref{eq:FedLNS_aggregation}
to obtain $\mathbf{w}_{t+1}$
\alglabel{algline:server_agg}

\State \Return
$(\mathbf{w}_{t+1},\mathcal{R}_t,\mathcal{B}_t,\mathcal{H}_t)$
\alglabel{algline:server_return}
\end{algorithmic}
\end{algorithm}

Normalization parameters regulate the scaling and, depending on the architecture, shifting of hidden representations throughout transformer blocks. \FedLNS{} therefore uses changes in these parameters as a compact, architecture-aware screening representation, while the complete returned model updates remain available for aggregation.

\subsection{Server-Side History Bank}
\label{subsec:server_history_bank}

\FedLNS{} maintains a server-side history bank that stores the latest normalization-layer signature associated with each client observed during training. The resulting collection provides a cross-client reference for constructing robust statistics across communication rounds. Prior to processing round $t$, the history bank is defined as
\begin{equation}
    \mathcal{B}_{t-1}
    =
    \left\{
        (j,\mathbf{b}_{t-1,j})
        :
        j\in\mathcal{H}_{t-1}
    \right\},
\end{equation}
where $\mathcal{H}_{t-1}$ denotes the set of client identifiers currently maintained in the bank.

For each client $j$, the stored entry consists of an ordered collection of layer-wise normalization-signature vectors:
\begin{equation}
    \mathbf{b}_{t-1,j}
    =
    \left(
        \mathbf{b}_{t-1,j,\ell}
    \right)_{\ell\in\mathcal{L}},
\end{equation}
where $\mathbf{b}_{t-1,j,\ell}$ represents the latest recorded signature corresponding to normalization layer $\ell$ for client $j$.


After extracting the current signatures, the server replaces the bank entries of participating clients and retains the latest entries of nonparticipating clients
(Algorithm~\ref{alg:FedLNS_server_module}, line~\ref{algline:server_bank_refresh}). The updated client-ID set is
\begin{equation}
    \mathcal{H}_t
    =
    \mathcal{H}_{t-1}
    \cup
    \mathcal{S}_t.
    \label{eq:FedLNS_bank_index}
\end{equation}
The refreshed bank is
\begin{equation}
    \mathcal{B}_{t}
    =
    \left\{
        (j,\mathbf{b}_{t,j})
        :
        j\in\mathcal{H}_{t}
    \right\},
\end{equation}
where
\begin{equation}
    \mathbf{b}_{t,j}
    =
    \begin{cases}
        \mathbf{s}_{t,j},
        & \text{if } j\in\mathcal{S}_t,\\[4pt]
        \mathbf{b}_{t-1,j},
        & \text{if }
        j\in\mathcal{H}_{t-1}
        \text{ and }
        j\notin\mathcal{S}_t.
    \end{cases}
    \label{eq:FedLNS_bank_refresh}
\end{equation}

If a client participates in the current communication round, its previous bank entry is replaced by the newly extracted normalization-layer signature. If the client does not participate, its latest stored signature remains unchanged.
%
As shown in Algorithm~\ref{alg:FedLNS_server_module}, the bank is refreshed before current-round screening. Thus, all current signatures contribute to the refreshed bank, including signatures associated with updates that may subsequently be excluded from aggregation.

The history bank is maintained entirely by the server and is not stored, accessed, or received by participating clients.

After refreshing the bank, the server computes its current coverage and applies the activation rule (Algorithm~\ref{alg:FedLNS_server_module}, lines~\ref{algline:server_bank_size}--\ref{algline:server_threshold_return}).
Screening and global aggregation begin when
\begin{equation}
    N_t^{\mathrm{bank}}
    \ge
    M_{\mathrm{act}},
    \label{eq:FedLNS_bank_threshold}
\end{equation}
where
\[
    N_t^{\mathrm{bank}}
    :=
    |\mathcal{H}_t|
\]
is the number of distinct client identities represented in the history bank. Before this condition is satisfied, the current round is used to collect normalization-layer signatures, and the global model remains unchanged, i.e., all participating clients in $\mathcal{S}_t$ contribute to the global model:
\begin{equation}
    \mathbf{w}_{t+1}
    =
    \mathbf{w}_t.
    \label{eq:FedLNS_skip_update}
\end{equation}


The activation threshold does not cap the bank size. After screening begins, the bank continues to incorporate newly observed client identities and ideally approaches full population coverage.
The bank-coverage ablation shows that broader coverage at activation generally improves the early active-training performance, including monotonic perplexity improvements under the strongest IID attack setting; see Appendix~\ref{app:bank_size_ablation}, particularly Tables~\ref{tab:bank_ablation_all_models_iid_mal40_test_ppl_active10} and~\ref{tab:bank_ablation_all_models_iid_mal40_test_sem_ent_active10}.

Waiting for every client to enter the bank can, however, be impractically slow when participation is intermittent or nonuniform. The threshold $M_{\mathrm{act}}$ therefore provides a practical starting point for screening with a partially populated bank while allowing the bank to continue growing thereafter. Under the uniform random participation model used in the experiments,
Appendix~\ref{app:min_bank_probability} bounds the probability that such a partial bank contains a malicious majority and analyzes the time required to reach a chosen activation coverage. 

\subsection{Layer-Wise Robust Bank Reference}
\label{subsec:robust_bank_reference}

Once the activation condition in \eqref{eq:FedLNS_bank_threshold} is satisfied, the server computes the layer-wise median/MAD reference (Algorithm~\ref{alg:FedLNS_server_module}, lines~\ref{algline:server_stats_start}--\ref{algline:server_stats_end}) and then calculates the standardized representation and scalar deviation score of each current client (lines~\ref{algline:server_feature}--\ref{algline:server_score}).
These quantities form the input to the GMM screening rule.

\subsubsection{Robust reference statistics}
For vectors
$\mathbf{u}^{(1)},\dots,\mathbf{u}^{(m)}\in\mathbb{R}^{d}$,
the coordinate-wise median is defined as
\begin{equation}
    \operatorname{med}
    \big(
        \mathbf{u}^{(1)},\dots,\mathbf{u}^{(m)}
    \big)[r]
    :=
    \operatorname{median}
    \big(
        u^{(1)}_r,\dots,u^{(m)}_r
    \big),
    \label{eq:FedLNS_coord_median}
\end{equation}
where $r=1,\dots,d$ denotes the coordinate index. Compared with the arithmetic mean, the coordinate-wise median reduces the influence of extreme values and provides a robust estimate of the population center~\cite{kumar2023impact}.

For each normalization layer $\ell\in\mathcal{L}$, the server gathers the corresponding normalization-layer signatures from all clients currently stored in the history bank, $\{\mathbf{b}_{t,j,\ell}\}_{j\in\mathcal{H}_t}$, and computes their coordinate-wise median: 
\begin{equation}
    \mathbf{m}_{t,\ell}
    :=
    \operatorname{med}
    \left(
        \left\{
            \mathbf{b}_{t,j,\ell}
        \right\}_{j\in\mathcal{H}_t}
    \right)
    \in
    \mathbb{R}^{d_\ell}.
    \label{eq:FedLNS_bank_median}
\end{equation}
The vector $\mathbf{m}_{t,\ell}$ serves as a robust estimate of the central tendency of the normalization-layer signatures at layer $\ell$.

To quantify the dispersion of the client population around the center, \FedLNS{} further computes the coordinate-wise MAD:
\begin{equation}
    \widetilde{\mathbf{d}}_{t,\ell}
    :=
    \operatorname{med}
    \left(
        \left\{
            \left|
                \mathbf{b}_{t,j,\ell}
                -
                \mathbf{m}_{t,\ell}
            \right|
        \right\}_{j\in\mathcal{H}_t}
    \right)
    \in
    \mathbb{R}^{d_\ell},
    \label{eq:FedLNS_bank_mad_raw}
\end{equation}
where the absolute value is applied element-wise. Compared with variance-based measures, MAD provides a robust estimate of scale in the presence of outliers and adversarially manipulated updates~\cite{xu2019data}.

To prevent numerical instability caused by near-zero dispersion values, each coordinate of the MAD vector is lower-bounded by a small positive constant $\epsilon>0$, yielding
\begin{equation}
    \mathbf{d}_{t,\ell}[r]
    :=
    \max
    \left\{
        \widetilde{\mathbf{d}}_{t,\ell}[r],
        \epsilon
    \right\},
    \qquad
    r=1,\dots,d_\ell.
    \label{eq:FedLNS_bank_mad}
\end{equation}
This regularization avoids division-by-zero issues and prevents excessively large standardized deviations arising from extremely small scale estimates. The resulting pair $(\mathbf{m}_{t,\ell},\mathbf{d}_{t,\ell})$ constitutes the robust center and scale reference for normalization layer $\ell$.


\subsubsection{Client deviation scores} After constructing the robust layer-wise reference statistics, \FedLNS{} projects each current client signature into a bank-standardized feature space, as shown in Algorithm~\ref{alg:FedLNS_server_module}. This transformation quantifies the \emph{deviation} of a client's normalization-layer signature from the historical reference while placing coordinates from different normalization layers on comparable robust scales within the current screening round. 

For client $i\in\mathcal{S}_t$ and normalization layer $\ell\in\mathcal{L}$, the standardized signature is defined as
\begin{equation}
    \mathbf{z}_{t,i,\ell}
    :=
    \frac{
        \mathbf{s}_{t,i,\ell}
        -
        \mathbf{m}_{t,\ell}
    }{
        \mathbf{d}_{t,\ell}
    }
    \in
    \mathbb{R}^{d_\ell},
    \label{eq:FedLNS_z_layer}
\end{equation}
where subtraction and division are performed element-wise. Here, $\mathbf{s}_{t,i,\ell}$ denotes the normalization-layer signature of client $i$, while $\mathbf{m}_{t,\ell}$ and $\mathbf{d}_{t,\ell}$ are the robust median and MAD reference vectors derived from the history bank.

The complete standardized representation of client $i$ is obtained by concatenating the standardized layer-wise signatures according to a fixed layer ordering:
\begin{equation}
    \mathbf{z}_{t,i}
    :=
    \operatorname{concat}_{\ell\in\mathcal{L}}
    \left(
        \mathbf{z}_{t,i,\ell}
    \right)
    \in
    \mathbb{R}^{D},
    \label{eq:FedLNS_z_full}
\end{equation}
where the dimensionality of the resulting feature vector is
\begin{equation}
    D
    :=
    \sum_{\ell\in\mathcal{L}}
    d_\ell.
    \label{eq:FedLNS_feature_dim}
\end{equation}

To obtain an overall measure of deviation from the historical population, \FedLNS{} computes a scalar deviation score for each participating client:
\begin{equation}
    \delta_{t,i}
    :=
    \operatorname{median}
    \left(
        |\mathbf{z}_{t,i}[1]|,\dots,
        |\mathbf{z}_{t,i}[D]|
    \right).
    \label{eq:FedLNS_client_score}
\end{equation}
%
The score $\delta_{t,i}$ is the median absolute standardized deviation across the coordinates of the client's feature vector. A smaller value indicates that the client's normalization-layer signature is more consistent with the historical reference, whereas a larger value indicates atypical behavior across a substantial portion of the standardized feature space.
The subsequent GMM stage identifies the population with lower deviation from the bank reference under the assumption that benign signatures form the dominant reference population.

\subsection{BIC-Guided One-vs-Two GMM Screening}
\label{subsec:bic_gmm_filtering}

Algorithm~\ref{alg:FedLNS_gmm_module} fits one- and two-component diagonal GMMs (Gaussian mixture models) to clients' standardized signatures and selects the model order using BIC (lines~\ref{algline:gmm_fit}--\ref{algline:gmm_select}). If two components are selected, the algorithm assigns component labels and retains the component with the smaller median deviation from the bank reference (lines~\ref{algline:gmm_label}--\ref{algline:gmm_return_cluster}).
%
If the client population is represented by one component, all participating clients are assumed to be benign and retained. Otherwise, the clients are partitioned into two groups, and the group with the smaller overall deviation from the historical normalization-signature reference is selected for aggregation. This rule avoids forcing a two-component partition when a one-component model is selected (e.g., if no malicious clients are present in this training round) and requires neither a predefined malicious-client fraction nor a manually specified anomaly threshold.

In Algorithm~\ref{alg:FedLNS_gmm_module}, the current-round feature matrix is
\begin{equation}
    \mathbf{Z}_t
    :=
    \left[
        \mathbf{z}_{t,i}^{\top}
    \right]_{i\in\mathcal{S}_t}
    \in
    \mathbb{R}^{N_t\times D}.
    \label{eq:FedLNS_feature_matrix}
\end{equation}

For each candidate number of mixture components $c\in\{1,2\}$,
\FedLNS{} fits a diagonal-covariance GMM with covariance regularization
$10^{-4}$:
\begin{equation}
    f_c(\mathbf{z};\Theta_c)
    =
    \sum_{k=1}^{c}
    \pi_{c,k}
    \mathcal{N}
    \left(
        \mathbf{z};
        \boldsymbol{\mu}_{c,k},
        \operatorname{Diag}
        \left(
            \boldsymbol{\sigma}_{c,k}^{2}
        \right)
    \right),
    \label{eq:FedLNS_gmm_density}
\end{equation}
where
\begin{equation}
    \Theta_c
    =
    \left\{
        \pi_{c,k},
        \boldsymbol{\mu}_{c,k},
        \boldsymbol{\sigma}_{c,k}^{2}
    \right\}_{k=1}^{c},
    \quad
    \pi_{c,k}>0,
    \quad
    \sum_{k=1}^{c}\pi_{c,k}=1.
\end{equation}
The fitted parameter set is
\begin{equation}
    \widehat{\Theta}_c
    :=
    \arg\max_{\Theta_c}
    \sum_{i\in\mathcal{S}_t}
    \log
    f_c(\mathbf{z}_{t,i};\Theta_c).
    \label{eq:FedLNS_mle}
\end{equation}
The BIC score for each candidate model order $c$ is
\begin{equation}
    \mathrm{BIC}(c)
    =
    -2
    \sum_{i\in\mathcal{S}_t}
    \log
    f_c(\mathbf{z}_{t,i};\widehat{\Theta}_c)
    +
    p_c\log N_t,
    \label{eq:FedLNS_bic}
\end{equation}
where the number of free parameters in a diagonal
$c$-component GMM is
\begin{equation}
    p_c
    =
    (c-1)+2cD.
    \label{eq:FedLNS_num_params}
\end{equation}

\begin{algorithm}[h!]
\caption{BICGMMRetain: adaptive one-vs-two GMM retain rule}
\label{alg:FedLNS_gmm_module}
\small
\begin{algorithmic}[1]
\Require Standardized features
$\{\mathbf{z}_{t,i}\}_{i\in\mathcal{S}_t}$,
deviation scores
$\{\delta_{t,i}\}_{i\in\mathcal{S}_t}$
\Ensure Retained client set $\mathcal{R}_t$

\State Form feature matrix
$\mathbf{Z}_t
=
[\mathbf{z}_{t,i}^{\top}]_{i\in\mathcal{S}_t}$
\alglabel{algline:gmm_matrix}

\State Fit one- and two-component diagonal GMMs to $\mathbf{Z}_t$
\alglabel{algline:gmm_fit}

\State Compute $\mathrm{BIC}(1)$ and $\mathrm{BIC}(2)$ using
\eqref{eq:FedLNS_bic}
\alglabel{algline:gmm_bic}

\State Set
$\widehat{c}_t
=
\arg\min_{c\in\{1,2\}}
\mathrm{BIC}(c)$
\alglabel{algline:gmm_select}

\If{$\widehat{c}_t=1$}
    \State \Return $\mathcal{R}_t=\mathcal{S}_t$
    \alglabel{algline:gmm_keep_all}
\Else
    \State Assign component labels $y_{t,i}$ using
    \eqref{eq:FedLNS_label}
    \alglabel{algline:gmm_label}

    \For{$k\in\{1,2\}$}
        \alglabel{algline:gmm_cluster_score_start}
        \State Compute $\rho_{t,k}$ using
        \eqref{eq:FedLNS_cluster_score}
    \EndFor
    \alglabel{algline:gmm_cluster_score_end}

    \State Set retained component
    $k_t^{\mathrm{ret}}
    =
    \arg\min_{k\in\{1,2\}}
    \rho_{t,k}$
    \alglabel{algline:gmm_select_cluster}

    \State \Return
    $\mathcal{R}_t
    =
    \{
        i\in\mathcal{S}_t:
        y_{t,i}=k_t^{\mathrm{ret}}
    \}$
    \alglabel{algline:gmm_return_cluster}
\EndIf
\end{algorithmic}
\end{algorithm}

\FedLNS{} first selects the number of mixture components
$\widehat{c}_t$ as
\begin{equation}
    \widehat{c}_t
    :=
    \arg\min_{c\in\{1,2\}}
    \mathrm{BIC}(c).
    \label{eq:FedLNS_choose_c}
\end{equation}
If $\widehat{c}_t=1$, all participating clients are retained:
\begin{equation}
    \mathcal{R}_t
    =
    \mathcal{S}_t.
    \label{eq:FedLNS_keep_unimodal}
\end{equation}
If $\widehat{c}_t=2$, each client $i$ is assigned to a component:
\begin{equation}
    y_{t,i}
    :=
    \arg\max_{k\in\{1,2\}}
    \widehat{\pi}_{2,k}
    \mathcal{N}
    \left(
        \mathbf{z}_{t,i};
        \widehat{\boldsymbol{\mu}}_{2,k},
        \operatorname{Diag}
        \left(
            \widehat{\boldsymbol{\sigma}}_{2,k}^{2}
        \right)
    \right).
    \label{eq:FedLNS_label}
\end{equation}
The median deviation score of component $k$ is
\begin{equation}
    \rho_{t,k}
    :=
    \operatorname{median}
    \left(
        \left\{
            \delta_{t,i}
            :
            i\in\mathcal{S}_t,
            \ y_{t,i}=k
        \right\}
    \right),
    \quad
    k\in\{1,2\}.
    \label{eq:FedLNS_cluster_score}
\end{equation}
The retained component is
\begin{equation}
    k_t^{\mathrm{ret}}
    :=
    \arg\min_{k\in\{1,2\}}
    \rho_{t,k}.
    \label{eq:FedLNS_keep_cluster}
\end{equation}
The retained client set is
\begin{equation}
    \mathcal{R}_t
    =
    \left\{
        i\in\mathcal{S}_t:
        y_{t,i}=k_t^{\mathrm{ret}}
    \right\}.
    \label{eq:FedLNS_retained_set}
\end{equation}

After the retained set $\mathcal{R}_t$ is determined, the server aggregates the complete retained updates (Algorithm~\ref{alg:FedLNS_server_module}, line~\ref{algline:server_agg}). For client $i\in\mathcal{R}_t$, the sample-weighted aggregation coefficient is
\begin{equation}
    \alpha_{t,i}
    :=
    \frac{
        n_{t,i}
    }{
        \sum_{j\in\mathcal{R}_t}
        n_{t,j}
    }.
    \label{eq:FedLNS_weight}
\end{equation}
The next global model is
\begin{equation}
    \mathbf{w}_{t+1}
    =
    \mathbf{w}_t
    +
    \sum_{i\in\mathcal{R}_t}
    \alpha_{t,i}
    \Delta\mathbf{w}_{t,i}.
    \label{eq:FedLNS_aggregation}
\end{equation}
\FedLNS{} therefore acts as a server-side screening stage before full-model aggregation. The retained updates can be passed to the sample-weighted rule above or to another compatible FL aggregation rule.


\section{Theoretical Analysis}
\label{sec_theory}

This section analyzes two aspects of \FedLNS{}.
First, we examine how changes in trainable normalization-layer parameters
contribute to changes in token preferences.
Second, we study how screening suspicious clients affects optimization of the benign global objective.

\subsection{Influence of Normalization-Parameter Changes on Token Margins}
\label{subsec:theory_ln_margin}
For an input sequence $x=(x_1,\dots,x_m)$, a language model with parameters $\mathbf{w}$ produces a logit vector
\begin{equation}
    \mathbf{z}(x;\mathbf{w})
    \in
    \mathbb{R}^{|\mathcal{V}|},
\end{equation}
where $\mathcal{V}$ is the token vocabulary. The probability assigned to token $y\in\mathcal{V}$ is
\begin{equation}
    \pi_{\mathbf{w}}(y\mid x)
    =
    \frac{
        \exp(z_y(x;\mathbf{w}))
    }{
        \sum_{v\in\mathcal{V}}
        \exp(z_v(x;\mathbf{w}))
    }.
    \label{eq:softmax_theory}
\end{equation}

For two candidate tokens $u,v\in\mathcal{V}$, define the pairwise token margin
\begin{equation}
    m_{u,v}(x;\mathbf{w})
    :=
    z_u(x;\mathbf{w})
    -
    z_v(x;\mathbf{w}).
    \label{eq:theory_pairwise_margin}
\end{equation}
The margin measures the model's relative preference for token $u$ over token $v$ and satisfies
\begin{equation}
    m_{u,v}(x;\mathbf{w})
    =
    \log
    \frac{
        \pi_{\mathbf{w}}(u\mid x)
    }{
        \pi_{\mathbf{w}}(v\mid x)
    }.
    \label{eq:margin_log_ratio}
\end{equation}
Thus, a client update that changes the pairwise margin also changes the probability ratio assigned to the two tokens.

Partition the model parameters as
\begin{equation}
    \mathbf{w}
    =
    (\mathbf{p},\mathbf{q}),
\end{equation}
where $\mathbf{p}$ contains the trainable normalization-layer parameters included in the \FedLNS{} signature and $\mathbf{q}$ contains the remaining model parameters.

For client $i$ at communication round $t$, write the local update as
\begin{equation}
    \Delta\mathbf{w}_{t,i}
    =
    (\Delta\mathbf{p}_{t,i},\Delta\mathbf{q}_{t,i}).
\end{equation}
Here, $\Delta\mathbf{p}_{t,i}$ is the concatenated normalization-parameter component corresponding to the normalization-layer signature used by \FedLNS{}.

Because a client update affects the model across inputs, we define its prompt-averaged pairwise-margin change as
\begin{equation}
    \Delta\bar{m}_{t,i}(u,v)
    :=
    \mathbb{E}_{x\sim\mathcal{X}}
    \left[
        m_{u,v}
        \left(
            x;
            \mathbf{w}_t+\Delta\mathbf{w}_{t,i}
        \right)
        -
        m_{u,v}(x;\mathbf{w}_t)
    \right].
    \label{eq:expected_margin_shift_theory}
\end{equation}
Here, $m_{u,v}(x;\mathbf{w})$ denotes the margin for an individual prompt, whereas $\Delta\bar{m}_{t,i}(u,v)$ denotes the change in that margin averaged over the prompt distribution $\mathcal{X}$.

The following proposition identifies the direct first-order contribution of the normalization-parameter component to the expected token-margin change.

\begin{proposition}
[Normalization-parameter contribution to token-margin change]
\label{prop:ln_margin_main}
Suppose
$m_{u,v}(x;\mathbf{w})$
is twice continuously differentiable in a neighborhood of $\mathbf{w}_t$,
and its Hessian is uniformly bounded in operator norm by $L_m$.
Then
\begin{equation}
    \Delta\bar{m}_{t,i}(u,v)
    =
    \mathbf{g}_{p,u,v}^{\top}
    \Delta\mathbf{p}_{t,i}
    +
    \mathbf{g}_{q,u,v}^{\top}
    \Delta\mathbf{q}_{t,i}
    +
    R_{t,i}(u,v),
    \label{eq:ln_margin_decomposition_main}
\end{equation}
where
\begin{equation}
\begin{aligned}
    \mathbf{g}_{p,u,v}
    &:=
    \mathbb{E}_{x\sim\mathcal{X}}
    \left[
        \nabla_{\mathbf{p}}
        m_{u,v}(x;\mathbf{w}_t)
    \right],\\
    \mathbf{g}_{q,u,v}
    &:=
    \mathbb{E}_{x\sim\mathcal{X}}
    \left[
        \nabla_{\mathbf{q}}
        m_{u,v}(x;\mathbf{w}_t)
    \right],
\end{aligned}
\end{equation}
and the prompt-averaged remainder satisfies
\begin{equation}
    |R_{t,i}(u,v)|
    \le
    \frac{L_m}{2}
    \left\|
        \Delta\mathbf{w}_{t,i}
    \right\|_2^2.
    \label{eq:ln_margin_remainder_main}
\end{equation}
\end{proposition}

Proposition~\ref{prop:ln_margin_main} shows that the normalization-parameter component contributes directly to the first-order change in token margins through
\begin{equation}
    \mathbf{g}_{p,u,v}^{\top}
    \Delta\mathbf{p}_{t,i}.
\end{equation}
Consequently, normalization-parameter changes contribute directly to first-order changes in relative token preferences. 

\subsection{Effect of Screening on the Aggregated Update}
\label{subsec:theory_convergence_main}

We next analyze the effect of excluding suspicious clients before
aggregation. Let $\mathcal{C}_{\mathrm{ben}}$ denote the set of benign clients. The benign global objective is
\begin{equation}
    F(\mathbf{w})
    :=
    \sum_{i\in\mathcal{C}_{\mathrm{ben}}}
    \lambda_i F_i(\mathbf{w}),
    \qquad
    \lambda_i\geq 0,
    \quad
    \sum_{i\in\mathcal{C}_{\mathrm{ben}}}
    \lambda_i=1,
    \label{eq:benign_objective_main}
\end{equation}
where $F_i$ is the local objective of benign client $i$.

At round $t$, \FedLNS{} returns a retained client set $\mathcal{R}_t\subseteq\mathcal{S}_t$.
For retained client $i$, define the normalized aggregation coefficient
\begin{equation}
    \alpha_{t,i}
    =
    \frac{
        n_{t,i}
    }{
        \sum_{j\in\mathcal{R}_t}
        n_{t,j}
    },
    \qquad
    i\in\mathcal{R}_t.
\end{equation}
The screened update is
\begin{equation}
    \mathbf{w}_{t+1}
    =
    \mathbf{w}_t
    +
    \sum_{i\in\mathcal{R}_t}
    \alpha_{t,i}
    \Delta\mathbf{w}_{t,i}.
    \label{eq:filtered_update_main}
\end{equation}

For analysis, write the same update as
\begin{equation}
    \mathbf{w}_{t+1}
    =
    \mathbf{w}_t
    -
    \eta_t\mathbf{g}_t,
    \label{eq:abstract_filtered_update_main}
\end{equation}
where $\eta_t>0$ is an effective step size and $\mathbf{g}_t$ is the effective retained update direction.

Define the screening error
\begin{equation}
    \mathbf{e}_t
    :=
    \mathbb{E}
    \left[
        \mathbf{g}_t
        \mid
        \mathcal{F}_t
    \right]
    -
    \nabla F(\mathbf{w}_t),
    \label{eq:filtering_error_main}
\end{equation}
where $\mathcal{F}_t$ contains the information available before the round-$t$ screening, i.e., the vector $\mathbf{e}_t$ measures the difference between the expected retained update direction and the gradient of the benign global objective.

\begin{theorem}[Finite-time bound under screened aggregation]
\label{thm:filtered_convergence_main}
Assume that $F$ is $L_F$-smooth and lower bounded by $F_{\inf}$.
Suppose that, for every round $t$, there exist nonnegative quantities $\beta_t$ and $\nu_t^2$ such that
\begin{equation}
    \|\mathbf{e}_t\|_2
    \le
    \beta_t,
    \qquad
    \beta_t\geq 0,
    \label{eq:filtering_error_bound_main}
\end{equation}
and
\begin{equation}
    \mathbb{E}
    \left[
        \left\|
            \mathbf{g}_t
            -
            \mathbb{E}
            [
                \mathbf{g}_t
                \mid
                \mathcal{F}_t
            ]
        \right\|_2^2
        \,\middle|\,
        \mathcal{F}_t
    \right]
    \le
    \nu_t^2.
    \label{eq:variance_bound_main}
\end{equation}
If
$0\leq\eta_t\leq 1/(4L_F)$,
then for any $T\geq 1$ with
\begin{equation}
    S_T
    :=
    \sum_{t=0}^{T-1}
    \eta_t
    >
    0,
\end{equation}
we have
\begin{equation}
\begin{aligned}
    &
    \frac{
        \sum_{t=0}^{T-1}
        \eta_t
        \mathbb{E}
        \left[
            \|\nabla F(\mathbf{w}_t)\|_2^2
        \right]
    }{
        S_T
    }
    \le
    \frac{
        2
        \left(
            F(\mathbf{w}_0)-F_{\inf}
        \right)
    }{
        S_T
    }
    \\
    &\quad
    +
    \frac{
        2
        \sum_{t=0}^{T-1}
        \left(
            \eta_t
            +
            L_F\eta_t^2
        \right)
        \mathbb{E}
        [
            \beta_t^2
        ]
    }{
        S_T
    }
    +
    \frac{
        L_F
        \sum_{t=0}^{T-1}
        \eta_t^2
        \mathbb{E}
        [
            \nu_t^2
        ]
    }{
        S_T
    }.
\end{aligned}
\label{eq:filtered_convergence_bound_main}
\end{equation}
\end{theorem}

Theorem~\ref{thm:filtered_convergence_main} bounds the step-size-weighted average squared gradient norm of the benign objective. 
The right-hand side contains three terms: the initial optimization gap, the contribution of the screening-error upper bound $\beta_t$ (i.e., the discrepancy between the expected retained update direction and the gradient of the benign global objective), and the stochastic update-noise contribution governed by $\nu_t^2$. 
A smaller $\beta_t$ indicates that the expected direction formed from the retained clients more closely approximates the benign global gradient. The theorem characterizes optimization under such an approximation, but it does not independently guarantee that the \FedLNS{} screening rule makes $\beta_t$ small in every round. 

The following corollary gives conditions under which the weighted average gradient norm vanishes.

\begin{corollary}[Convergence to first-order stationarity]
\label{cor:stationarity_main}
Suppose
\begin{equation}
    S_T
    =
    \sum_{t=0}^{T-1}
    \eta_t
    \to
    \infty,
    \qquad
    \frac{
        \sum_{t=0}^{T-1}
        \eta_t^2
    }{
        S_T
    }
    \to
    0.
    \label{eq:stepsize_stationarity_condition}
\end{equation}
Assume that the update variance is uniformly bounded:
\begin{equation}
    \mathbb{E}[\nu_t^2]
    \le
    \nu^2
\end{equation}
for some finite constant $\nu^2$.
Also assume that the weighted average screening error vanishes:
\begin{equation}
    \frac{
        \sum_{t=0}^{T-1}
        \eta_t
        \mathbb{E}
        [
            \beta_t^2
        ]
    }{
        S_T
    }
    \to
    0.
    \label{eq:weighted_filtering_error_condition}
\end{equation}
Then
\begin{equation}
    \frac{
        \sum_{t=0}^{T-1}
        \eta_t
        \mathbb{E}
        \left[
            \|\nabla F(\mathbf{w}_t)\|_2^2
        \right]
    }{
        S_T
    }
    \to
    0.
    \label{eq:weighted_stationarity_main}
\end{equation}
Equivalently, if a random round $\tau$ is sampled from
$\{0,\dots,T-1\}$ with probability
\begin{equation}
    \Pr(\tau=t)
    =
    \frac{\eta_t}{S_T},
\end{equation}
then
\begin{equation}
    \mathbb{E}
    \left[
        \|\nabla F(\mathbf{w}_{\tau})\|_2^2
    \right]
    \to
    0.
    \label{eq:random_iterate_stationarity_main}
\end{equation}
\end{corollary}

Corollary~\ref{cor:stationarity_main} shows that the expected gradient norm of a step-size-weighted random iterate vanishes when the step sizes provide sufficient cumulative progress, the stochastic update variance remains bounded, and the average screening error becomes negligible. The final condition requires the retained update direction to approach the benign gradient direction.

The following proposition decomposes the screening error into quantities associated with retained malicious updates, benign-client selection, and local-training drift.

\begin{proposition}[Interpretation of screening error]
\label{prop:filtering_error_interpretation_main}
Let
\begin{equation}
    \mu_t
    :=
    \sum_{a\in\mathcal{R}_t\cap\mathcal{C}_{\mathrm{mal}}}
    \alpha_{t,a}
\end{equation}
denote the total aggregation weight assigned to retained malicious clients. When $\mu_t<1$, define the renormalized gradient of the retained benign clients as
\begin{equation}
    \bar{\mathbf{g}}_t^{B,\mathrm{ret}}
    :=
    \sum_{i\in\mathcal{R}_t\cap\mathcal{C}_{\mathrm{ben}}}
    \frac{\alpha_{t,i}}{1-\mu_t}
    \nabla F_i(\mathbf{w}_t).
\end{equation}
The retained-benign representativeness error is
\begin{equation}
    \chi_t
    :=
    \left\|
        \bar{\mathbf{g}}_t^{B,\mathrm{ret}}
        -
        \nabla F(\mathbf{w}_t)
    \right\|_2.
    \label{eq:benign_representativeness_main}
\end{equation}
Thus, $\chi_t$ measures how closely the retained benign clients, after their weights are renormalized, represent the full benign-client gradient. Let $\varrho_t$ denote the retained benign local-training drift, as formally defined in Appendix~\ref{app:filtering_error_decomposition}.
Under the conditions stated there, the screening error satisfies
\begin{equation}
    \|\mathbf{e}_t\|_2
    \le
    (1-\mu_t)\chi_t
    +
    2G\mu_t
    +
    \varrho_t.
    \label{eq:filtering_error_interpretation_main}
\end{equation}
Consequently, a valid choice of the screening-error bound in Theorem~\ref{thm:filtered_convergence_main} is
\begin{equation}
    \beta_t
    :=
    (1-\mu_t)\chi_t
    +
    2G\mu_t
    +
    \varrho_t.
\end{equation}
\end{proposition}

Proposition~\ref{prop:filtering_error_interpretation_main} identifies three sources of screening error. The term $\mu_t$ is small when malicious clients receive little retained aggregation weight. The term $\chi_t$ is small when the retained benign clients remain representative of the benign population. The term $\varrho_t$ is small when local training remains close to the global model. Together, these quantities determine the screening-error contribution in Theorem~\ref{thm:filtered_convergence_main}.

For a constant effective step size, the finite-time result yields a stationary-neighborhood bound:

\begin{corollary}[Stationary neighborhood with constant step size]
\label{cor:constant_stationary_neighborhood_main}
Suppose that
$\eta_t=\eta\leq 1/(4L_F),\,\forall t$.
Define
\begin{equation}
    \overline{\beta}_T^2
    :=
    \frac{1}{T}
    \sum_{t=0}^{T-1}
    \mathbb{E}
    [
        \beta_t^2
    ],
    \qquad
    \overline{\nu}_T^2
    :=
    \frac{1}{T}
    \sum_{t=0}^{T-1}
    \mathbb{E}
    [
        \nu_t^2
    ].
\end{equation}
Then
\begin{equation}
\begin{aligned}
    \frac{1}{T}
    \sum_{t=0}^{T-1}
    \mathbb{E}
    \left[
        \|\nabla F(\mathbf{w}_t)\|_2^2
    \right]
    &\le
    \frac{
        2
        \left(
            F(\mathbf{w}_0)-F_{\inf}
        \right)
    }{
        \eta T
    }
    \\
    &\quad
    +
    2
    \left(
        1+L_F\eta
    \right)
    \overline{\beta}_T^2
    +
    L_F\eta
    \overline{\nu}_T^2.
    \label{eq:constant_stationary_neighborhood_main}
\end{aligned}
\end{equation}
\end{corollary}

Corollary~\ref{cor:constant_stationary_neighborhood_main} shows that, with a constant effective step size, the average gradient norm approaches a neighborhood whose size is controlled by the average screening error and the stochastic update variance.

Complete assumptions, proofs, the screening-error decomposition, and the local-training drift bound are provided in Appendix~\ref{app:theory}.
\section{Experimental Evaluation}
\label{sec:experiments}
In this section, we evaluate \FedLNS{} in federated transformer-based language-model training under target-corruption attacks. The evaluation examines robustness across model architectures, IID and non-IID data partitions, and malicious-client fractions ranging from $0\%$ to $40\%$.

\subsection{Experimental Setup}

We evaluate \FedLNS{} on three transformer model families with different language-modeling objectives:
\begin{itemize}
    \item A GPT-style causal transformer~\cite{radford2019language} trained from scratch on WikiText using autoregressive next-token prediction.
    \item A BERT-style masked-language transformer~\cite{devlin2019bert} trained from scratch on Tiny Shakespeare using bidirectional masked-token prediction.
    \item A LLaMA-style decoder-only causal transformer~\cite{touvron2023llama} trained from scratch on StackOverflow text.
\end{itemize}

\textbf{Settings.} Unless otherwise specified, each experiment uses $K=200$ clients, a $10\%$ participation rate, $20$ clients sampled uniformly at random per communication round, and $200$ communication rounds~\cite{sun2023efficient,kang2024polaris}. Each configuration is repeated using three random seeds, $\{100,200,300\}$, under IID and two-shard non-IID data partitions~\cite{dong2023no}.

For the $K=200$ clients considered in our experiments, we use the default activation coverage $M_{\mathrm{act}}=100$, so screening begins after signatures from half of the client identities have been observed. 

\textbf{Malicious clients.} We consider a target-corruption attack in which a fixed subset of the client population is designated as malicious at the beginning of each training run and remains malicious throughout that run.
For every malicious client selected in a communication round, the input sequence remains unchanged, while the target token at each supervised prediction position is independently replaced, with probability $p_{\mathrm{corr}}$, by a randomly sampled valid vocabulary token; padding and EOS tokens are excluded from the replacement candidates. We set $p_{\mathrm{corr}}=1.0$, so every eligible target position used for local training is replaced.
The malicious-client fraction is defined over the complete client population. In each communication round, $20$ clients are sampled uniformly without replacement from the $K=200$ clients; therefore, the number of malicious clients participating in an individual round varies according to the sampled client set.

\textbf{Performance metrics.} We report perplexity and token-level semantic entropy. For the causal language models, perplexity is the exponential of the mean next-token cross-entropy. For the masked-language model, masked-token perplexity is the exponential of the mean cross-entropy over masked positions. Token-level semantic entropy is the mean predictive entropy over the evaluated token positions and measures the concentration of the corresponding predictive distributions.

\begin{figure*}[ht!]
\centering
\begin{subfigure}{0.31\textwidth}
  \centering
  \includegraphics[width=\linewidth]{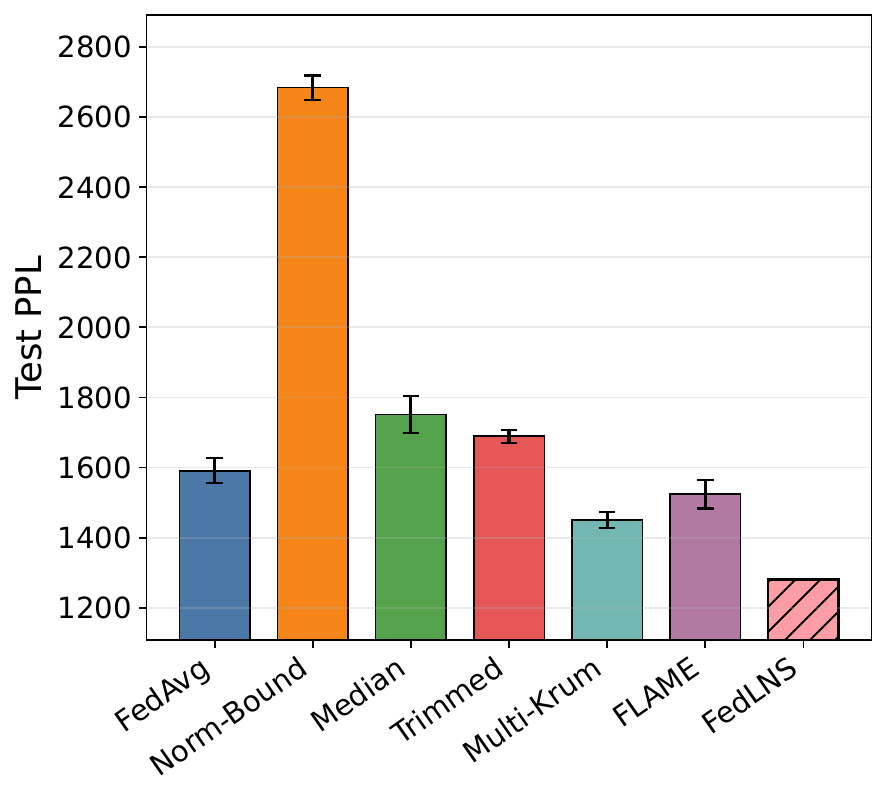}
  \subcaption{GPT/WikiText, IID partitions}
\end{subfigure}
\hfill
\begin{subfigure}{0.31\textwidth}
  \centering
  \includegraphics[width=\linewidth]{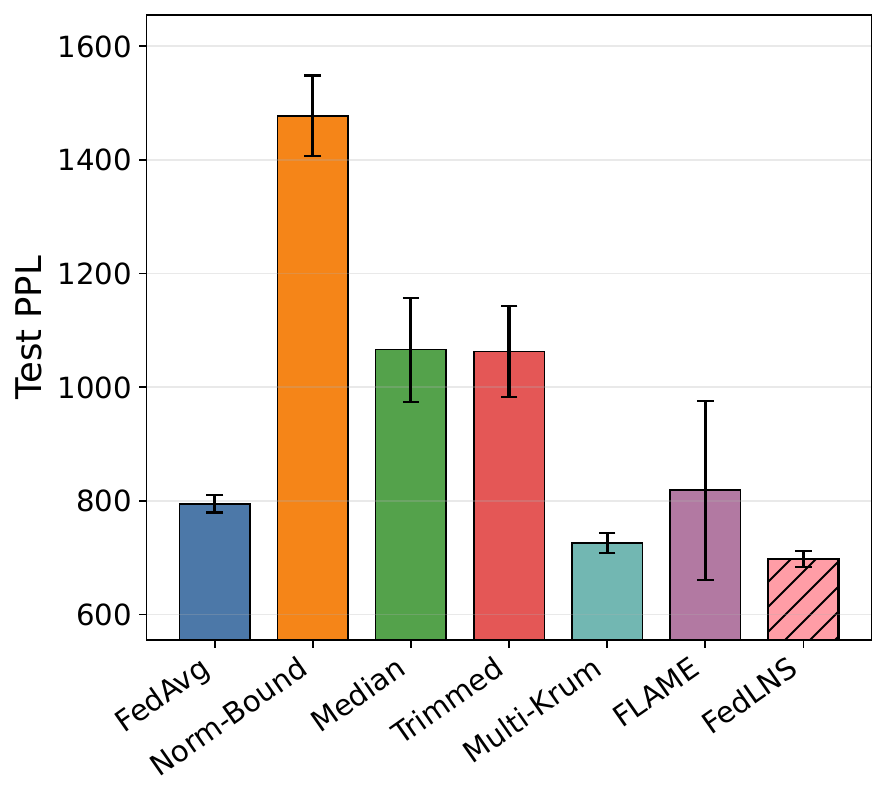}
  \subcaption{BERT/Shakespeare, IID partitions}
\end{subfigure}
\hfill
\begin{subfigure}{0.31\textwidth}
  \centering
  \includegraphics[width=\linewidth]{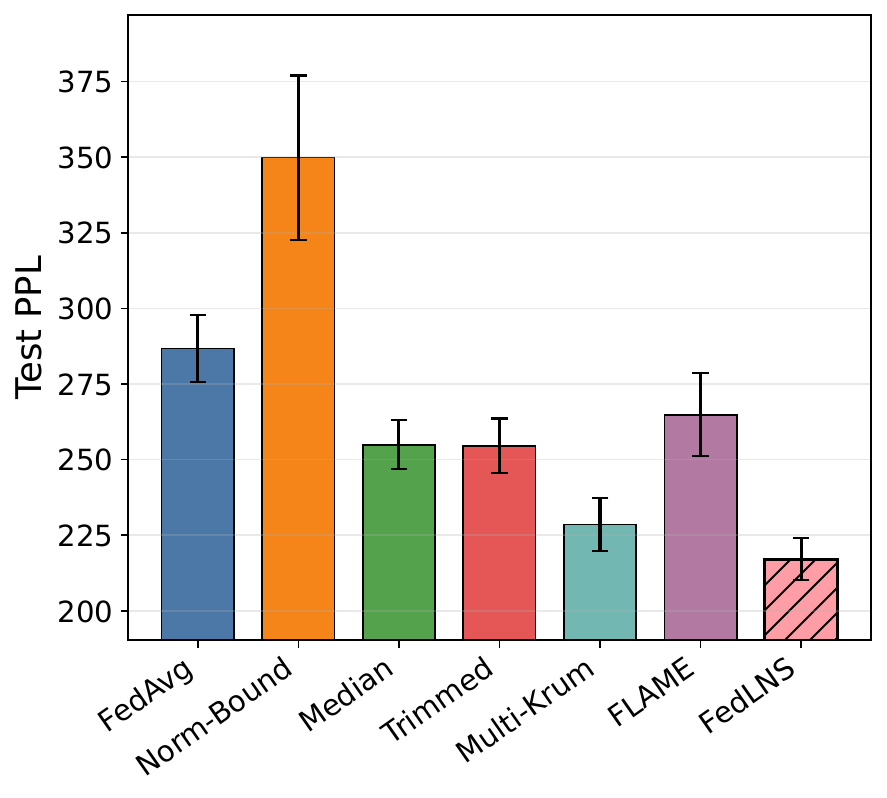}
  \subcaption{LLaMA/StackOverflow, IID partitions}
\end{subfigure}

\vspace{0.8em}

\begin{subfigure}{0.31\textwidth}
  \centering
  \includegraphics[width=\linewidth]{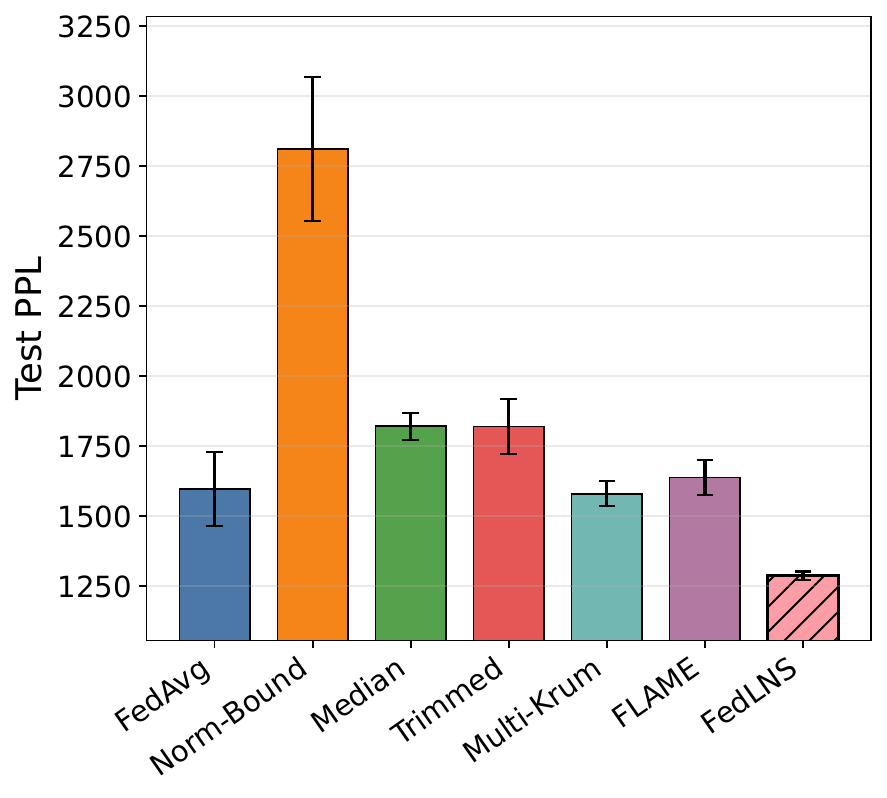}
  \subcaption{GPT/WikiText, non-IID partitions}
\end{subfigure}
\hfill
\begin{subfigure}{0.31\textwidth}
  \centering
  \includegraphics[width=\linewidth]{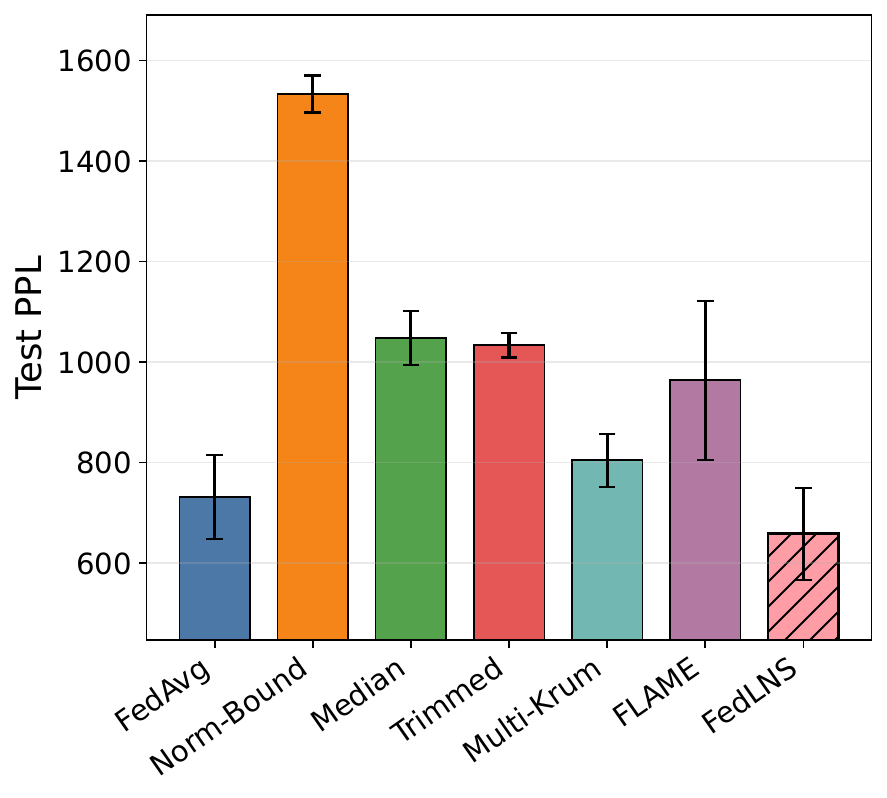}
  \subcaption{BERT/Shakespeare, non-IID partitions}
\end{subfigure}
\hfill
\begin{subfigure}{0.31\textwidth}
  \centering
  \includegraphics[width=\linewidth]{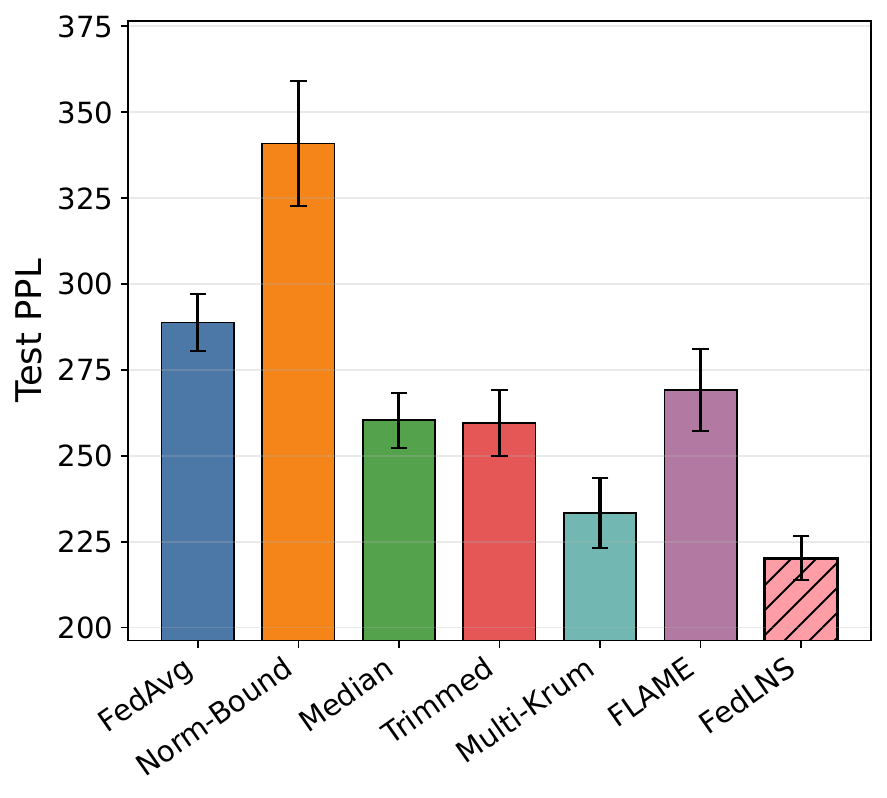}
  \subcaption{LLaMA/StackOverflow, non-IID partitions}
\end{subfigure}

\caption{Perplexity (lower is better) under target manipulation affecting $40\%$ of the client population. Bars show the mean over three random seeds, and error bars show one sample standard deviation. \FedLNS{} consistently shows the lowest perplexity, across both IID and non-IID data partitions.}
\label{fig:attack40_ppl}
\end{figure*}

For each random seed, the checkpoint with the lowest test loss
is selected, and the test loss, perplexity, and token-level
semantic entropy are reported from that same model state.
Results are reported as the mean and sample standard deviation across seeds $\{100,200,300\}$ for each model, data partition, aggregation method, and malicious-client fraction.

\textbf{Baselines.} We compare \FedLNS{} with six representative aggregation methods covering conventional averaging, norm-constrained aggregation, coordinate-wise robust statistics, distance-based selection, and clustering-based filtering. \textit{FedAvg} performs sample-weighted averaging and serves as the non-robust baseline~\cite{mcmahan2017communication}. \textit{Norm-Bounded FedAvg} clips client updates before aggregation to limit the effect of large update magnitudes~\cite{sun2019can}. \textit{Coordinate-wise Median} and \textit{Trimmed Mean} apply robust statistics independently to each model coordinate~\cite{yin2018byzantine}. \textit{Multi-Krum} selects updates with mutually consistent distance profiles~\cite{blanchard2017machine}. \textit{FLAME} combines clustering, norm clipping, and Gaussian noise injection~\cite{nguyen2022flame}. Trimmed Mean and Multi-Krum are configured using the true population-level malicious-client fraction, whereas \FedLNS{} does not require this fraction
as an input.

Appendix~\ref{app:experimental_details} provides the complete dataset, tokenization, partition, architecture, optimization, attack, baseline, and \FedLNS{} configurations; Appendix~\ref{app:remaining_tables} reports the full malicious-client-fraction results; and Appendices~\ref{app:bank_size_ablation} and~\ref{app:gmm_ablation} provide the bank-coverage-at-activation and GMM-selection ablations, respectively.

\subsection{Perplexity and Token-Level Semantic Entropy}
\label{sec:strongest_attack_results}
Figure~\ref{fig:attack40_ppl} compares the perplexity of \FedLNS{} and the six baselines when $40\%$ of the client population is malicious. Figures~\ref{fig:attack40_ppl}(a)--\ref{fig:attack40_ppl}(c) report the IID results, whereas Figures~\ref{fig:attack40_ppl}(d)--\ref{fig:attack40_ppl}(f) report the non-IID results. Error bars show one sample standard deviation across seeds $\{100,200,300\}$.

Under IID data partitions, \FedLNS{} achieves the lowest perplexity for all three evaluated model families. Relative to the strongest competing baseline, it reduces perplexity by $11.59\%$ for GPT/WikiText (from $1449.72$ to $1281.63$), $3.88\%$ for BERT/Tiny Shakespeare (from $725.76$ to $697.62$), and $4.98\%$ for LLaMA/StackOverflow (from $228.52$ to $217.14$).
These results show that the global models trained with \FedLNS{} exhibit less attack-induced degradation across all three architectures.

Under the two-shard non-IID partitions, \FedLNS{} reduces perplexity by $18.57\%$ for GPT/WikiText (from $1579.81$ to $1286.50$), $10.05\%$ for BERT/Tiny Shakespeare (from $731.73$ to $658.21$), and $5.67\%$ for LLaMA/StackOverflow (from $233.42$ to $220.18$) relative to the strongest non-\FedLNS{} baseline. The improvement across all three non-IID settings shows that the downstream benefit of the screening procedure persists when benign clients have heterogeneous local data distributions.


\begin{figure*}[ht]
\centering
\begin{subfigure}{0.31\textwidth}
  \centering
  \includegraphics[width=\linewidth]{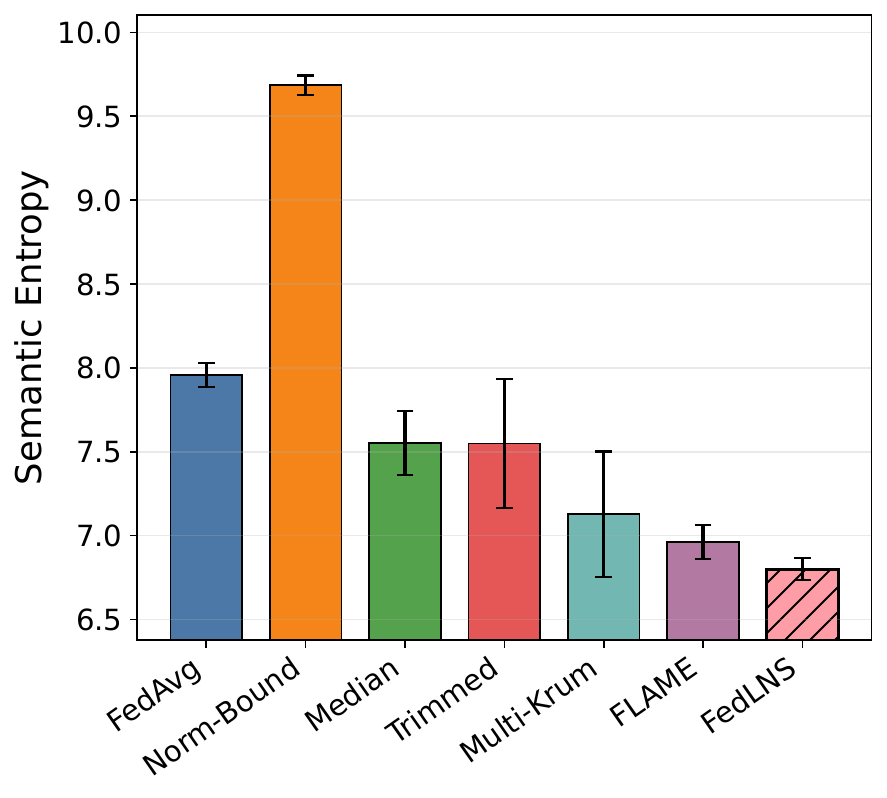}
  \subcaption{GPT/WikiText, IID partitions}
\end{subfigure}
\hfill
\begin{subfigure}{0.31\textwidth}
  \centering
  \includegraphics[width=\linewidth]{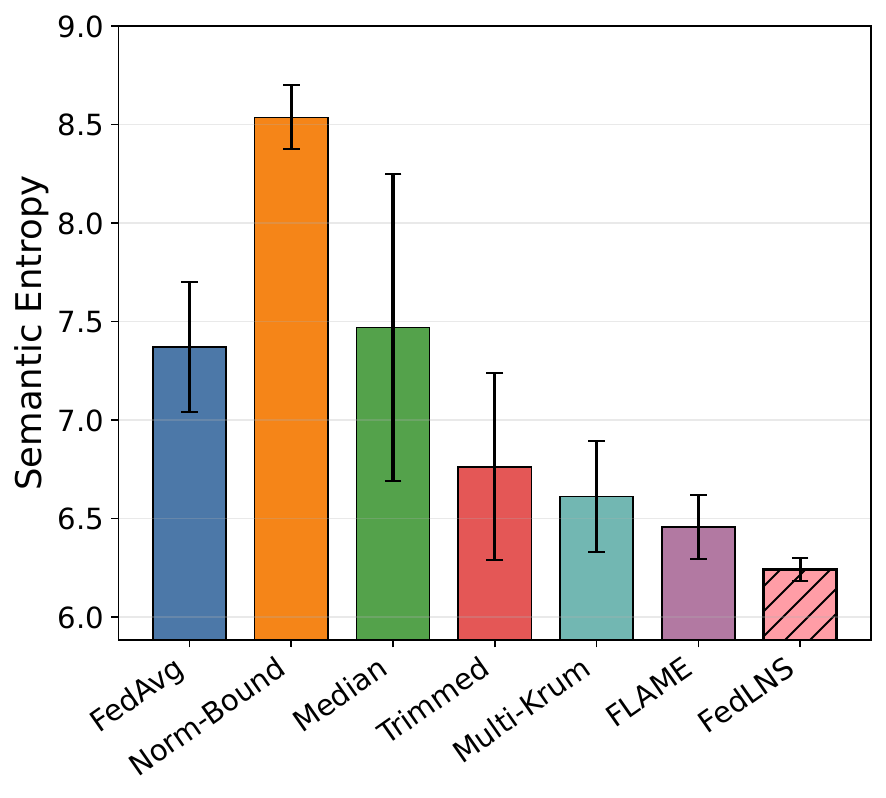}
  \subcaption{BERT/Shakespeare, IID partitions}
\end{subfigure}
\hfill
\begin{subfigure}{0.31\textwidth}
  \centering
  \includegraphics[width=\linewidth]{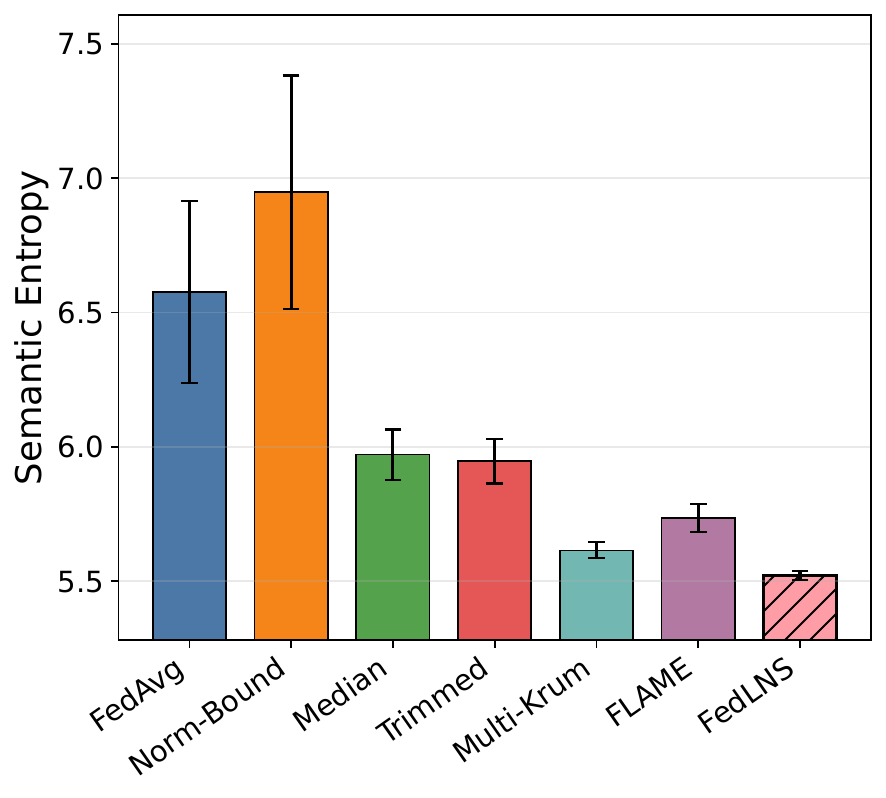}
  \subcaption{LLaMA/StackOverflow, IID partitions}
\end{subfigure}

\vspace{0.8em}

\begin{subfigure}{0.31\textwidth}
  \centering
  \includegraphics[width=\linewidth]{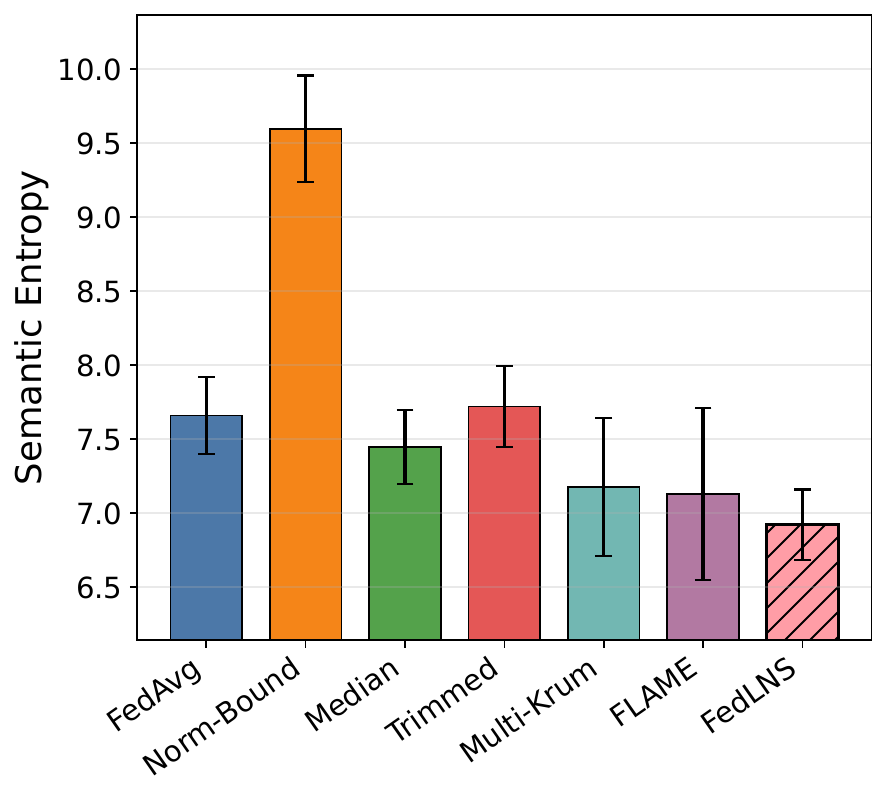}
  \subcaption{GPT/WikiText, non-IID partitions}
\end{subfigure}
\hfill
\begin{subfigure}{0.31\textwidth}
  \centering
  \includegraphics[width=\linewidth]{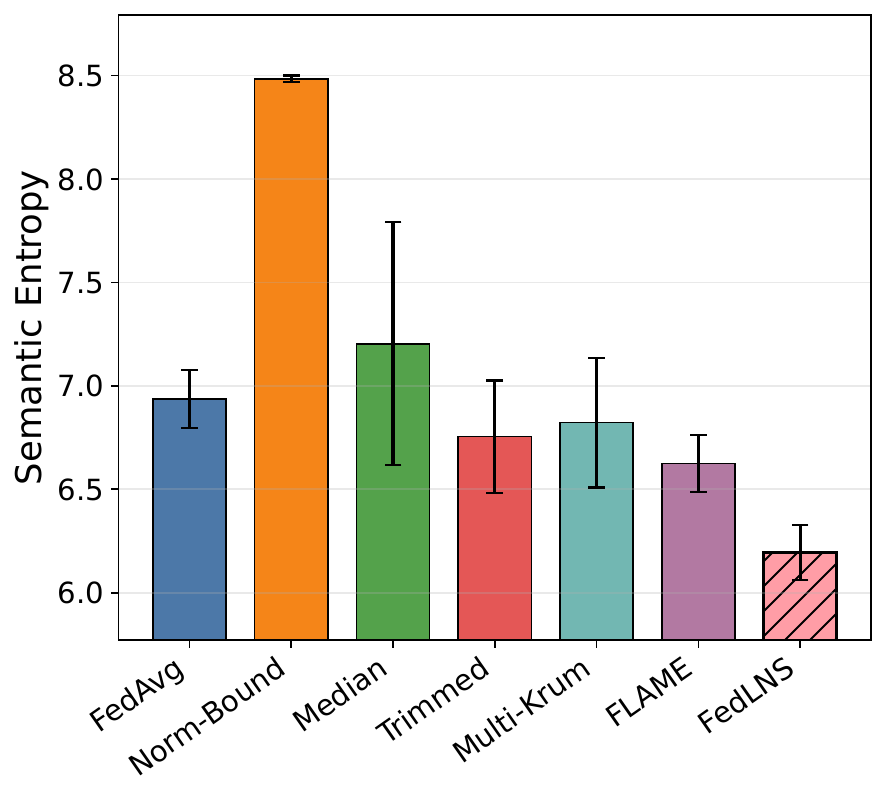}
  \subcaption{BERT/Shakespeare, non-IID partitions}
\end{subfigure}
\hfill
\begin{subfigure}{0.31\textwidth}
  \centering
  \includegraphics[width=\linewidth]{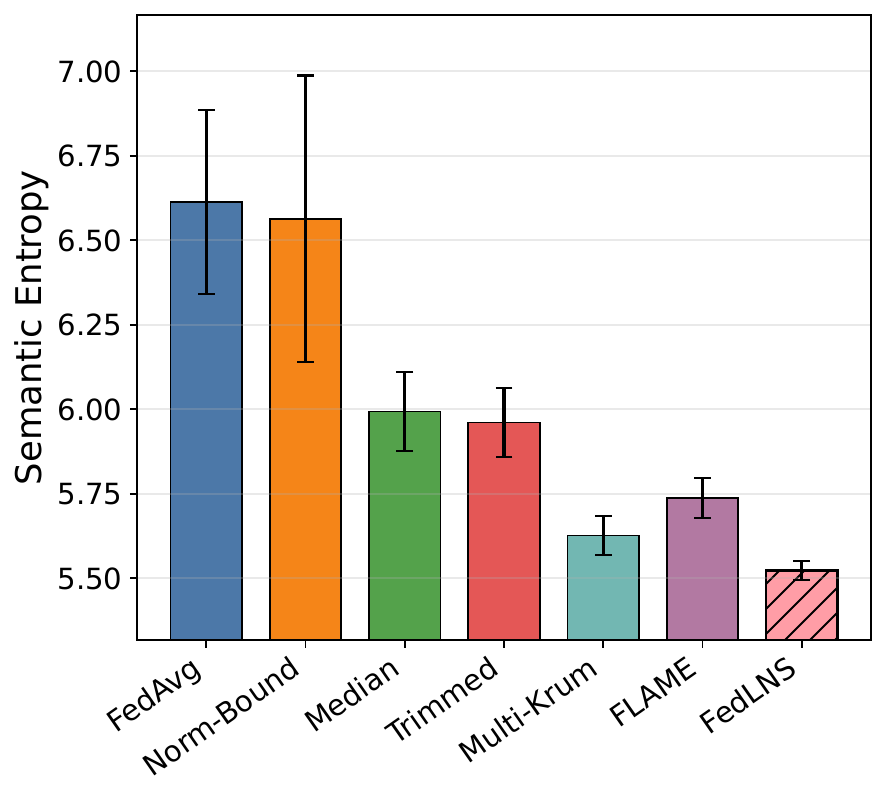}
  \subcaption{LLaMA/StackOverflow, non-IID partitions}
\end{subfigure}

\caption{Token-level semantic entropy (lower is better) under target manipulation affecting $40\%$ of the client population. Bars show the mean over three random seeds, and error bars show one sample standard deviation.  \FedLNS{} consistently shows the lowest semantic entropy, across both IID and non-IID data partitions.}
\label{fig:attack40_sem_ent}
\end{figure*}

Because lower entropy does not necessarily imply greater correctness, we interpret the entropy results jointly with perplexity.

Figure~\ref{fig:attack40_sem_ent} reports token-level semantic entropy for the three model families under $40\%$ population-level target manipulation. Under IID data partitions, Figures~\ref{fig:attack40_sem_ent}(a)--\ref{fig:attack40_sem_ent}(c) show that \FedLNS{} achieves the lowest entropy in all three settings. Relative to the strongest non-\FedLNS{} baseline, it reduces entropy by $2.30\%$ for GPT/WikiText (from $6.96$ to $6.80$), $3.41\%$ for BERT/Tiny Shakespeare (from $6.46$ to $6.24$), and $1.60\%$ for LLaMA/StackOverflow (from $5.61$ to $5.52$).

Under non-IID data partitions, Figures~\ref{fig:attack40_sem_ent}(d)--\ref{fig:attack40_sem_ent}(f) show reductions of $2.95\%$ for GPT/WikiText (from $7.13$ to $6.92$), $6.50\%$ for BERT/Tiny Shakespeare (from $6.62$ to $6.19$), and $1.95\%$ for LLaMA/StackOverflow (from $5.63$ to $5.52$). 
Together with the corresponding perplexity reductions, these results show that the global models obtained after \FedLNS{} screening achieve lower test loss and lower uncertainty over competing token predictions under the evaluated attack.

\subsection{Impact of Malicious Clients}
Appendix~\ref{app:remaining_tables} reports the complete numerical results across malicious-client fractions for all model families and data partitions. For GPT/WikiText under IID partitioning, the population-level malicious-client fraction increases from $0\%$ to $40\%$.

For perplexity, \FedLNS{} exhibits limited degradation as the malicious-client fraction increases. Its perplexity rises by $2.82\%$, from $1246.57$ at $0\%$ to $1281.63$ at $40\%$. Over the same range, perplexity increases by $27.97\%$ for FedAvg, $48.32\%$ for Norm-Bounded FedAvg, $17.95\%$ for Median, $36.05\%$ for Trimmed Mean, $16.72\%$ for Multi-Krum, and $15.15\%$ for FLAME. \FedLNS{} achieves the lowest perplexity at every nonzero attack fraction from $10\%$ to $40\%$.

Token-level semantic entropy exhibits a similar attacked-setting trend. Although Norm-Bounded FedAvg has the lowest entropy in the attack-free setting, its entropy increases by $47.94\%$ between $0\%$ and $40\%$ malicious clients. In contrast, the entropy of \FedLNS{} increases by $1.04\%$, from $6.73$ to $6.80$, and is the lowest among the evaluated methods for every nonzero attack fraction. The simultaneous stability of perplexity and entropy indicates that \FedLNS{} limits both predictive degradation and uncertainty as the target-corruption rate increases.

\subsection{Impact of Bank Coverage at Activation}
As defined by Eq.~\eqref{eq:FedLNS_bank_threshold}, \FedLNS{} begins screening and global aggregation after the bank reaches the activation coverage $M_{\mathrm{act}}$. We evaluate $M_{\mathrm{act}}\in\{0,50,100,150\}$, corresponding to immediate, $25\%$, $50\%$, and $75\%$ population coverage. Setting $M_{\mathrm{act}}=0$ activates screening immediately using the currently available bank; it does not remove the bank. In all settings, the bank continues growing after activation as additional client identities are observed.

The ablation reports performance over the first ten active aggregation rounds under IID partitioning with $40\%$ malicious clients. Relative to immediate activation with $M_{\mathrm{act}}=0$, activation at $M_{\mathrm{act}}=150$ reduces perplexity by $10.62\%$ for GPT/WikiText, $31.89\%$ for BERT/Tiny Shakespeare, and $74.68\%$ for LLaMA/StackOverflow. Token-level semantic entropy decreases by $2.53\%$, $20.20\%$, and $25.33\%$, respectively. Thus, the main ablation exhibits a monotonic perplexity improvement as activation coverage increases.

These results support maximizing bank coverage whenever the resulting activation delay is acceptable. A broader bank provides a more representative cross-client reference, and the bank should continue growing toward full population coverage after screening begins.
The activation threshold is therefore a practical starting criterion rather than a desired final bank size. Smaller activation coverages remain usable when waiting for broader coverage would delay training excessively, but they should not be interpreted as equivalent to a fully populated bank. The largest evaluated activation coverage is $75\%$; the experiments therefore support the observed benefit of increasing coverage within this range rather than establishing an empirical optimum at $100\%$ coverage.

The GMM-selection ablation shows that BIC-selected one-vs-two-component
screening and forced two-component screening produce nearly identical
selected-checkpoint results in the evaluated settings. BIC retains the
option to keep all participating clients when the one-component model is
preferred, but the endpoint metrics exhibit limited sensitivity to this
model-selection rule. Complete results are provided in
Appendix~\ref{app:gmm_ablation}.


\section{Conclusion}
\label{sec_cond}

This paper introduced \FedLNS{}, a lightweight server-side framework that develops a normalization-layer update signature to screen malicious updates before standard full-model aggregation. Moreover, a server-side retain rule was designed, which constructs a robust cross-client reference from the latest observed signatures and combines median/MAD standardization with BIC-guided mixture modeling. For GPT-style, BERT-style, and LLaMA-style models under both IID and non-IID partitions, \FedLNS{} achieved lower perplexity than all six baselines when $40\%$ of the client population was malicious. The maximum reductions were $18.57\%$ in perplexity and $6.50\%$ in token-level semantic entropy.

\section*{Acknowledgements}
This work was in part supported by the Fundação para a Ciência e a Tecnologia (Portuguese Foundation for Science and Technology) through the Carnegie Mellon Portugal Program. 

\bibliographystyle{IEEEtran}
\bibliography{bibDefALLM}

\ifCLASSOPTIONcaptionsoff
  \newpage
\fi

\appendices
\onecolumn
\clearpage

\setcounter{secnumdepth}{2}
\setcounter{tocdepth}{2}

\startcontents[appendix]
\printcontents[appendix]{l}{1}{
  \section*{Appendix Contents}
}

\clearpage
\section{Probabilistic Analysis of Partial Bank Coverage at Activation}
\label{app:min_bank_probability}
\FedLNS{} ideally benefits from the broadest available client coverage,
and the history bank continues expanding after screening begins. Waiting for full population coverage before beginning aggregation may nevertheless be impractically slow. We therefore use $M_{\mathrm{act}}$ to denote a partial bank coverage at which screening is permitted to start.

This appendix analyzes, under uniform random participation, the probability that malicious clients form a majority of a partially populated bank and the number of rounds required to reach a chosen activation coverage. The analysis relies on the $50\%$ breakdown point of the coordinate-wise median: the median/MAD reference remains in its benign-majority regime when malicious identities constitute fewer than half of the bank. The resulting bound provides a reliability condition for early activation; it does not imply that a partial bank has the same statistical quality as a fully populated bank.

\subsection{A Hypergeometric Model of Bank Contamination}

Let
\begin{equation}
    N_t^{\mathrm{bank}}
    :=
    |\mathcal{H}_t|
\end{equation}
denote the number of distinct client identities represented in the refreshed history bank at round $t$.

Let $\mathcal{C}_{\mathrm{mal}}$ be the set of malicious clients, with
\begin{equation}
    |\mathcal{C}_{\mathrm{mal}}| = \alpha K,
    \qquad 0 \le \alpha < \frac{1}{2},
\end{equation}
where $\alpha$ is the global malicious-client fraction (if $\alpha K$ is not an integer, the same discussion applies with $\lfloor \alpha K\rfloor$ or $\lceil \alpha K\rceil$ malicious clients). 

Let $M_t := |\mathcal{H}_t \cap \mathcal{C}_{\mathrm{mal}}|$ denote the number of malicious client identities represented in the bank, so the bank has a benign majority when $M_t < N_t^{\mathrm{bank}}/2$. The undesirable event is
\begin{equation}
    \mathcal{E}_t
    :=
    \left\{
        M_t \ge \frac{N_t^{\mathrm{bank}}}{2}
    \right\}.
    \label{eq:bad_bank_event}
\end{equation}

\FedLNS{} uses random client participation, where the server samples participating clients, in each communication round, uniformly at random from the population of $K$ clients. Because the bank stores \emph{distinct} client IDs, once a client has appeared, the bank holds that client's most recent signature. 
Conditioned on the bank's membership size $N_t^{\mathrm{bank}} = b$, the exchangeability of the sampling process implies that the bank can be treated as a uniformly random subset of size $b$ drawn from $K$ clients. Particularly, the server's sampling rule does not prefer any client ID; any two subsets of the same size are equally likely after conditioning on the number of distinct observed clients. Since the malicious set $\mathcal{C}_{\mathrm{mal}}$ is fixed at the start of the run, the question becomes how many of a random size-$b$ subset, drawn without replacement from $K$ clients, belong to a fixed subset of size $\alpha K$, which is the hypergeometric setting. Hence, conditioned on $N_t^{\mathrm{bank}}=b$, we have
\begin{equation}
    M_t \mid N_t^{\mathrm{bank}}=b
    \sim
    \operatorname{Hypergeometric}(K,\alpha K,b),
\end{equation}
i.e., for feasible $m$, 
\begin{equation}
    \Pr(M_t=m \mid N_t^{\mathrm{bank}}=b)
    =
    \frac{
        \binom{\alpha K}{m}
        \binom{(1-\alpha)K}{b-m}
    }{
        \binom{K}{b}
    }.
    \label{eq:hypergeom_pmf}
\end{equation}
The conditional expectation is $\mathbb{E}[M_t \mid N_t^{\mathrm{bank}}=b] = \alpha b$, so when $\alpha<1/2$ the expected malicious fraction in the bank stays below the median's $50\%$ breakdown point. 

\subsection{Concentration of the Malicious Fraction in the Bank}

For a hypergeometric random variable, a Chernoff-type upper-tail bound gives, for any threshold $\gamma>\alpha$,
\begin{equation}
    \Pr\big(
        \frac{M_t}{N_t^{\mathrm{bank}}}
        \ge \gamma
        \,\big|\,
        N_t^{\mathrm{bank}}
    \big)
    \le
    \exp\left(
        -N_t^{\mathrm{bank}} D(\gamma\|\alpha)
    \right),
    \label{eq:hypergeom_chernoff_general}
\end{equation}
where $D(p\|q)$ is the Bernoulli KL divergence
\begin{equation}
    D(p\|q)
    =
    p\log\frac{p}{q}
    +
    (1-p)\log\frac{1-p}{1-q}.
    \label{eq:bernoulli_kl}
\end{equation}
Setting $\gamma=1/2$ and using $\alpha<1/2$ gives
\begin{equation}
    \Pr\left(
        M_t \ge \frac{N_t^{\mathrm{bank}}}{2}
        \,\middle|\,
        N_t^{\mathrm{bank}}
    \right)
    \le
    \exp\left(
        -N_t^{\mathrm{bank}}
        D\left(
            \frac{1}{2}\middle\|\alpha
        \right)
    \right),
    \label{eq:bad_bank_kl_bound}
\end{equation}
with
\begin{equation}
    D\left(
        \frac{1}{2}\middle\|\alpha
    \right)
    =
    \frac{1}{2}\log\frac{1/2}{\alpha}
    +
    \frac{1}{2}\log\frac{1/2}{1-\alpha}
    \label{eq:kl_half_alpha_expanded}
\end{equation}
or, equivalently,
\begin{equation}
    D\left(
        \frac{1}{2}\middle\|\alpha
    \right)
    =
    \frac{1}{2}
    \log
    \frac{1}{4\alpha(1-\alpha)}.
    \label{eq:kl_half_alpha_compact}
\end{equation}
This quantity is positive for $\alpha<1/2$, so the upper bound on the probability of a malicious-majority bank decays exponentially with $N_t^{\mathrm{bank}}$. This provides the probabilistic motivation for delaying screening until the history bank reaches a sufficiently broad client coverage under random participation.

\begin{proposition}[Bound on a malicious-majority history bank]
\label{prop:bad_bank_simple_bound}
Suppose that the fraction of malicious clients satisfies $\alpha<1/2$, and let $M_t$ denote the number of malicious signatures in a history bank of size $N_t^{\mathrm{bank}}$. Then the probability that malicious signatures constitute at least half of the bank is bounded by
\begin{equation}
    \Pr\left(
        M_t \ge \frac{N_t^{\mathrm{bank}}}{2}
        \,\middle|\,
        N_t^{\mathrm{bank}}
    \right)
    \le
    \exp\left(
        -2N_t^{\mathrm{bank}}
        \left(
            \frac{1}{2}-\alpha
        \right)^2
    \right).
    \label{eq:bad_bank_simple_bound}
\end{equation}
\end{proposition}

\begin{proof}
From the KL-based concentration bound in \eqref{eq:bad_bank_kl_bound}, we have
\begin{equation}
    \Pr\left(
        M_t \ge \frac{N_t^{\mathrm{bank}}}{2}
        \,\middle|\,
        N_t^{\mathrm{bank}}
    \right)
    \le
    \exp\left(
        -N_t^{\mathrm{bank}}
        D\left(
            \frac{1}{2}\middle\|\alpha
        \right)
    \right).
\end{equation}
To obtain a simpler sufficient bound, let
\begin{equation}
    g := \frac{1}{2}-\alpha
\end{equation}
be the margin between the median breakdown point and the attack fraction. Then
\begin{equation}
    \alpha(1-\alpha)
    =
    \left(\frac{1}{2}-g\right)
    \left(\frac{1}{2}+g\right)
    =
    \frac{1}{4}-g^2.
\end{equation}
Therefore,
\begin{equation}
    4\alpha(1-\alpha)
    =
    1-4g^2.
\end{equation}
Using \eqref{eq:kl_half_alpha_compact}, we obtain
\begin{equation}
    D\left(
        \frac{1}{2}\middle\|\alpha
    \right)
    =
    \frac{1}{2}
    \log\frac{1}{1-4g^2}.
\end{equation}
Since $\log\frac{1}{1-x} \ge x$ for $0\le x<1$, it follows that
\begin{equation}
    D\left(
        \frac{1}{2}\middle\|\alpha
    \right)
    \ge
    \frac{1}{2}\cdot 4g^2
    =
    2g^2
    =
    2\left(
        \frac{1}{2}-\alpha
    \right)^2.
    \label{eq:kl_lower_bound}
\end{equation}
Substituting \eqref{eq:kl_lower_bound} into \eqref{eq:bad_bank_kl_bound} yields
\begin{equation}
    \Pr\left(
        M_t \ge \frac{N_t^{\mathrm{bank}}}{2}
        \,\middle|\,
        N_t^{\mathrm{bank}}
    \right)
    \le
    \exp\left(
        -2N_t^{\mathrm{bank}}
        \left(
            \frac{1}{2}-\alpha
        \right)^2
    \right),
\end{equation}
which proves \eqref{eq:bad_bank_simple_bound}.
\end{proof}

The bound in Proposition~\ref{prop:bad_bank_simple_bound} makes the dependence on attack strength explicit. As $\alpha\to1/2$, the margin $g=1/2-\alpha$ shrinks, the exponent weakens, and a larger history bank is required to achieve the same confidence level. In contrast, when $\alpha\ll1/2$, the same bank size already provides a strong guarantee against a malicious-majority bank.


Suppose screening must begin before full population coverage is available and the desired probability of a malicious-majority activation bank is at most $\eta$. From \eqref{eq:bad_bank_kl_bound}, it suffices to choose $M_{\mathrm{act}}$ such that
$\exp(-M_{\mathrm{act}}D(1/2\|\alpha))\le\eta$,
\begin{equation}
    M_{\mathrm{act}}
    \ge
    \frac{
        \log(1/\eta)
    }{
        D\left(
            \frac{1}{2}\middle\|\alpha
        \right)
    }.
    \label{eq:mmin_exact_guideline}
\end{equation}
Using the lower bound \eqref{eq:kl_lower_bound}, a more conservative sufficient condition is
\begin{equation}
    M_{\mathrm{act}}
    \ge
    \frac{
        \log(1/\eta)
    }{
        2\left(
            \frac{1}{2}-\alpha
        \right)^2
    }.
    \label{eq:mmin_simple_guideline}
\end{equation}
When the total number of clients $K$ is fixed, the activation coverage can be expressed as a population fraction $M_{\mathrm{act}}=\rho K$, in which case \eqref{eq:bad_bank_kl_bound} becomes
\begin{equation}
    \Pr\left(
        M_t \ge \frac{N_t^{\mathrm{bank}}}{2}
        \,\middle|\,
        N_t^{\mathrm{bank}}\ge \rho K
    \right)
    \le
    \exp\left(
        -\rho K
        D\left(
            \frac{1}{2}\middle\|\alpha
        \right)
    \right).
    \label{eq:rho_fraction_bound}
\end{equation}

\subsection{Time to Reach the Activation Coverage}

We now estimate how quickly the bank reaches a chosen activation coverage $M_{\mathrm{act}}$ under uniform random participation. Let $q$ be the per-round participation rate, so that in each round approximately $qK$ clients are selected uniformly at random. For a fixed client, the probability of not being selected in one round is $1-q$, so the probability of not appearing after $r$ rounds is $(1-q)^r$, and the probability of having appeared at least once by round $r$ is $1-(1-q)^r$. The expected number of distinct clients in the bank after $r$ rounds is therefore
\begin{equation}
    \mathbb{E}[B_r]
    =
    K\left(
        1-(1-q)^r
    \right).
    \label{eq:expected_bank_size}
\end{equation}
For a target coverage fraction $\rho$, the expected bank size reaches $\rho K$ once $K(1-(1-q)^r) \ge \rho K$, i.e., $(1-q)^r \le 1-\rho$, which gives
\begin{equation}
    r
    \ge
    \frac{
        \log(1-\rho)
    }{
        \log(1-q)
    }.
    \label{eq:round_to_reach_fraction}
\end{equation}

\begin{table}[t]
\centering
\caption{Expected round at which the bank reaches each coverage threshold under random participation with $K=200$ and participation rate $q=0.1$.}
\label{tab:bank_reach_rounds}
\begin{tabular}{c|c|c}
\hline
Bank coverage $\rho$ & $M_{\mathrm{act}}=\rho K$ & Expected activation round \\
\hline
$25\%$ & $50$  & $3$ \\
$50\%$ & $100$ & $7$ \\
$75\%$ & $150$ & $14$ \\
\hline
\end{tabular}
\end{table}

In our experiments, where $q=0.1$ and $K=200$, Table~\ref{tab:bank_reach_rounds} presents the expected round at which the bank reaches each coverage threshold.
This expectation can be complemented with a tail bound on the probability that the bank has \emph{not} yet reached the threshold. Let $\mu_r := \mathbb{E}[B_r] = K(1-(1-q)^r)$. Because the coverage indicators are exchangeable and negatively associated under uniform random sampling, a Chernoff-style lower-tail bound gives
\begin{equation}
    \Pr(B_r < M_{\mathrm{act}})
    \le
    \exp\big(
        -\frac{\mu_r \delta_r^2}{2}
    \big),
    \;
    \delta_r
    :=
    1-\frac{M_{\mathrm{act}}}{\mu_r},
    \label{eq:bank_reach_tail_bound}
\end{equation}
whenever $\mu_r>M_{\mathrm{act}}$. For example, with $K=200$, $q=0.1$, and $M_{\mathrm{act}}=100$, after $r=12$ rounds $\mu_{12} = 200(1-0.9^{12}) \approx 143.5$, so $\delta_{12} = 1-100/143.5 \approx 0.303$ and
\begin{equation}
    \Pr(B_{12}<100)
    \le
    \exp\big(
        -\frac{143.5\cdot 0.303^2}{2}
    \big)
    \approx
    1.4\times 10^{-3}.
\end{equation}
Under the participation rate used in our experiments, the bank therefore reaches the half-population threshold early with high probability.

The activation-time and bank-composition analyses above assume uniform random client participation, matching the sampling procedure used in the experiments. In deployments with nonuniform or intermittent participation, some client identities may appear rarely or never within the training horizon, making full bank coverage impractically slow. The hypergeometric and activation-time bounds do not directly apply to such participation processes. In that setting, $M_{\mathrm{act}}$ should be interpreted as a practical early-activation coverage, while the bank should continue to expand whenever previously unseen clients participate.

\clearpage
\section{Proofs for Theoretical Analysis}
\label{app:theory}

This appendix provides the assumptions and proofs supporting Section~\ref{sec_theory}. Appendix~\ref{app:ln_margin_proof} proves the normalization-parameter margin decomposition, Appendix~\ref{app:filtered_convergence} proves the screened-aggregation and stationarity bounds, Appendix~\ref{app:filtering_error_decomposition} decomposes the screening error, and Appendix~\ref{app:local_drift_bound} bounds local-training drift.

\subsection{Proof of the Normalization-Parameter Margin Decomposition}
\label{app:ln_margin_proof}

This proof establishes the first-order contribution of a client's normalization-parameter update to output-token margins. We proceed in three steps.

First, we analyze the margin change for one fixed prompt $x$. 
Second, we split the client update into two parts: the normalization-parameter part and the remaining-parameter part. 
Third, we average over prompts $x\sim\mathcal{X}$, because the expected margin change in \eqref{eq:expected_margin_shift_theory} is defined over a prompt distribution.

We use two remainder terms. 
The term
\begin{equation}
\mathcal{R}_{t,i}(x,u,v)
\end{equation}
is the \emph{pointwise Taylor remainder}. 
It is the Taylor approximation error for one fixed prompt $x$ and token pair $(u,v)$. 
The term
\begin{equation}
R_{t,i}(u,v)
\end{equation}
is the \emph{prompt-averaged remainder}, defined by
\begin{equation}
R_{t,i}(u,v)
:=
\mathbb{E}_{x\sim\mathcal{X}}
\left[
\mathcal{R}_{t,i}(x,u,v)
\right].
\label{eq:expected_remainder_definition_app}
\end{equation}
Thus, $\mathcal{R}_{t,i}(x,u,v)$ depends on a specific prompt, while $R_{t,i}(u,v)$ averages this error over prompts.

\begin{assumption}[Smooth prompt-dependent margins]
\label{assump:smooth_margin}
Fix a token pair $(u,v)\in\mathcal{V}\times\mathcal{V}$. 
For every prompt $x$ in the support of $\mathcal{X}$, the mapping
\begin{equation}
\mathbf{w}
\mapsto
m_{u,v}(x;\mathbf{w})
\end{equation}
is twice continuously differentiable in a neighborhood of $\mathbf{w}_t$. 
Moreover, there exists $L_m>0$ such that
\begin{equation}
\left\|
\nabla_{\mathbf{w}}^{2}
m_{u,v}(x;\mathbf{w})
\right\|_{\mathrm{op}}
\le
L_m
\label{eq:hessian_margin_bound_app}
\end{equation}
for all $\mathbf{w}$ in that neighborhood and all relevant prompts $x$.
\end{assumption}

Conditioned on round $t$ and client $i$, the update
\begin{equation}
\Delta\mathbf{w}_{t,i}
=
(\Delta\mathbf{p}_{t,i},\Delta\mathbf{q}_{t,i})
\end{equation}
is treated as fixed with respect to the prompt draw $x\sim\mathcal{X}$. 
Thus, the expectation over $x$ averages prompt-dependent gradients and remainders, while the update vectors $\Delta\mathbf{p}_{t,i}$ and $\Delta\mathbf{q}_{t,i}$ do not depend on the sampled prompt.

\begin{proof}[Proof of Proposition~\ref{prop:ln_margin_main}]
Fix a prompt $x$ and token pair $(u,v)$. 
Define
\begin{equation}
\phi_x(s)
:=
m_{u,v}
\left(
x;
\mathbf{w}_t+s\Delta\mathbf{w}_{t,i}
\right),
\qquad s\in[0,1].
\end{equation}
This scalar function traces the token margin along the straight path from the current global model $\mathbf{w}_t$ to the locally updated model 
\begin{equation}
\mathbf{w}_t+\Delta\mathbf{w}_{t,i}.
\end{equation}

By the fundamental theorem of calculus,
\begin{equation}
m_{u,v}(x;\mathbf{w}_t+\Delta\mathbf{w}_{t,i})
-
m_{u,v}(x;\mathbf{w}_t)
=
\int_0^1
\phi_x'(s)\,ds.
\label{eq:ftc_margin_app}
\end{equation}
Using the chain rule,
\begin{equation}
\phi_x'(s)
=
\nabla_{\mathbf{w}}
m_{u,v}
\left(
x;
\mathbf{w}_t+s\Delta\mathbf{w}_{t,i}
\right)^{\top}
\Delta\mathbf{w}_{t,i}.
\end{equation}
Therefore,
\begin{equation}
m_{u,v}(x;\mathbf{w}_t+\Delta\mathbf{w}_{t,i})
-
m_{u,v}(x;\mathbf{w}_t)
\nonumber=
\int_0^1
\nabla_{\mathbf{w}}
m_{u,v}
\left(
x;
\mathbf{w}_t+s\Delta\mathbf{w}_{t,i}
\right)^{\top}
\Delta\mathbf{w}_{t,i}
\,ds.
\label{eq:integral_margin_app}
\end{equation}

We now isolate the first-order effect at the current global model. 
Add and subtract 
\begin{equation}
\nabla_{\mathbf{w}}m_{u,v}(x;\mathbf{w}_t)
\end{equation}
inside the integral:
\begin{equation}
m_{u,v}(x;\mathbf{w}_t+\Delta\mathbf{w}_{t,i})
-
m_{u,v}(x;\mathbf{w}_t)
=
\nabla_{\mathbf{w}}m_{u,v}(x;\mathbf{w}_t)^{\top}
\Delta\mathbf{w}_{t,i} 
+
\mathcal{R}_{t,i}(x,u,v),
\label{eq:first_order_margin_app}
\end{equation}
where the pointwise Taylor remainder is
\begin{equation}
\mathcal{R}_{t,i}(x,u,v)
:=
\int_0^1
\Big[
\nabla_{\mathbf{w}}
m_{u,v}
\left(
x;
\mathbf{w}_t+s\Delta\mathbf{w}_{t,i}
\right)
-
\nabla_{\mathbf{w}}m_{u,v}(x;\mathbf{w}_t)
\Big]^{\top}
\Delta\mathbf{w}_{t,i}
\,ds.
\label{eq:remainder_integral_app}
\end{equation}
This term contains the part of the margin change not captured by the linear approximation at $\mathbf{w}_t$.

Equivalently, Taylor's theorem with integral remainder gives
\begin{equation}
\mathcal{R}_{t,i}(x,u,v)
=
\int_0^1
(1-s)
\Delta\mathbf{w}_{t,i}^{\top}
\nabla_{\mathbf{w}}^{2}
m_{u,v}
\left(
x;
\mathbf{w}_t+s\Delta\mathbf{w}_{t,i}
\right)
\Delta\mathbf{w}_{t,i}
\,ds.
\label{eq:remainder_hessian_app}
\end{equation}
Using Assumption~\ref{assump:smooth_margin},
\begin{align}
|\mathcal{R}_{t,i}(x,u,v)|
&\le
\int_0^1
(1-s)
\left\|
\nabla_{\mathbf{w}}^{2}
m_{u,v}
\left(
x;
\mathbf{w}_t+s\Delta\mathbf{w}_{t,i}
\right)
\right\|_{\mathrm{op}}
\left\|
\Delta\mathbf{w}_{t,i}
\right\|_2^2
\,ds
\nonumber\\
&\le
\int_0^1
(1-s)L_m
\left\|
\Delta\mathbf{w}_{t,i}
\right\|_2^2
\,ds
\nonumber\\
&=
\frac{L_m}{2}
\left\|
\Delta\mathbf{w}_{t,i}
\right\|_2^2.
\label{eq:remainder_bound_app}
\end{align}

Next, split the first-order term into the normalization-parameter part and the remaining-parameter part. 
Since
\begin{equation}
\mathbf{w}
=
(\mathbf{p},\mathbf{q}),
\qquad
\Delta\mathbf{w}_{t,i}
=
(\Delta\mathbf{p}_{t,i},\Delta\mathbf{q}_{t,i}),
\end{equation}
we have
\begin{equation}
\nabla_{\mathbf{w}}m_{u,v}(x;\mathbf{w}_t)^{\top}
\Delta\mathbf{w}_{t,i}
=
\nabla_{\mathbf{p}}m_{u,v}(x;\mathbf{w}_t)^{\top}
\Delta\mathbf{p}_{t,i}
+\nabla_{\mathbf{q}}m_{u,v}(x;\mathbf{w}_t)^{\top}
\Delta\mathbf{q}_{t,i}.
\label{eq:split_gradient_app}
\end{equation}

Substituting \eqref{eq:split_gradient_app} into \eqref{eq:first_order_margin_app} gives the pointwise decomposition
\begin{equation}
m_{u,v}(x;\mathbf{w}_t+\Delta\mathbf{w}_{t,i})
-
m_{u,v}(x;\mathbf{w}_t)
=
\nabla_{\mathbf{p}}m_{u,v}(x;\mathbf{w}_t)^{\top}
\Delta\mathbf{p}_{t,i} 
+
\nabla_{\mathbf{q}}m_{u,v}(x;\mathbf{w}_t)^{\top}
\Delta\mathbf{q}_{t,i}
+
\mathcal{R}_{t,i}(x,u,v).
\label{eq:pointwise_ln_decomposition_app}
\end{equation}

We now take expectation over $x\sim\mathcal{X}$. 
This step is necessary because $\Delta\bar{m}_{t,i}(u,v)$ in \eqref{eq:expected_margin_shift_theory} is the average margin change over prompts. 
Taking expectation on both sides of \eqref{eq:pointwise_ln_decomposition_app} gives
\begin{equation}
\Delta\bar{m}_{t,i}(u,v)
=
\mathbb{E}_{x\sim\mathcal{X}}
\left[
\nabla_{\mathbf{p}}m_{u,v}(x;\mathbf{w}_t)^{\top}
\Delta\mathbf{p}_{t,i}
\right]
+
\mathbb{E}_{x\sim\mathcal{X}}
\left[
\nabla_{\mathbf{q}}m_{u,v}(x;\mathbf{w}_t)^{\top}
\Delta\mathbf{q}_{t,i}
\right]
+
\mathbb{E}_{x\sim\mathcal{X}}
\left[
\mathcal{R}_{t,i}(x,u,v)
\right].
\label{eq:expectation_before_pullout_app}
\end{equation}
Because $\Delta\mathbf{p}_{t,i}$ and $\Delta\mathbf{q}_{t,i}$ are fixed with respect to the prompt draw $x$, they can be moved outside the expectation. 
For the normalization-parameter part, if $\Delta\mathbf{p}_{t,i}\in\mathbb{R}^{d_p}$, then
\begin{align}
\mathbb{E}_{x\sim\mathcal{X}}
\left[
\nabla_{\mathbf{p}}m_{u,v}(x;\mathbf{w}_t)^{\top}
\Delta\mathbf{p}_{t,i}
\right]
\nonumber
&=
\mathbb{E}_{x\sim\mathcal{X}}
\left[
\sum_{r=1}^{d_p}
\frac{\partial m_{u,v}(x;\mathbf{w}_t)}
{\partial \mathbf{p}[r]}
\Delta\mathbf{p}_{t,i}[r]
\right]
\nonumber\\
&=
\sum_{r=1}^{d_p}
\mathbb{E}_{x\sim\mathcal{X}}
\left[
\frac{\partial m_{u,v}(x;\mathbf{w}_t)}
{\partial \mathbf{p}[r]}
\right]
\Delta\mathbf{p}_{t,i}[r]
\nonumber\\
&=
\mathbb{E}_{x\sim\mathcal{X}}
\left[
\nabla_{\mathbf{p}}m_{u,v}(x;\mathbf{w}_t)
\right]^{\top}
\Delta\mathbf{p}_{t,i}.
\label{eq:pullout_p_app}
\end{align}
The same argument gives
\begin{equation}
\mathbb{E}_{x\sim\mathcal{X}}
\left[
\nabla_{\mathbf{q}}m_{u,v}(x;\mathbf{w}_t)^{\top}
\Delta\mathbf{q}_{t,i}
\right]
=
\mathbb{E}_{x\sim\mathcal{X}}
\left[
\nabla_{\mathbf{q}}m_{u,v}(x;\mathbf{w}_t)
\right]^{\top}
\Delta\mathbf{q}_{t,i}.
\label{eq:pullout_q_app}
\end{equation}

Define
\begin{equation}
\mathbf{g}_{p,u,v}
:=
\mathbb{E}_{x\sim\mathcal{X}}
\left[
\nabla_{\mathbf{p}}m_{u,v}(x;\mathbf{w}_t)
\right],
\end{equation}
\begin{equation}
\mathbf{g}_{q,u,v}
:=
\mathbb{E}_{x\sim\mathcal{X}}
\left[
\nabla_{\mathbf{q}}m_{u,v}(x;\mathbf{w}_t)
\right],
\end{equation}
and
\begin{equation}
R_{t,i}(u,v)
:=
\mathbb{E}_{x\sim\mathcal{X}}
\left[
\mathcal{R}_{t,i}(x,u,v)
\right].
\end{equation}
Using these definitions in \eqref{eq:expectation_before_pullout_app} yields
\begin{equation}
\Delta\bar{m}_{t,i}(u,v)
=
\mathbf{g}_{p,u,v}^{\top}\Delta\mathbf{p}_{t,i}
+
\mathbf{g}_{q,u,v}^{\top}\Delta\mathbf{q}_{t,i}
+
R_{t,i}(u,v).
\end{equation}

Finally, using \eqref{eq:remainder_bound_app},
\begin{align}
|R_{t,i}(u,v)|
&=
\left|
\mathbb{E}_{x\sim\mathcal{X}}
\left[
\mathcal{R}_{t,i}(x,u,v)
\right]
\right|
\nonumber\\
&\le
\mathbb{E}_{x\sim\mathcal{X}}
\left[
\left|
\mathcal{R}_{t,i}(x,u,v)
\right|
\right]
\nonumber\\
&\le
\frac{L_m}{2}
\left\|
\Delta\mathbf{w}_{t,i}
\right\|_2^2.
\end{align}
This proves Proposition~\ref{prop:ln_margin_main}.
\end{proof}

The following corollary bounds the difference between two clients' first-order margin changes using their normalization-parameter and remaining-parameter update differences, together with the second-order remainders.

\begin{corollary}[Stability of margin changes]
\label{cor:margin_pair_stability}
Under Assumption~\ref{assump:smooth_margin}, for any two clients $i$ and $j$,
\begin{equation}
\left|
\Delta\bar{m}_{t,i}(u,v)
-
\Delta\bar{m}_{t,j}(u,v)
\right|
\le
\|\mathbf{g}_{p,u,v}\|_2
\left\|
\Delta\mathbf{p}_{t,i}
-
\Delta\mathbf{p}_{t,j}
\right\|_2
+
\|\mathbf{g}_{q,u,v}\|_2
\left\|
\Delta\mathbf{q}_{t,i}
-
\Delta\mathbf{q}_{t,j}
\right\|_2
+
\frac{L_m}{2}
\left(
\left\|
\Delta\mathbf{w}_{t,i}
\right\|_2^2
+
\left\|
\Delta\mathbf{w}_{t,j}
\right\|_2^2
\right).
\label{eq:margin_pair_stability_app}
\end{equation}
\end{corollary}

\begin{proof}
Subtract \eqref{eq:ln_margin_decomposition_main} for clients $i$ and $j$:
\begin{equation}
\Delta\bar{m}_{t,i}(u,v)-\Delta\bar{m}_{t,j}(u,v)
=
\mathbf{g}_{p,u,v}^{\top}
\left(
\Delta\mathbf{p}_{t,i}
-
\Delta\mathbf{p}_{t,j}
\right)
+
\mathbf{g}_{q,u,v}^{\top}
\left(
\Delta\mathbf{q}_{t,i}
-
\Delta\mathbf{q}_{t,j}
\right)+
R_{t,i}(u,v)-R_{t,j}(u,v).
\end{equation}
Taking absolute values and applying Cauchy--Schwarz gives
\begin{equation}
\left|
\Delta\bar{m}_{t,i}(u,v)-\Delta\bar{m}_{t,j}(u,v)
\right|
\le
\|\mathbf{g}_{p,u,v}\|_2
\left\|
\Delta\mathbf{p}_{t,i}
-
\Delta\mathbf{p}_{t,j}
\right\|_2
+
\|\mathbf{g}_{q,u,v}\|_2
\left\|
\Delta\mathbf{q}_{t,i}
-
\Delta\mathbf{q}_{t,j}
\right\|_2
+
|R_{t,i}(u,v)|
+
|R_{t,j}(u,v)|.
\end{equation}
Using \eqref{eq:ln_margin_remainder_main} for both remainders proves the claim.
\end{proof}

\subsection{Proof of the Screened-Aggregation Bound}
\label{app:filtered_convergence}

We now prove Theorem~\ref{thm:filtered_convergence_main}. 
The proof shows how the screening error $\mathbf{e}_t$ enters the descent inequality. 
If the retained update direction is close to the benign gradient direction, then $\|\mathbf{e}_t\|_2$ is small, and the update behaves similarly to benign training.

\begin{assumption}[Benign objective]
\label{assump:benign_objective}
The target objective is
\begin{equation}
F(\mathbf{w})
=
\sum_{i\in\mathcal{C}_{\mathrm{ben}}}
\lambda_iF_i(\mathbf{w}),
\qquad
\lambda_i\ge 0,\quad
\sum_{i\in\mathcal{C}_{\mathrm{ben}}}\lambda_i=1.
\end{equation}
The function $F$ is differentiable, $L_F$-smooth, and lower bounded:
\begin{equation}
F(\mathbf{w})\ge F_{\inf}>-\infty.
\end{equation}
\end{assumption}

The $L_F$-smoothness condition means that, for all $\mathbf{w}$ and $\mathbf{w}'$,
\begin{equation}
F(\mathbf{w}')
\le
F(\mathbf{w})
+
\nabla F(\mathbf{w})^{\top}
(\mathbf{w}'-\mathbf{w})
+
\frac{L_F}{2}
\left\|
\mathbf{w}'-\mathbf{w}
\right\|_2^2.
\label{eq:smoothness_app}
\end{equation}

\begin{assumption}[Screened update model]
\label{assump:filtered_update_model}
At round $t$, the screened update can be written as
\begin{equation}
\mathbf{w}_{t+1}
=
\mathbf{w}_t-\eta_t\mathbf{g}_t,
\label{eq:filtered_update_model_app}
\end{equation}
where $\eta_t\ge 0$ is the effective step size and $\mathbf{g}_t$ is the retained update direction. 
Let $\mathcal{F}_t$ contain the information available before the update noise at round $t$ is realized. 
There exist $\beta_t\ge 0$ and $\nu_t\ge 0$ such that
\begin{equation}
\left\|
\mathbb{E}[\mathbf{g}_t\mid\mathcal{F}_t]
-
\nabla F(\mathbf{w}_t)
\right\|_2
\le
\beta_t,
\label{eq:error_assumption_app}
\end{equation}
and
\begin{equation}
\mathbb{E}
\left[
\left\|
\mathbf{g}_t-
\mathbb{E}[\mathbf{g}_t\mid\mathcal{F}_t]
\right\|_2^2
\mid
\mathcal{F}_t
\right]
\le
\nu_t^2.
\label{eq:variance_assumption_app}
\end{equation}
\end{assumption}

\begin{proof}[Proof of Theorem~\ref{thm:filtered_convergence_main}]
By smoothness \eqref{eq:smoothness_app} and the update rule \eqref{eq:filtered_update_model_app},
\begin{align}
F(\mathbf{w}_{t+1})
&\le
F(\mathbf{w}_t)
-
\eta_t
\nabla F(\mathbf{w}_t)^{\top}
\mathbf{g}_t
+
\frac{L_F\eta_t^2}{2}
\|\mathbf{g}_t\|_2^2.
\label{eq:one_step_smooth_app}
\end{align}
Take conditional expectation given $\mathcal{F}_t$. 
Define
\begin{equation}
\mathbf{a}_t
:=
\nabla F(\mathbf{w}_t),
\qquad
\bar{\mathbf{g}}_t
:=
\mathbb{E}[\mathbf{g}_t\mid\mathcal{F}_t],
\qquad
\mathbf{e}_t
:=
\bar{\mathbf{g}}_t-\mathbf{a}_t.
\end{equation}
Then
\begin{equation}
\bar{\mathbf{g}}_t
=
\mathbf{a}_t+\mathbf{e}_t,
\qquad
\|\mathbf{e}_t\|_2\le\beta_t.
\end{equation}

The first-order term satisfies
\begin{align}
-\eta_t\mathbf{a}_t^{\top}\bar{\mathbf{g}}_t
&=
-\eta_t\mathbf{a}_t^{\top}
(\mathbf{a}_t+\mathbf{e}_t)
\nonumber\\
&=
-\eta_t\|\mathbf{a}_t\|_2^2
-
\eta_t\mathbf{a}_t^{\top}\mathbf{e}_t.
\end{align}
By Young's inequality,
\begin{equation}
|\mathbf{a}_t^{\top}\mathbf{e}_t|
\le
\frac{1}{4}\|\mathbf{a}_t\|_2^2
+
\|\mathbf{e}_t\|_2^2.
\end{equation}
Therefore,
\begin{equation}
-\eta_t\mathbf{a}_t^{\top}\bar{\mathbf{g}}_t
\le
-\frac{3\eta_t}{4}\|\mathbf{a}_t\|_2^2
+
\eta_t\|\mathbf{e}_t\|_2^2.
\label{eq:first_order_bound_app}
\end{equation}

Next,
\begin{align}
\mathbb{E}
\left[
\|\mathbf{g}_t\|_2^2
\mid
\mathcal{F}_t
\right]
&=
\|\bar{\mathbf{g}}_t\|_2^2
+
\mathbb{E}
\left[
\|\mathbf{g}_t-\bar{\mathbf{g}}_t\|_2^2
\mid
\mathcal{F}_t
\right]
\nonumber\\
&\le
\|\mathbf{a}_t+\mathbf{e}_t\|_2^2+\nu_t^2
\nonumber\\
&\le
2\|\mathbf{a}_t\|_2^2
+
2\|\mathbf{e}_t\|_2^2
+
\nu_t^2.
\label{eq:second_moment_bound_app}
\end{align}
Substituting \eqref{eq:first_order_bound_app} and \eqref{eq:second_moment_bound_app} into \eqref{eq:one_step_smooth_app} gives
\begin{equation}
\mathbb{E}
\left[
F(\mathbf{w}_{t+1})
\mid
\mathcal{F}_t
\right]
\le
F(\mathbf{w}_t)
-
\left(
\frac{3\eta_t}{4}
-
L_F\eta_t^2
\right)
\|\mathbf{a}_t\|_2^2
+
(\eta_t+L_F\eta_t^2)
\|\mathbf{e}_t\|_2^2
+
\frac{L_F\eta_t^2}{2}\nu_t^2.
\label{eq:descent_before_stepsize_app}
\end{equation}
If $0\le\eta_t\le 1/(4L_F)$, then
\begin{equation}
\frac{3\eta_t}{4}-L_F\eta_t^2
\ge
\frac{\eta_t}{2}.
\end{equation}
Using $\|\mathbf{e}_t\|_2^2\le\beta_t^2$, we obtain
\begin{equation}
\mathbb{E}
\left[
F(\mathbf{w}_{t+1})
\mid
\mathcal{F}_t
\right]
\le
F(\mathbf{w}_t)
-
\frac{\eta_t}{2}
\|\nabla F(\mathbf{w}_t)\|_2^2
+
(\eta_t+L_F\eta_t^2)\beta_t^2
+
\frac{L_F\eta_t^2}{2}\nu_t^2.
\label{eq:one_step_descent_app}
\end{equation}
Taking total expectation and summing from $t=0$ to $T-1$ yields
\begin{equation}
\frac{1}{2}
\sum_{t=0}^{T-1}
\eta_t
\mathbb{E}
\left[
\|\nabla F(\mathbf{w}_t)\|_2^2
\right]
\le
F(\mathbf{w}_0)
-
\mathbb{E}[F(\mathbf{w}_T)]
+
\sum_{t=0}^{T-1}
(\eta_t+L_F\eta_t^2)
\mathbb{E}[\beta_t^2]
+
\frac{L_F}{2}
\sum_{t=0}^{T-1}
\eta_t^2
\mathbb{E}[\nu_t^2].
\end{equation}
Since $F(\mathbf{w}_T)\ge F_{\inf}$,
\begin{equation}
\frac{1}{2}
\sum_{t=0}^{T-1}
\eta_t
\mathbb{E}
\left[
\|\nabla F(\mathbf{w}_t)\|_2^2
\right]
\le
F(\mathbf{w}_0)-F_{\inf}
\quad+
\sum_{t=0}^{T-1}
(\eta_t+L_F\eta_t^2)
\mathbb{E}[\beta_t^2]
\quad+
\frac{L_F}{2}
\sum_{t=0}^{T-1}
\eta_t^2
\mathbb{E}[\nu_t^2].
\end{equation}
Dividing both sides by $\frac{1}{2}S_T$ proves \eqref{eq:filtered_convergence_bound_main}.
\end{proof}

\begin{proof}[Proof of Corollary~\ref{cor:stationarity_main}]
From Theorem~\ref{thm:filtered_convergence_main},
\begin{equation}
\frac{
\sum_{t=0}^{T-1}
\eta_t
\mathbb{E}
\left[
\|\nabla F(\mathbf{w}_t)\|_2^2
\right]
}{
S_T
}
\le
\frac{
2\left(F(\mathbf{w}_0)-F_{\inf}\right)
}{
S_T
}
+
\frac{
2\sum_{t=0}^{T-1}
(\eta_t+L_F\eta_t^2)
\mathbb{E}[\beta_t^2]
}{
S_T
}
+
\frac{
L_F\sum_{t=0}^{T-1}
\eta_t^2
\mathbb{E}[\nu_t^2]
}{
S_T
}.
\label{eq:stationarity_start_app}
\end{equation}
We show that each term on the right-hand side goes to zero.

First, since $S_T\to\infty$ and $F(\mathbf{w}_0)-F_{\inf}$ is finite,
\begin{equation}
\frac{
2\left(F(\mathbf{w}_0)-F_{\inf}\right)
}{
S_T
}
\to 0.
\label{eq:first_term_vanish_app}
\end{equation}

Second, because $\eta_t\le 1/(4L_F)$,
\begin{equation}
L_F\eta_t^2
\le
\frac{1}{4}\eta_t,
\end{equation}
and therefore
\begin{equation}
\eta_t+L_F\eta_t^2
\le
\frac{5}{4}\eta_t.
\end{equation}
Thus,
\begin{align}
\frac{
\sum_{t=0}^{T-1}
(\eta_t+L_F\eta_t^2)
\mathbb{E}[\beta_t^2]
}{
S_T
}
&\le
\frac{5}{4}
\frac{
\sum_{t=0}^{T-1}
\eta_t
\mathbb{E}[\beta_t^2]
}{
S_T
}.
\end{align}
By \eqref{eq:weighted_filtering_error_condition},
\begin{equation}
\frac{
\sum_{t=0}^{T-1}
(\eta_t+L_F\eta_t^2)
\mathbb{E}[\beta_t^2]
}{
S_T
}
\to 0.
\label{eq:second_term_vanish_app}
\end{equation}

Third, since $\mathbb{E}[\nu_t^2]\le\nu^2$,
\begin{equation}
\frac{
L_F\sum_{t=0}^{T-1}
\eta_t^2
\mathbb{E}[\nu_t^2]
}{
S_T
}
\le
L_F\nu^2
\frac{
\sum_{t=0}^{T-1}\eta_t^2
}{
S_T
}.
\end{equation}
By \eqref{eq:stepsize_stationarity_condition},
\begin{equation}
L_F\nu^2
\frac{
\sum_{t=0}^{T-1}\eta_t^2
}{
S_T
}
\to 0.
\label{eq:third_term_vanish_app}
\end{equation}

Combining \eqref{eq:first_term_vanish_app}, \eqref{eq:second_term_vanish_app}, and \eqref{eq:third_term_vanish_app} in \eqref{eq:stationarity_start_app} gives
\begin{equation}
\frac{
\sum_{t=0}^{T-1}
\eta_t
\mathbb{E}
\left[
\|\nabla F(\mathbf{w}_t)\|_2^2
\right]
}{
S_T
}
\to 0.
\end{equation}

Now define a random round $\tau$ by
\begin{equation}
\Pr(\tau=t)=\frac{\eta_t}{S_T},
\qquad t=0,\dots,T-1.
\end{equation}
Then, by the law of total expectation,
\begin{align}
\mathbb{E}
\left[
\|\nabla F(\mathbf{w}_{\tau})\|_2^2
\right]
&=
\sum_{t=0}^{T-1}
\Pr(\tau=t)
\mathbb{E}
\left[
\|\nabla F(\mathbf{w}_{t})\|_2^2
\right]
\nonumber\\
&=
\frac{
\sum_{t=0}^{T-1}
\eta_t
\mathbb{E}
\left[
\|\nabla F(\mathbf{w}_{t})\|_2^2
\right]
}{
S_T
}.
\end{align}
Therefore,
\begin{equation}
\mathbb{E}
\left[
\|\nabla F(\mathbf{w}_{\tau})\|_2^2
\right]
\to 0.
\end{equation}
This proves the corollary.
\end{proof}

\begin{proof}[Proof of Corollary~\ref{cor:constant_stationary_neighborhood_main}]
Set $\eta_t=\eta$ in Theorem~\ref{thm:filtered_convergence_main}. 
Then
\begin{equation}
S_T
=
\sum_{t=0}^{T-1}\eta
=
\eta T.
\end{equation}
The left-hand side of \eqref{eq:filtered_convergence_bound_main} becomes
\begin{equation}
\frac{
\sum_{t=0}^{T-1}
\eta
\mathbb{E}
\left[
\|\nabla F(\mathbf{w}_t)\|_2^2
\right]
}{
\eta T
}
=
\frac{1}{T}
\sum_{t=0}^{T-1}
\mathbb{E}
\left[
\|\nabla F(\mathbf{w}_t)\|_2^2
\right].
\end{equation}
The optimization term becomes
\begin{equation}
\frac{
2\left(F(\mathbf{w}_0)-F_{\inf}\right)
}{
\eta T
}.
\end{equation}
The screening-error term becomes
\begin{align}
\frac{
2\sum_{t=0}^{T-1}
(\eta+L_F\eta^2)
\mathbb{E}[\beta_t^2]
}{
\eta T
}
&=
2(1+L_F\eta)
\frac{1}{T}
\sum_{t=0}^{T-1}
\mathbb{E}[\beta_t^2]
\nonumber\\
&=
2(1+L_F\eta)\overline{\beta}_T^2.
\end{align}
The variance term becomes
\begin{align}
\frac{
L_F\sum_{t=0}^{T-1}
\eta^2
\mathbb{E}[\nu_t^2]
}{
\eta T
}
&=
L_F\eta
\frac{1}{T}
\sum_{t=0}^{T-1}
\mathbb{E}[\nu_t^2]
\nonumber\\
&=
L_F\eta\overline{\nu}_T^2.
\end{align}
Substituting these three expressions into \eqref{eq:filtered_convergence_bound_main} proves
\eqref{eq:constant_stationary_neighborhood_main}.
\end{proof}

\subsection{Screening-Error Decomposition}
\label{app:filtering_error_decomposition}

We now connect $\beta_t$ to the retained set returned by \FedLNS{}. 
This decomposition relates the screening error to retained malicious weight, benign-client representativeness, and local-training drift.

Let $\mathcal{C}_{\mathrm{mal}}$ denote the malicious client set and $\mathcal{C}_{\mathrm{ben}}$ denote the benign client set. 
At round $t$, \FedLNS{} retains $\mathcal{R}_t$. 
For each retained client $i\in\mathcal{R}_t$, define
\begin{equation}
\alpha_{t,i}
:=
\frac{n_{t,i}}{\sum_{j\in\mathcal{R}_t}n_{t,j}}.
\end{equation}
The total retained malicious weight is
\begin{equation}
\mu_t
:=
\sum_{a\in\mathcal{R}_t\cap\mathcal{C}_{\mathrm{mal}}}
\alpha_{t,a}.
\label{eq:retained_malicious_mass_app}
\end{equation}

For a retained benign client $i$, define its conditional expected local direction as
\begin{equation}
\bar{\mathbf{g}}_{t,i}^{B}
:=
\mathbb{E}
\left[
\mathbf{g}_{t,i}
\mid
\mathcal{F}_t
\right],
\qquad i\in\mathcal{C}_{\mathrm{ben}}.
\end{equation}
We decompose it as
\begin{equation}
\bar{\mathbf{g}}_{t,i}^{B}
=
\nabla F_i(\mathbf{w}_t)
+
\mathbf{r}_{t,i},
\label{eq:local_drift_decomp_app}
\end{equation}
where $\mathbf{r}_{t,i}$ is the local-training drift vector. 
For a retained malicious client $a$, define
\begin{equation}
\mathbf{h}_{t,a}
:=
\mathbb{E}
\left[
\mathbf{g}_{t,a}
\mid
\mathcal{F}_t
\right],
\qquad a\in\mathcal{C}_{\mathrm{mal}}.
\end{equation}
The vector $\mathbf{h}_{t,a}$ can be arbitrary.

The conditional mean retained direction is
\begin{align}
\mathbb{E}[\mathbf{g}_t\mid\mathcal{F}_t]
&=
\sum_{i\in\mathcal{R}_t\cap\mathcal{C}_{\mathrm{ben}}}
\alpha_{t,i}
\left(
\nabla F_i(\mathbf{w}_t)
+
\mathbf{r}_{t,i}
\right)
+
\sum_{a\in\mathcal{R}_t\cap\mathcal{C}_{\mathrm{mal}}}
\alpha_{t,a}\mathbf{h}_{t,a}.
\label{eq:conditional_retained_direction_app}
\end{align}

Assume $\mu_t<1$. 
Define the retained benign gradient after renormalizing the benign retained weights:
\begin{equation}
\bar{\mathbf{g}}_t^{B,\mathrm{ret}}
:=
\sum_{i\in\mathcal{R}_t\cap\mathcal{C}_{\mathrm{ben}}}
\frac{\alpha_{t,i}}{1-\mu_t}
\nabla F_i(\mathbf{w}_t).
\end{equation}
The retained benign representativeness error is
\begin{equation}
\chi_t
:=
\left\|
\bar{\mathbf{g}}_t^{B,\mathrm{ret}}
-
\nabla F(\mathbf{w}_t)
\right\|_2.
\label{eq:benign_representativeness_app}
\end{equation}
The retained benign local-drift magnitude is
\begin{equation}
\varrho_t
:=
\left\|
\sum_{i\in\mathcal{R}_t\cap\mathcal{C}_{\mathrm{ben}}}
\alpha_{t,i}\mathbf{r}_{t,i}
\right\|_2.
\label{eq:local_drift_magnitude_app}
\end{equation}

\begin{assumption}[Bounded retained directions]
\label{assump:bounded_retained_directions}
There exists $G>0$ such that
\begin{equation}
\|\mathbf{h}_{t,a}\|_2\le G,
\qquad
\|\nabla F(\mathbf{w}_t)\|_2\le G
\end{equation}
for all retained malicious clients $a$ and all rounds considered.
\end{assumption}

\begin{proposition}[Screening-error bound]
\label{prop:filtering_error_bound_app}
Under the definitions above and Assumption~\ref{assump:bounded_retained_directions},
\begin{equation}
\left\|
\mathbb{E}[\mathbf{g}_t\mid\mathcal{F}_t]
-
\nabla F(\mathbf{w}_t)
\right\|_2
\le
(1-\mu_t)\chi_t
+
2G\mu_t
+
\varrho_t.
\label{eq:filtering_error_bound_app}
\end{equation}
Therefore, one may choose
\begin{equation}
\beta_t
=
(1-\mu_t)\chi_t
+
2G\mu_t
+
\varrho_t.
\label{eq:beta_choice_app}
\end{equation}
\end{proposition}

\begin{proof}
Using \eqref{eq:conditional_retained_direction_app},
\begin{equation}
\mathbb{E}[\mathbf{g}_t\mid\mathcal{F}_t]
-
\nabla F(\mathbf{w}_t)
=
\sum_{i\in\mathcal{R}_t\cap\mathcal{C}_{\mathrm{ben}}}
\alpha_{t,i}
\nabla F_i(\mathbf{w}_t)
-
\nabla F(\mathbf{w}_t)
+
\sum_{a\in\mathcal{R}_t\cap\mathcal{C}_{\mathrm{mal}}}
\alpha_{t,a}\mathbf{h}_{t,a}
+
\sum_{i\in\mathcal{R}_t\cap\mathcal{C}_{\mathrm{ben}}}
\alpha_{t,i}\mathbf{r}_{t,i}.
\label{eq:error_decompose_step_app}
\end{equation}
The benign-gradient part satisfies
\begin{align}
\sum_{i\in\mathcal{R}_t\cap\mathcal{C}_{\mathrm{ben}}}
\alpha_{t,i}
\nabla F_i(\mathbf{w}_t)
-
\nabla F(\mathbf{w}_t)
\nonumber
&=
(1-\mu_t)
\bar{\mathbf{g}}_t^{B,\mathrm{ret}}
-
\nabla F(\mathbf{w}_t)
\nonumber\\
&=
(1-\mu_t)
\left(
\bar{\mathbf{g}}_t^{B,\mathrm{ret}}
-
\nabla F(\mathbf{w}_t)
\right)
-
\mu_t\nabla F(\mathbf{w}_t).
\end{align}
Thus,
\begin{equation}
\left\|
\sum_{i\in\mathcal{R}_t\cap\mathcal{C}_{\mathrm{ben}}}
\alpha_{t,i}
\nabla F_i(\mathbf{w}_t)
-
\nabla F(\mathbf{w}_t)
\right\|_2
\le
(1-\mu_t)\chi_t
+
\mu_t\|\nabla F(\mathbf{w}_t)\|_2.
\label{eq:benign_part_bound_app}
\end{equation}
The retained malicious part satisfies
\begin{align}
\left\|
\sum_{a\in\mathcal{R}_t\cap\mathcal{C}_{\mathrm{mal}}}
\alpha_{t,a}\mathbf{h}_{t,a}
\right\|_2
&\le
\sum_{a\in\mathcal{R}_t\cap\mathcal{C}_{\mathrm{mal}}}
\alpha_{t,a}\|\mathbf{h}_{t,a}\|_2
\nonumber\\
&\le
G\mu_t.
\label{eq:malicious_part_bound_app}
\end{align}
The local-drift term is $\varrho_t$ by definition. 
Combining the three bounds and using 
$\|\nabla F(\mathbf{w}_t)\|_2\le G$ gives
\begin{equation}
\left\|
\mathbb{E}[\mathbf{g}_t\mid\mathcal{F}_t]
-
\nabla F(\mathbf{w}_t)
\right\|_2
\le
(1-\mu_t)\chi_t
+
2G\mu_t
+
\varrho_t.
\end{equation}
This proves the result.
\end{proof}

\begin{lemma}[Sufficient condition for vanishing screening error]
\label{lem:vanishing_filtering_error}
Suppose
\begin{equation}
\beta_t
\le
(1-\mu_t)\chi_t
+
2G\mu_t
+
\varrho_t.
\end{equation}
If
\begin{equation}
\frac{\sum_{t=0}^{T-1}\eta_t\mathbb{E}[\chi_t^2]}{S_T}\to 0,
\qquad
\frac{\sum_{t=0}^{T-1}\eta_t\mathbb{E}[\mu_t^2]}{S_T}\to 0,
\qquad
\frac{\sum_{t=0}^{T-1}\eta_t\mathbb{E}[\varrho_t^2]}{S_T}\to 0,
\end{equation}
then
\begin{equation}
\frac{
\sum_{t=0}^{T-1}
\eta_t
\mathbb{E}[\beta_t^2]
}{
S_T
}
\to 0.
\end{equation}
\end{lemma}

\begin{proof}
Using $(a+b+c)^2\le 3a^2+3b^2+3c^2$, we obtain
\begin{align}
\beta_t^2
&\le
3(1-\mu_t)^2\chi_t^2
+
12G^2\mu_t^2
+
3\varrho_t^2.
\end{align}
Since $(1-\mu_t)^2\le 1$, this gives
\begin{equation}
\beta_t^2
\le
3\chi_t^2
+
12G^2\mu_t^2
+
3\varrho_t^2.
\end{equation}
Multiplying by $\eta_t$, summing from $t=0$ to $T-1$, taking expectation, and dividing by $S_T$ yields
\begin{equation}
\frac{
\sum_{t=0}^{T-1}
\eta_t\mathbb{E}[\beta_t^2]
}{
S_T
}
\le
3
\frac{
\sum_{t=0}^{T-1}
\eta_t\mathbb{E}[\chi_t^2]
}{
S_T
}
+
12G^2
\frac{
\sum_{t=0}^{T-1}
\eta_t\mathbb{E}[\mu_t^2]
}{
S_T
}
+
3
\frac{
\sum_{t=0}^{T-1}
\eta_t\mathbb{E}[\varrho_t^2]
}{
S_T
}.
\end{equation}
Each term on the right-hand side converges to zero by assumption. 
Therefore,
\begin{equation}
\frac{
\sum_{t=0}^{T-1}
\eta_t\mathbb{E}[\beta_t^2]
}{
S_T
}
\to 0.
\end{equation}
\end{proof}

\subsection{Local-Training Drift Bound}
\label{app:local_drift_bound}

For completeness, we provide a bound on $\varrho_t$ under local gradient descent.

\begin{assumption}[Bounded benign gradients during local training]
\label{assump:bounded_local_gradients}
For every benign client $i$, local step $e$, and round $t$,
\begin{equation}
\left\|
\nabla F_i(\mathbf{w}_{t,i}^{(e)})
\right\|_2
\le
G_B.
\end{equation}
Also, each $F_i$ is $L_i$-smooth with $L_i\le L_{\max}$.
\end{assumption}

Suppose benign client $i$ performs $E$ local gradient steps with local learning rate $\eta_{\mathrm{loc}}$:
\begin{equation}
\mathbf{w}_{t,i}^{(0)}
=
\mathbf{w}_t,
\qquad
\mathbf{w}_{t,i}^{(e+1)}
=
\mathbf{w}_{t,i}^{(e)}
-
\eta_{\mathrm{loc}}
\nabla F_i(\mathbf{w}_{t,i}^{(e)}),
\qquad e=0,\dots,E-1.
\end{equation}
Define the average local gradient direction
\begin{equation}
\bar{\mathbf{g}}_{t,i}^{B}
=
\frac{1}{E}
\sum_{e=0}^{E-1}
\nabla F_i(\mathbf{w}_{t,i}^{(e)}).
\end{equation}
Then
\begin{equation}
\mathbf{r}_{t,i}
=
\bar{\mathbf{g}}_{t,i}^{B}
-
\nabla F_i(\mathbf{w}_t).
\end{equation}

\begin{lemma}[Local-training drift bound]
\label{lem:local_drift_bound}
Under Assumption~\ref{assump:bounded_local_gradients},
\begin{equation}
\|\mathbf{r}_{t,i}\|_2
\le
\frac{L_{\max}\eta_{\mathrm{loc}}G_B(E-1)}{2}.
\label{eq:local_drift_bound_app}
\end{equation}
Consequently,
\begin{equation}
\varrho_t
\le
(1-\mu_t)
\frac{L_{\max}\eta_{\mathrm{loc}}G_B(E-1)}{2}.
\label{eq:weighted_local_drift_bound_app}
\end{equation}
\end{lemma}

\begin{proof}
By $L_i$-smoothness,
\begin{equation}
\left\|
\nabla F_i(\mathbf{w}_{t,i}^{(e)})
-
\nabla F_i(\mathbf{w}_t)
\right\|_2
\le
L_i
\left\|
\mathbf{w}_{t,i}^{(e)}
-
\mathbf{w}_t
\right\|_2.
\end{equation}
Because each local step has norm at most $\eta_{\mathrm{loc}}G_B$,
\begin{equation}
\left\|
\mathbf{w}_{t,i}^{(e)}
-
\mathbf{w}_t
\right\|_2
\le
e\eta_{\mathrm{loc}}G_B.
\end{equation}
Thus,
\begin{align}
\|\mathbf{r}_{t,i}\|_2
&=
\left\|
\frac{1}{E}
\sum_{e=0}^{E-1}
\left[
\nabla F_i(\mathbf{w}_{t,i}^{(e)})
-
\nabla F_i(\mathbf{w}_t)
\right]
\right\|_2
\nonumber\\
&\le
\frac{1}{E}
\sum_{e=0}^{E-1}
L_i
\left\|
\mathbf{w}_{t,i}^{(e)}
-
\mathbf{w}_t
\right\|_2
\nonumber\\
&\le
\frac{1}{E}
\sum_{e=0}^{E-1}
L_{\max}e\eta_{\mathrm{loc}}G_B
\nonumber\\
&=
\frac{L_{\max}\eta_{\mathrm{loc}}G_B(E-1)}{2}.
\end{align}
Finally,
\begin{align}
\varrho_t
&=
\left\|
\sum_{i\in\mathcal{R}_t\cap\mathcal{C}_{\mathrm{ben}}}
\alpha_{t,i}\mathbf{r}_{t,i}
\right\|_2
\nonumber\\
&\le
\sum_{i\in\mathcal{R}_t\cap\mathcal{C}_{\mathrm{ben}}}
\alpha_{t,i}
\|\mathbf{r}_{t,i}\|_2
\nonumber\\
&\le
(1-\mu_t)
\frac{L_{\max}\eta_{\mathrm{loc}}G_B(E-1)}{2}.
\end{align}
This proves the lemma.
\end{proof}

\clearpage
\section{Detailed Experimental Settings}
\label{app:experimental_details}

This appendix provides the complete experimental configuration used in Section~\ref{sec:experiments}, including the datasets, tokenization, client partitions, model architectures, optimization settings, attack configuration, baseline parameters, and \FedLNS{} implementation details.

\subsection{Federated Learning Protocol}
\label{app:fl_protocol}

Table~\ref{tab:app_fl_protocol} summarizes the simulation protocol shared by all three model families: a population of $K=200$ clients participates at rate $0.1$ ($20$ clients per round) over $200$ communication rounds, with local AdamW optimization (weight decay $0.01$), evaluation after every round, and results averaged over three random seeds under both IID and non-IID partitions.

\begin{table}[h!]
\centering
\caption{Federated simulation protocol used across model families.}
\label{tab:app_fl_protocol}
\begin{tabular}{p{0.47\linewidth}p{0.47\linewidth}}
\toprule
Item & Value \\
\midrule
Number of clients & $K=200$ \\
Participation rate & $0.1$ \\
Clients per round & $20$ \\
Communication rounds & $200$ \\
Random seeds & $100,200,300$ \\
Data partitions & IID and non-IID \\
Non-IID construction & Two shards per client \\
Local optimizer & AdamW \\
Weight decay & $0.01$ \\
Evaluation frequency & Every communication round \\
\bottomrule
\end{tabular}
\end{table}

\subsection{Datasets and Tokenization}
\label{app:datasets_tokenization}

Table~\ref{tab:app_dataset_tokenization} lists the dataset, tokenizer, and sequence-construction choices for each model family, and Table~\ref{tab:app_client_partitions} describes how the resulting training units are partitioned across clients.

\begin{table*}[h!]
\centering
\caption{Dataset and tokenization settings. All models are initialized from scratch; tokenizer names specify tokenization only.}
\label{tab:app_dataset_tokenization}
\begin{tabular}{p{0.14\textwidth}p{0.14\textwidth}p{0.24\textwidth}p{0.22\textwidth}p{0.13\textwidth}}
\toprule
Model family & Dataset & Tokenizer & Sequence construction & Evaluation split used for selection and reporting \\
\midrule
GPT-style LM
& WikiText-2 raw
& \texttt{gpt2}
& Causal LM blocks of length $128$
& Official validation split \\
BERT-style MLM
& Tiny Shakespeare
& \texttt{bert-base-uncased}
& MLM blocks of length $128$ with [CLS]/[SEP]
& Final $10\%$ of the text \\
LLaMA-style LM
& StackOverflow posts
& \texttt{TinyLlama/TinyLlama-1.1B-Chat-v1.0}
& Causal LM blocks of length $128$
& Evaluation split defined during dataset construction \\
\bottomrule
\end{tabular}
\end{table*}

\begin{table*}[h!]
\centering
\caption{Client data partition construction for IID and non-IID settings.}
\label{tab:app_client_partitions}
\begin{tabular}{p{0.19\textwidth}p{0.19\textwidth}p{0.26\textwidth}p{0.26\textwidth}}
\toprule
Model family & Training units before partitioning & IID partition & Non-IID partition \\
\midrule
GPT-style LM / WikiText-2
& Causal LM token blocks of length $128$
& Shuffle all block indices and assign them round-robin to $200$ clients
& Split ordered block indices into $400$ contiguous shards; assign two shuffled shards per client \\
BERT-style MLM / Tiny Shakespeare
& BERT MLM blocks of length $128$ from the first $90\%$ of the text
& Shuffle all block indices and assign them round-robin to $200$ clients
& Split ordered block indices into $400$ contiguous shards; assign two shuffled shards per client \\
LLaMA-style LM / StackOverflow
& Causal LM token blocks of length $128$ from up to $50{,}000$ training posts
& Shuffle all block indices and assign them round-robin to $200$ clients
& Split ordered block indices into $400$ contiguous shards; assign two shuffled shards per client \\
\bottomrule
\end{tabular}
\end{table*}

For each model family, the training corpus is first tokenized and converted into fixed-length training blocks, after which the client partition is applied over the training-block indices. The evaluation split is never partitioned across clients. It is used for global checkpoint selection and for reporting test loss, perplexity, and token-level semantic entropy.

In the IID setting, all training block indices are shuffled using the experiment seed and distributed to the $K=200$ clients in round-robin order, so each client receives an approximately equal number of blocks sampled from across the entire training corpus.

In the non-IID setting, the ordered list of block indices is first divided into $K \times 2 = 400$ contiguous shards, the shard list is shuffled, and each client is assigned two shards. Clients therefore still have approximately balanced data sizes, but each one observes only a small number of localized regions of the corpus -- inducing client-level heterogeneity through topic, style, speaker, or source locality rather than through class-label skew.

\subsection{Model Architectures and Optimization}
\label{app:model_architectures}

Table~\ref{tab:app_model_architectures} lists the scratch transformer architecture used for each model family, and Table~\ref{tab:app_optimization_eval} lists the corresponding local optimization and evaluation settings.

\begin{table*}[h!]
\centering
\caption{Scratch transformer architectures.}
\label{tab:app_model_architectures}
\begin{tabular}{p{0.29\textwidth}p{0.19\textwidth}p{0.19\textwidth}p{0.19\textwidth}}
\toprule
Hyperparameter & GPT-style LM & BERT-style MLM & LLaMA-style LM \\
\midrule
Hidden / embedding dimension & $256$ & $256$ & $256$ \\
Number of layers & $12$ & $12$ & $8$ \\
Attention heads & $16$ & $16$ & $16$ \\
Key-value heads & -- & -- & $8$ \\
Intermediate dimension & -- & $1024$ & $1536$ \\
Context / block size & $128$ & $128$ & $128$ \\
Max position embeddings & $128$ & $130$ & $256$ \\
Dropout & $0.1$ & Default config & Default config \\
RMSNorm epsilon & -- & -- & $10^{-6}$ \\
Initializer range & Default config & Default config & $0.02$ \\
\bottomrule
\end{tabular}
\end{table*}

\label{app:optimization_hyperparameters}
\begin{table*}[h!]
\centering
\caption{Optimization and evaluation settings.}
\label{tab:app_optimization_eval}
\begin{tabular}{p{0.29\textwidth}p{0.19\textwidth}p{0.19\textwidth}p{0.19\textwidth}}
\toprule
Setting & GPT-style LM & BERT-style MLM & LLaMA-style LM \\
\midrule
Local epochs & $3$ & $3$ & $2$ \\
Batch size & $8$ & $8$ & $8$ \\
Learning rate & $3\times 10^{-4}$ & $3\times 10^{-4}$ & $10^{-5}$ \\
Weight decay & $0.01$ & $0.01$ & $0.01$ \\
Train-evaluation batches & $200$ & $200$ & $50$ \\
Evaluation batches & $200$ & $200$ & $50$ \\
MLM masking probability & -- & $0.15$ & -- \\
Maximum training examples & Full split & Full split & $50{,}000$ \\
Maximum evaluation examples & Full split & Full split & $5{,}000$ \\
\bottomrule
\end{tabular}
\end{table*}

\newpage
\subsection{Attack and Aggregation-Baseline Configuration}
\label{app:attack_configuration}

Table~\ref{tab:app_attack_config} details the target-corruption attack used throughout the experiments: input tokens are left unchanged, only training targets are corrupted, and each malicious client corrupts its labels with probability $p_{\mathrm{corr}}=1.0$. The main text reports results at the malicious-client fraction $\alpha=0.4$ ($80$ of $200$ clients); this appendix evaluates the full sweep $\alpha\in\{0,0.1,0.2,0.3,0.4\}$.

\begin{table}[h!]
\centering
\caption{Target-corruption attack settings.}
\label{tab:app_attack_config}
\begin{tabular}{p{0.40\linewidth}p{0.50\linewidth}}
\toprule
Item & Value \\
\midrule
Attack type & Target corruption \\
Input tokens & Kept unchanged \\
Corrupted quantity & Training targets / labels \\
Corruption probability & $p_{\mathrm{corr}}=1.0$ \\
Malicious fractions in sweep & $0,0.1,0.2,0.3,0.4$ \\
Strongest main-text attack fraction & $0.4$ \\
Malicious clients in the population at $0.4$ & $80$ out of $200$ \\
Special-token handling & Padding labels avoided; EOS/padding avoided when sampling corrupt labels \\
\bottomrule
\end{tabular}
\end{table}

\label{app:baseline_settings}
Table~\ref{tab:app_baselines} summarizes the aggregation methods compared against \FedLNS{}, and Table~\ref{tab:app_robust_agg_params} lists the attack-dependent parameters used for each population-level malicious-client fraction: the trim fraction for Trimmed Mean and the parameters $f$ and $m$ for Multi-Krum, with $20$ clients selected per round. These two baselines are therefore configured using the population-level corruption fraction.

\begin{table*}[h!]
\centering
\caption{Aggregation methods compared in the experiments.}
\label{tab:app_baselines}
\begin{tabular}{p{0.24\textwidth}p{0.68\textwidth}}
\toprule
Method & Implementation summary \\
\midrule
FedAvg & Sample-weighted average of client updates \\
Norm-Bounded FedAvg & Each update is clipped by an $\ell_2$ norm bound before FedAvg \\
Coordinate Median & Coordinate-wise median of flattened updates \\
Trimmed Mean & Coordinate-wise trimming followed by averaging \\
Multi-Krum & Distance-based update selection followed by averaging selected updates \\
FLAME & Two-cluster filtering, majority-cluster retention, norm clipping, and Gaussian noise \\
\FedLNS{} & Normalization-signature bank, median/MAD standardization, and BIC-guided GMM screening \\
\bottomrule
\end{tabular}
\end{table*}

\begin{table}[h!]
\centering
\caption{Robust-aggregation parameters used for each malicious-client fraction. The number of selected clients per round is $20$.}
\label{tab:app_robust_agg_params}
\begin{tabular}{p{0.25\linewidth}p{0.21\linewidth}p{0.20\linewidth}p{0.20\linewidth}}
\toprule
Malicious fraction & Trim fraction & Multi-Krum $f$ & Multi-Krum $m$ \\
\midrule
$0.0$ & $0.0$ & $0$ & $20$ \\
$0.1$ & $0.1$ & $2$ & $18$ \\
$0.2$ & $0.2$ & $4$ & $16$ \\
$0.3$ & $0.3$ & $6$ & $14$ \\
$0.4$ & $0.4$ & $8$ & $12$ \\
\bottomrule
\end{tabular}
\end{table}

\newpage
\subsection{FedLNS Implementation Details}
\label{app:fedlns_implementation_details}

Table~\ref{tab:app_fedlns_details} lists the \FedLNS{} implementation settings used in all experiments, covering normalization-parameter extraction, bank update timing, robust standardization, and BIC-guided one-vs-two-component GMM screening.

\begin{table*}[h!]
\centering
\caption{FedLNS implementation settings used in the experiments.}
\label{tab:app_fedlns_details}
\begin{tabular}{p{0.28\textwidth}p{0.64\textwidth}}
\toprule
Item & Setting \\
\midrule
Signature source
& Trainable normalization-layer parameters whose names contain \texttt{ln} or \texttt{layernorm} and end with \texttt{weight} or \texttt{bias} \\
Bank entry
& One latest per-layer signature dictionary per client \\
Bank update timing
& Before screening in the current round \\
Bank update rule
& Latest replacement; EMA coefficient $\beta=0$ \\
Robust center
& Coordinate-wise bank median \\
Robust scale
& Coordinate-wise MAD \\
MAD floor
& $10^{-6}$ \\
Client deviation score
& Median absolute standardized deviation \\
GMM candidates
& One diagonal Gaussian vs. two diagonal Gaussians \\
GMM model selection
& BIC for default \FedLNS{} \\
GMM covariance
& Diagonal \\
GMM regularization
& $10^{-4}$ \\
GMM initialization
& K-means initialization, $5$ restarts \\
Maximum GMM iterations
& $200$ \\
Two-component retain rule
& Keep component with smaller median standardized deviation \\
Safety fallback
& If no client is retained, retain
$\arg\min_{i\in\mathcal{S}_t}\delta_{t,i}$ \\
\bottomrule
\end{tabular}
\end{table*}

\clearpage
\section{Complete Results Across Malicious-Client Fractions}
\label{app:remaining_tables}

This appendix reports the complete numerical results for the malicious-client fraction sweep, evaluated over
\begin{equation}
    \alpha \in \{0,0.1,0.2,0.3,0.4\},
\end{equation}
which for $K=200$ clients correspond to $0$, $20$, $40$, $60$, and $80$ malicious clients, respectively.

All tables follow the reporting rule used in the main text: for each seed, the checkpoint with the lowest test loss is selected, and test loss, perplexity, and token-level semantic entropy are reported from that same model state. For each model, partition, aggregation method, and malicious-client fraction, each entry reports the mean and sample standard deviation over seeds $\{100,200,300\}$, with lower values indicating better performance for all three metrics.

The tables are grouped by model family and data partition. For each model and partition, they report test loss, perplexity, and token-level semantic entropy across the full attack-fraction sweep from the clean setting to the strongest evaluated setting of $40\%$.

\subsection{GPT-Style Causal Language Modeling on WikiText}
Tables~\ref{tab:gpt_iid_test_loss_attack_fraction}--\ref{tab:gpt_iid_test_sem_ent_attack_fraction} report the IID results for GPT/WikiText, and Tables~\ref{tab:gpt_noniid_test_loss_attack_fraction}--\ref{tab:gpt_noniid_test_sem_ent_attack_fraction} report the non-IID results.
\begin{table}[ht]\centering
\caption{Test Loss under different malicious-client fractions for GPT / WikiText with IID partition. Values are mean $\pm$ standard deviation over three seeds. For each seed, all metrics are taken from the checkpoint with the lowest test loss. The best case is highlighted in bold.}
\label{tab:gpt_iid_test_loss_attack_fraction}
\begin{tabular}{lccccc}
\toprule
Method & 0\% & 10\% & 20\% & 30\% & 40\% \\
\midrule
FedAvg & 7.13 $\pm$ 0.01 & 7.17 $\pm$ 0.01 & 7.23 $\pm$ 0.00 & 7.29 $\pm$ 0.01 & 7.37 $\pm$ 0.02 \\
Norm-Bound & 7.50 $\pm$ 0.00 & 7.51 $\pm$ 0.01 & 7.59 $\pm$ 0.01 & 7.69 $\pm$ 0.02 & 7.89 $\pm$ 0.01 \\
Median & 7.30 $\pm$ 0.01 & 7.27 $\pm$ 0.01 & 7.28 $\pm$ 0.01 & 7.30 $\pm$ 0.01 & 7.47 $\pm$ 0.03 \\
Trimmed & \textbf{7.12 $\pm$ 0.01} & 7.19 $\pm$ 0.01 & 7.25 $\pm$ 0.01 & 7.32 $\pm$ 0.03 & 7.43 $\pm$ 0.01 \\
Multi-Krum & 7.12 $\pm$ 0.01 & 7.15 $\pm$ 0.01 & 7.19 $\pm$ 0.00 & 7.21 $\pm$ 0.00 & 7.28 $\pm$ 0.02 \\
FLAME & 7.19 $\pm$ 0.00 & 7.19 $\pm$ 0.01 & 7.20 $\pm$ 0.00 & 7.21 $\pm$ 0.00 & 7.33 $\pm$ 0.03 \\
FedLNS & 7.13 $\pm$ 0.00 & \textbf{7.13 $\pm$ 0.00} & \textbf{7.13 $\pm$ 0.00} & \textbf{7.14 $\pm$ 0.00} & \textbf{7.16 $\pm$ 0.00} \\
\bottomrule
\end{tabular}
\end{table}

\begin{table}[ht]\centering
\caption{Test PPL under different malicious-client fractions for GPT / WikiText with IID partition. Values are mean $\pm$ standard deviation over three seeds. For each seed, all metrics are taken from the checkpoint with the lowest test loss. The best case is highlighted in bold.}
\label{tab:gpt_iid_test_ppl_attack_fraction}
\begin{tabular}{lccccc}
\toprule
Method & 0\% & 10\% & 20\% & 30\% & 40\% \\
\midrule
FedAvg & 1242.91 $\pm$ 6.58 & 1303.02 $\pm$ 9.59 & 1373.61 $\pm$ 1.63 & 1472.41 $\pm$ 15.31 & 1590.56 $\pm$ 35.51 \\
Norm-Bound & 1809.27 $\pm$ 8.31 & 1828.13 $\pm$ 10.69 & 1985.85 $\pm$ 25.34 & 2196.84 $\pm$ 33.99 & 2683.53 $\pm$ 34.31 \\
Median & 1484.69 $\pm$ 11.20 & 1441.44 $\pm$ 9.33 & 1452.99 $\pm$ 11.40 & 1477.26 $\pm$ 17.65 & 1751.07 $\pm$ 52.62 \\
Trimmed & \textbf{1241.55 $\pm$ 10.08} & 1326.15 $\pm$ 17.45 & 1411.15 $\pm$ 15.57 & 1516.84 $\pm$ 43.71 & 1689.16 $\pm$ 18.00 \\
Multi-Krum & 1242.08 $\pm$ 8.40 & 1275.34 $\pm$ 7.09 & 1320.41 $\pm$ 5.63 & 1356.21 $\pm$ 6.19 & 1449.72 $\pm$ 22.36 \\
FLAME & 1323.55 $\pm$ 4.87 & 1328.88 $\pm$ 7.21 & 1345.05 $\pm$ 2.82 & 1355.34 $\pm$ 3.27 & 1524.02 $\pm$ 40.71 \\
FedLNS & 1246.57 $\pm$ 2.52 & \textbf{1247.28 $\pm$ 5.97} & \textbf{1249.11 $\pm$ 1.78} & \textbf{1258.21 $\pm$ 4.70} & \textbf{1281.63 $\pm$ 1.74} \\
\bottomrule
\end{tabular}
\end{table}

\begin{table}[ht]\centering
\caption{Semantic Entropy under different malicious-client fractions for GPT / WikiText with IID partition. Values are mean $\pm$ standard deviation over three seeds. For each seed, all metrics are taken from the checkpoint with the lowest test loss. The best case is highlighted in bold.}
\label{tab:gpt_iid_test_sem_ent_attack_fraction}
\begin{tabular}{lccccc}
\toprule
Method & 0\% & 10\% & 20\% & 30\% & 40\% \\
\midrule
FedAvg & 6.80 $\pm$ 0.06 & 7.08 $\pm$ 0.08 & 7.37 $\pm$ 0.18 & 7.56 $\pm$ 0.16 & 7.96 $\pm$ 0.07 \\
Norm-Bound & \textbf{6.55 $\pm$ 0.02} & 8.04 $\pm$ 0.12 & 8.31 $\pm$ 0.10 & 8.92 $\pm$ 0.03 & 9.69 $\pm$ 0.06 \\
Median & 6.89 $\pm$ 0.04 & 7.30 $\pm$ 0.15 & 7.27 $\pm$ 0.14 & 7.61 $\pm$ 0.13 & 7.55 $\pm$ 0.19 \\
Trimmed & 6.77 $\pm$ 0.03 & 7.30 $\pm$ 0.18 & 7.29 $\pm$ 0.19 & 7.47 $\pm$ 0.13 & 7.55 $\pm$ 0.39 \\
Multi-Krum & 6.80 $\pm$ 0.07 & 7.14 $\pm$ 0.24 & 7.07 $\pm$ 0.09 & 7.03 $\pm$ 0.11 & 7.13 $\pm$ 0.37 \\
FLAME & 7.11 $\pm$ 0.05 & 7.13 $\pm$ 0.10 & 7.20 $\pm$ 0.05 & 7.11 $\pm$ 0.09 & 6.96 $\pm$ 0.10 \\
FedLNS & 6.73 $\pm$ 0.03 & \textbf{6.81 $\pm$ 0.01} & \textbf{6.75 $\pm$ 0.02} & \textbf{6.74 $\pm$ 0.04} & \textbf{6.80 $\pm$ 0.06} \\
\bottomrule
\end{tabular}
\end{table}

\begin{table}[ht]\centering
\caption{Test Loss under different malicious-client fractions for GPT / WikiText with Non-IID partition. Values are mean $\pm$ standard deviation over three seeds. For each seed, all metrics are taken from the checkpoint with the lowest test loss. The best case is highlighted in bold.}
\label{tab:gpt_noniid_test_loss_attack_fraction}
\begin{tabular}{lccccc}
\toprule
Method & 0\% & 10\% & 20\% & 30\% & 40\% \\
\midrule
FedAvg & \textbf{7.10 $\pm$ 0.01} & 7.16 $\pm$ 0.01 & 7.23 $\pm$ 0.04 & 7.29 $\pm$ 0.06 & 7.37 $\pm$ 0.08 \\
Norm-Bound & 7.56 $\pm$ 0.01 & 7.54 $\pm$ 0.00 & 7.63 $\pm$ 0.01 & 7.75 $\pm$ 0.02 & 7.94 $\pm$ 0.09 \\
Median & 7.21 $\pm$ 0.00 & 7.18 $\pm$ 0.00 & 7.26 $\pm$ 0.02 & 7.33 $\pm$ 0.02 & 7.51 $\pm$ 0.03 \\
Trimmed & 7.11 $\pm$ 0.01 & 7.17 $\pm$ 0.01 & 7.25 $\pm$ 0.01 & 7.38 $\pm$ 0.03 & 7.51 $\pm$ 0.05 \\
Multi-Krum & 7.11 $\pm$ 0.01 & 7.18 $\pm$ 0.00 & 7.25 $\pm$ 0.01 & 7.30 $\pm$ 0.00 & 7.36 $\pm$ 0.03 \\
FLAME & 7.21 $\pm$ 0.00 & 7.25 $\pm$ 0.01 & 7.28 $\pm$ 0.01 & 7.29 $\pm$ 0.01 & 7.40 $\pm$ 0.04 \\
FedLNS & 7.11 $\pm$ 0.01 & \textbf{7.12 $\pm$ 0.00} & \textbf{7.13 $\pm$ 0.01} & \textbf{7.15 $\pm$ 0.02} & \textbf{7.16 $\pm$ 0.01} \\
\bottomrule
\end{tabular}
\end{table}

\begin{table}[ht]\centering
\caption{Test PPL under different malicious-client fractions for GPT / WikiText with Non-IID partition. Values are mean $\pm$ standard deviation over three seeds. For each seed, all metrics are taken from the checkpoint with the lowest test loss. The best case is highlighted in bold.}
\label{tab:gpt_noniid_test_ppl_attack_fraction}
\begin{tabular}{lccccc}
\toprule
Method & 0\% & 10\% & 20\% & 30\% & 40\% \\
\midrule
FedAvg & \textbf{1211.23 $\pm$ 6.70} & 1282.08 $\pm$ 13.22 & 1378.40 $\pm$ 55.98 & 1471.79 $\pm$ 90.09 & 1595.68 $\pm$ 132.68 \\
Norm-Bound & 1916.47 $\pm$ 13.00 & 1886.00 $\pm$ 3.63 & 2060.01 $\pm$ 13.63 & 2314.33 $\pm$ 44.03 & 2811.07 $\pm$ 257.29 \\
Median & 1350.32 $\pm$ 1.34 & 1309.94 $\pm$ 6.40 & 1419.46 $\pm$ 32.27 & 1529.64 $\pm$ 31.03 & 1820.62 $\pm$ 48.31 \\
Trimmed & 1225.89 $\pm$ 8.21 & 1295.29 $\pm$ 16.41 & 1406.67 $\pm$ 11.55 & 1599.60 $\pm$ 52.85 & 1819.34 $\pm$ 98.93 \\
Multi-Krum & 1225.07 $\pm$ 8.06 & 1311.67 $\pm$ 5.60 & 1408.79 $\pm$ 20.45 & 1474.93 $\pm$ 3.80 & 1579.81 $\pm$ 43.94 \\
FLAME & 1359.58 $\pm$ 3.82 & 1415.19 $\pm$ 13.32 & 1449.95 $\pm$ 14.72 & 1467.85 $\pm$ 14.75 & 1637.88 $\pm$ 63.90 \\
FedLNS & 1225.90 $\pm$ 12.06 & \textbf{1238.52 $\pm$ 4.97} & \textbf{1251.65 $\pm$ 11.37} & \textbf{1277.28 $\pm$ 23.52} & \textbf{1286.50 $\pm$ 15.34} \\
\bottomrule
\end{tabular}
\end{table}

\begin{table}[ht]\centering
\caption{Semantic Entropy under different malicious-client fractions for GPT / WikiText with Non-IID partition. Values are mean $\pm$ standard deviation over three seeds. For each seed, all metrics are taken from the checkpoint with the lowest test loss. The best case is highlighted in bold.}
\label{tab:gpt_noniid_test_sem_ent_attack_fraction}
\begin{tabular}{lccccc}
\toprule
Method & 0\% & 10\% & 20\% & 30\% & 40\% \\
\midrule
FedAvg & 6.79 $\pm$ 0.07 & 7.31 $\pm$ 0.17 & 7.12 $\pm$ 0.37 & 7.61 $\pm$ 0.30 & 7.66 $\pm$ 0.26 \\
Norm-Bound & \textbf{6.37 $\pm$ 0.02} & 7.85 $\pm$ 0.19 & 8.44 $\pm$ 0.20 & 9.02 $\pm$ 0.09 & 9.60 $\pm$ 0.36 \\
Median & 6.74 $\pm$ 0.15 & 7.31 $\pm$ 0.34 & 7.14 $\pm$ 0.32 & 7.52 $\pm$ 0.10 & 7.45 $\pm$ 0.25 \\
Trimmed & 6.75 $\pm$ 0.15 & 7.36 $\pm$ 0.07 & 7.27 $\pm$ 0.07 & 7.39 $\pm$ 0.11 & 7.72 $\pm$ 0.28 \\
Multi-Krum & 6.79 $\pm$ 0.17 & 7.08 $\pm$ 0.12 & 7.07 $\pm$ 0.45 & 7.25 $\pm$ 0.47 & 7.18 $\pm$ 0.46 \\
FLAME & 7.04 $\pm$ 0.03 & 7.08 $\pm$ 0.15 & 7.32 $\pm$ 0.20 & 6.94 $\pm$ 0.07 & 7.13 $\pm$ 0.58 \\
FedLNS & 6.80 $\pm$ 0.09 & \textbf{6.79 $\pm$ 0.05} & \textbf{7.04 $\pm$ 0.40} & \textbf{6.90 $\pm$ 0.02} & \textbf{6.92 $\pm$ 0.24} \\
\bottomrule
\end{tabular}
\end{table}

\clearpage
\subsection{BERT-Style Masked Language Modeling on Tiny Shakespeare}
Tables~\ref{tab:bert_iid_test_loss_attack_fraction}--\ref{tab:bert_iid_test_sem_ent_attack_fraction} report the IID results for BERT/Tiny Shakespeare, and Tables~\ref{tab:bert_noniid_test_loss_attack_fraction}--\ref{tab:bert_noniid_test_sem_ent_attack_fraction} report the non-IID results. The reported perplexity is masked-token perplexity and should be interpreted within the masked-language-modeling setting.
\begin{table}[ht]
\centering
\caption{Test Loss under different malicious-client fractions for BERT / Shakespeare with IID partition. Values are mean $\pm$ standard deviation over three seeds. For each seed, all metrics are taken from the checkpoint with the lowest test loss. The best case is highlighted in bold.}
\label{tab:bert_iid_test_loss_attack_fraction}
\begin{tabular}{lccccc}
\toprule
Method & 0\% & 10\% & 20\% & 30\% & 40\% \\
\midrule
FedAvg & 6.49 $\pm$ 0.03 & 6.54 $\pm$ 0.01 & 6.57 $\pm$ 0.01 & 6.62 $\pm$ 0.02 & 6.68 $\pm$ 0.02 \\
Norm-Bound & 6.80 $\pm$ 0.03 & 6.80 $\pm$ 0.01 & 6.92 $\pm$ 0.01 & 7.11 $\pm$ 0.01 & 7.30 $\pm$ 0.05 \\
Median & 6.80 $\pm$ 0.03 & 6.82 $\pm$ 0.01 & 6.82 $\pm$ 0.02 & 6.80 $\pm$ 0.03 & 6.97 $\pm$ 0.08 \\
Trimmed & 6.49 $\pm$ 0.03 & 6.66 $\pm$ 0.01 & 6.72 $\pm$ 0.02 & 6.80 $\pm$ 0.03 & 6.97 $\pm$ 0.07 \\
Multi-Krum & 6.49 $\pm$ 0.03 & 6.51 $\pm$ 0.01 & 6.52 $\pm$ 0.01 & 6.54 $\pm$ 0.02 & 6.59 $\pm$ 0.02 \\
FLAME & \textbf{6.49 $\pm$ 0.03} & \textbf{6.47 $\pm$ 0.02} & \textbf{6.49 $\pm$ 0.02} & 6.52 $\pm$ 0.00 & 6.70 $\pm$ 0.18 \\
FedLNS & 6.50 $\pm$ 0.05 & 6.49 $\pm$ 0.02 & \textbf{6.49 $\pm$ 0.02} & \textbf{6.51 $\pm$ 0.01} & \textbf{6.55 $\pm$ 0.02} \\
\bottomrule
\end{tabular}
\end{table}

\begin{table}[ht]
\centering
\caption{Test PPL under different malicious-client fractions for BERT / Shakespeare with IID partition. Values are mean $\pm$ standard deviation over three seeds. For each seed, all metrics are taken from the checkpoint with the lowest test loss. The best case is highlighted in bold.}
\label{tab:bert_iid_test_ppl_attack_fraction}
\begin{tabular}{lccccc}
\toprule
Method & 0\% & 10\% & 20\% & 30\% & 40\% \\
\midrule
FedAvg & 659.93 $\pm$ 21.92 & 693.01 $\pm$ 4.36 & 710.15 $\pm$ 6.19 & 750.09 $\pm$ 12.08 & 794.54 $\pm$ 15.01 \\
Norm-Bound & 901.04 $\pm$ 29.47 & 897.17 $\pm$ 10.13 & 1008.61 $\pm$ 7.00 & 1227.23 $\pm$ 15.36 & 1477.67 $\pm$ 70.84 \\
Median & 894.85 $\pm$ 27.58 & 912.29 $\pm$ 9.07 & 913.38 $\pm$ 19.06 & 898.46 $\pm$ 22.55 & 1066.13 $\pm$ 91.56 \\
Trimmed & 660.99 $\pm$ 20.10 & 779.94 $\pm$ 9.18 & 831.19 $\pm$ 20.60 & 900.65 $\pm$ 26.88 & 1062.80 $\pm$ 79.49 \\
Multi-Krum & 661.00 $\pm$ 20.10 & 672.89 $\pm$ 7.97 & 676.44 $\pm$ 4.57 & 695.35 $\pm$ 12.21 & 725.76 $\pm$ 17.84 \\
FLAME & \textbf{658.52 $\pm$ 18.95} & \textbf{645.26 $\pm$ 15.47} & \textbf{657.21 $\pm$ 10.53} & 679.98 $\pm$ 3.26 & 818.57 $\pm$ 157.53 \\
FedLNS & 663.68 $\pm$ 29.76 & 655.97 $\pm$ 11.66 & 661.38 $\pm$ 11.94 & \textbf{669.37 $\pm$ 5.86} & \textbf{697.62 $\pm$ 13.62} \\
\bottomrule
\end{tabular}
\end{table}

\begin{table}[ht]\centering
\caption{Semantic Entropy under different malicious-client fractions for BERT / Shakespeare with IID partition. Values are mean $\pm$ standard deviation over three seeds. For each seed, all metrics are taken from the checkpoint with the lowest test loss. The best case is highlighted in bold.}
\label{tab:bert_iid_test_sem_ent_attack_fraction}
\begin{tabular}{lccccc}
\toprule
Method & 0\% & 10\% & 20\% & 30\% & 40\% \\
\midrule
FedAvg & \textbf{6.26 $\pm$ 0.18} & 6.79 $\pm$ 0.19 & 6.95 $\pm$ 0.07 & 7.22 $\pm$ 0.36 & 7.37 $\pm$ 0.33 \\
Norm-Bound & 6.31 $\pm$ 0.03 & 7.27 $\pm$ 0.21 & 7.65 $\pm$ 0.07 & 8.07 $\pm$ 0.15 & 8.54 $\pm$ 0.16 \\
Median & 6.60 $\pm$ 0.06 & 6.92 $\pm$ 0.22 & 7.05 $\pm$ 0.34 & 7.41 $\pm$ 0.22 & 7.47 $\pm$ 0.78 \\
Trimmed & 6.28 $\pm$ 0.16 & 6.63 $\pm$ 0.17 & 6.97 $\pm$ 0.18 & 7.35 $\pm$ 0.54 & 6.76 $\pm$ 0.47 \\
Multi-Krum & 6.28 $\pm$ 0.16 & 6.40 $\pm$ 0.07 & 6.69 $\pm$ 0.38 & 6.53 $\pm$ 0.36 & 6.61 $\pm$ 0.28 \\
FLAME & 6.26 $\pm$ 0.03 & \textbf{6.26 $\pm$ 0.11} & \textbf{6.11 $\pm$ 0.05} & \textbf{6.15 $\pm$ 0.10} & 6.46 $\pm$ 0.16 \\
FedLNS & 6.28 $\pm$ 0.09 & 6.38 $\pm$ 0.16 & 6.19 $\pm$ 0.09 & 6.24 $\pm$ 0.10 & \textbf{6.24 $\pm$ 0.06} \\
\bottomrule
\end{tabular}
\end{table}

\begin{table}[ht]\centering
\caption{Test Loss under different malicious-client fractions for BERT / Shakespeare with Non-IID partition. Values are mean $\pm$ standard deviation over three seeds. For each seed, all metrics are taken from the checkpoint with the lowest test loss. The best case is highlighted in bold.}
\label{tab:bert_noniid_test_loss_attack_fraction}
\begin{tabular}{lccccc}
\toprule
Method & 0\% & 10\% & 20\% & 30\% & 40\% \\
\midrule
FedAvg & \textbf{6.36 $\pm$ 0.03} & 6.38 $\pm$ 0.04 & 6.51 $\pm$ 0.10 & 6.56 $\pm$ 0.12 & 6.59 $\pm$ 0.11 \\
Norm-Bound & 6.83 $\pm$ 0.01 & 6.85 $\pm$ 0.01 & 6.95 $\pm$ 0.01 & 7.14 $\pm$ 0.02 & 7.34 $\pm$ 0.02 \\
Median & 6.86 $\pm$ 0.01 & 6.89 $\pm$ 0.01 & 6.86 $\pm$ 0.03 & 6.86 $\pm$ 0.01 & 6.95 $\pm$ 0.05 \\
Trimmed & 6.54 $\pm$ 0.03 & 6.72 $\pm$ 0.01 & 6.80 $\pm$ 0.02 & 6.84 $\pm$ 0.02 & 6.94 $\pm$ 0.02 \\
Multi-Krum & 6.54 $\pm$ 0.03 & 6.59 $\pm$ 0.02 & 6.60 $\pm$ 0.01 & 6.61 $\pm$ 0.02 & 6.69 $\pm$ 0.06 \\
FLAME & 6.57 $\pm$ 0.02 & 6.56 $\pm$ 0.03 & 6.56 $\pm$ 0.01 & 6.57 $\pm$ 0.01 & 6.86 $\pm$ 0.17 \\
FedLNS & 6.47 $\pm$ 0.05 & \textbf{6.36 $\pm$ 0.06} & \textbf{6.42 $\pm$ 0.14} & \textbf{6.38 $\pm$ 0.18} & \textbf{6.48 $\pm$ 0.13} \\
\bottomrule
\end{tabular}
\end{table}

\begin{table}[ht]\centering
\caption{Test PPL under different malicious-client fractions for BERT / Shakespeare with Non-IID partition. Values are mean $\pm$ standard deviation over three seeds. For each seed, all metrics are taken from the checkpoint with the lowest test loss. The best case is highlighted in bold.}
\label{tab:bert_noniid_test_ppl_attack_fraction}
\begin{tabular}{lccccc}
\toprule
Method & 0\% & 10\% & 20\% & 30\% & 40\% \\
\midrule
FedAvg & \textbf{576.49 $\pm$ 19.76} & 588.64 $\pm$ 24.25 & 672.55 $\pm$ 66.84 & 708.50 $\pm$ 88.05 & 731.73 $\pm$ 83.15 \\
Norm-Bound & 922.98 $\pm$ 7.18 & 941.70 $\pm$ 9.95 & 1039.58 $\pm$ 9.18 & 1262.16 $\pm$ 26.29 & 1533.59 $\pm$ 36.83 \\
Median & 950.49 $\pm$ 10.79 & 985.58 $\pm$ 6.47 & 949.59 $\pm$ 26.06 & 953.32 $\pm$ 10.76 & 1047.36 $\pm$ 53.66 \\
Trimmed & 691.27 $\pm$ 19.12 & 829.90 $\pm$ 7.29 & 893.88 $\pm$ 19.37 & 938.21 $\pm$ 17.35 & 1033.35 $\pm$ 24.28 \\
Multi-Krum & 691.27 $\pm$ 19.11 & 727.43 $\pm$ 13.60 & 732.73 $\pm$ 10.31 & 739.83 $\pm$ 12.52 & 804.39 $\pm$ 53.06 \\
FLAME & 710.27 $\pm$ 16.45 & 705.47 $\pm$ 23.82 & 704.63 $\pm$ 4.08 & 712.01 $\pm$ 8.24 & 963.65 $\pm$ 157.76 \\
FedLNS & 643.56 $\pm$ 33.46 & \textbf{576.63 $\pm$ 36.58} & \textbf{618.30 $\pm$ 91.12} & \textbf{594.61 $\pm$ 108.75} & \textbf{658.21 $\pm$ 91.66} \\
\bottomrule
\end{tabular}
\end{table}

\begin{table}[ht]\centering
\caption{Semantic Entropy under different malicious-client fractions for BERT / Shakespeare with Non-IID partition. Values are mean $\pm$ standard deviation over three seeds. For each seed, all metrics are taken from the checkpoint with the lowest test loss. The best case is highlighted in bold.}
\label{tab:bert_noniid_test_sem_ent_attack_fraction}
\begin{tabular}{lccccc}
\toprule
Method & 0\% & 10\% & 20\% & 30\% & 40\% \\
\midrule
FedAvg & \textbf{5.87 $\pm$ 0.14} & 6.34 $\pm$ 0.41 & 6.96 $\pm$ 0.20 & 7.19 $\pm$ 0.38 & 6.94 $\pm$ 0.14 \\
Norm-Bound & 6.27 $\pm$ 0.01 & 6.99 $\pm$ 0.14 & 7.63 $\pm$ 0.15 & 7.98 $\pm$ 0.19 & 8.48 $\pm$ 0.02 \\
Median & 6.59 $\pm$ 0.06 & 7.28 $\pm$ 0.10 & 7.20 $\pm$ 0.36 & 7.41 $\pm$ 0.26 & 7.20 $\pm$ 0.59 \\
Trimmed & 6.24 $\pm$ 0.10 & 6.61 $\pm$ 0.15 & 7.00 $\pm$ 0.13 & 7.23 $\pm$ 0.17 & 6.75 $\pm$ 0.27 \\
Multi-Krum & 6.24 $\pm$ 0.10 & 6.67 $\pm$ 0.24 & 6.43 $\pm$ 0.06 & 6.69 $\pm$ 0.13 & 6.82 $\pm$ 0.31 \\
FLAME & 6.19 $\pm$ 0.06 & 6.27 $\pm$ 0.16 & \textbf{6.30 $\pm$ 0.12} & 6.41 $\pm$ 0.12 & 6.62 $\pm$ 0.14 \\
FedLNS & 6.12 $\pm$ 0.05 & \textbf{6.06 $\pm$ 0.18} & 6.33 $\pm$ 0.09 & \textbf{6.00 $\pm$ 0.38} & \textbf{6.19 $\pm$ 0.13} \\
\bottomrule
\end{tabular}
\end{table}

\clearpage
\subsection{LLaMA-Style Causal Language Modeling on StackOverflow}
Tables~\ref{tab:llama_iid_test_loss_attack_fraction}--\ref{tab:llama_iid_test_sem_ent_attack_fraction} report the IID results for LLaMA/StackOverflow, and Tables~\ref{tab:llama_noniid_test_loss_attack_fraction}--\ref{tab:llama_noniid_test_sem_ent_attack_fraction} report the non-IID results.
\begin{table}[ht]\centering
\caption{Test Loss under different malicious-client fractions for LLaMA / StackOverflow with IID partition. Values are mean $\pm$ standard deviation over three seeds. For each seed, all metrics are taken from the checkpoint with the lowest test loss. The best case is highlighted in bold.}
\label{tab:llama_iid_test_loss_attack_fraction}
\begin{tabular}{lccccc}
\toprule
Method & 0\% & 10\% & 20\% & 30\% & 40\% \\
\midrule
FedAvg & \textbf{5.36 $\pm$ 0.03} & 5.40 $\pm$ 0.03 & 5.46 $\pm$ 0.04 & 5.53 $\pm$ 0.02 & 5.66 $\pm$ 0.04 \\
Norm-Bound & 5.49 $\pm$ 0.03 & 5.56 $\pm$ 0.03 & 5.63 $\pm$ 0.04 & 5.77 $\pm$ 0.02 & 5.86 $\pm$ 0.08 \\
Median & 5.39 $\pm$ 0.03 & 5.40 $\pm$ 0.03 & 5.43 $\pm$ 0.03 & 5.46 $\pm$ 0.03 & 5.54 $\pm$ 0.03 \\
Trimmed & 5.36 $\pm$ 0.03 & 5.39 $\pm$ 0.03 & 5.43 $\pm$ 0.03 & 5.47 $\pm$ 0.03 & 5.54 $\pm$ 0.04 \\
Multi-Krum & 5.36 $\pm$ 0.03 & 5.37 $\pm$ 0.03 & 5.38 $\pm$ 0.03 & 5.39 $\pm$ 0.03 & 5.43 $\pm$ 0.04 \\
FLAME & 5.48 $\pm$ 0.03 & 5.48 $\pm$ 0.03 & 5.48 $\pm$ 0.03 & 5.49 $\pm$ 0.03 & 5.58 $\pm$ 0.05 \\
FedLNS & \textbf{5.36 $\pm$ 0.03} & \textbf{5.37 $\pm$ 0.03} & \textbf{5.37 $\pm$ 0.03} & \textbf{5.37 $\pm$ 0.03} & \textbf{5.38 $\pm$ 0.03} \\
\bottomrule
\end{tabular}
\end{table}

\begin{table}[ht]\centering
\caption{Test PPL under different malicious-client fractions for LLaMA / StackOverflow with IID partition. Values are mean $\pm$ standard deviation over three seeds. For each seed, all metrics are taken from the checkpoint with the lowest test loss. The best case is highlighted in bold.}
\label{tab:llama_iid_test_ppl_attack_fraction}
\begin{tabular}{lccccc}
\toprule
Method & 0\% & 10\% & 20\% & 30\% & 40\% \\
\midrule
FedAvg & \textbf{213.15 $\pm$ 6.77} & 221.94 $\pm$ 6.63 & 234.03 $\pm$ 8.53 & 251.01 $\pm$ 5.88 & 286.73 $\pm$ 10.99 \\
Norm-Bound & 243.23 $\pm$ 6.54 & 258.63 $\pm$ 6.54 & 279.81 $\pm$ 11.33 & 319.10 $\pm$ 6.04 & 349.83 $\pm$ 27.13 \\
Median & 218.98 $\pm$ 7.04 & 222.39 $\pm$ 7.10 & 227.13 $\pm$ 7.36 & 235.44 $\pm$ 6.60 & 254.94 $\pm$ 8.06 \\
Trimmed & 213.15 $\pm$ 6.77 & 219.60 $\pm$ 7.16 & 227.49 $\pm$ 7.80 & 237.43 $\pm$ 7.51 & 254.52 $\pm$ 9.07 \\
Multi-Krum & 213.15 $\pm$ 6.77 & 215.53 $\pm$ 7.04 & 217.38 $\pm$ 6.99 & 219.76 $\pm$ 6.92 & 228.52 $\pm$ 8.72 \\
FLAME & 240.02 $\pm$ 6.40 & 239.85 $\pm$ 6.77 & 240.36 $\pm$ 6.81 & 242.54 $\pm$ 7.33 & 264.80 $\pm$ 13.70 \\
FedLNS & 213.42 $\pm$ 6.57 & \textbf{213.93 $\pm$ 6.56} & \textbf{214.57 $\pm$ 6.75} & \textbf{215.61 $\pm$ 6.70} & \textbf{217.14 $\pm$ 6.85} \\
\bottomrule
\end{tabular}
\end{table}

\begin{table}[ht]\centering
\caption{Semantic Entropy under different malicious-client fractions for LLaMA / StackOverflow with IID partition. Values are mean $\pm$ standard deviation over three seeds. For each seed, all metrics are taken from the checkpoint with the lowest test loss. The best case is highlighted in bold.}
\label{tab:llama_iid_test_sem_ent_attack_fraction}
\begin{tabular}{lccccc}
\toprule
Method & 0\% & 10\% & 20\% & 30\% & 40\% \\
\midrule
FedAvg & 5.52 $\pm$ 0.02 & 5.80 $\pm$ 0.02 & 5.97 $\pm$ 0.12 & 6.12 $\pm$ 0.06 & 6.58 $\pm$ 0.34 \\
Norm-Bound & 5.70 $\pm$ 0.01 & 5.96 $\pm$ 0.01 & 6.43 $\pm$ 0.22 & 6.83 $\pm$ 0.07 & 6.95 $\pm$ 0.43 \\
Median & 5.55 $\pm$ 0.02 & 5.73 $\pm$ 0.03 & 5.81 $\pm$ 0.08 & 5.91 $\pm$ 0.02 & 5.97 $\pm$ 0.09 \\
Trimmed & \textbf{5.52 $\pm$ 0.02} & 5.73 $\pm$ 0.04 & 5.81 $\pm$ 0.09 & 5.89 $\pm$ 0.04 & 5.95 $\pm$ 0.08 \\
Multi-Krum & \textbf{5.52 $\pm$ 0.02} & 5.54 $\pm$ 0.04 & 5.55 $\pm$ 0.03 & 5.56 $\pm$ 0.03 & 5.61 $\pm$ 0.03 \\
FLAME & 5.64 $\pm$ 0.05 & 5.61 $\pm$ 0.04 & 5.64 $\pm$ 0.05 & 5.65 $\pm$ 0.04 & 5.74 $\pm$ 0.05 \\
FedLNS & \textbf{5.52 $\pm$ 0.02} & \textbf{5.52 $\pm$ 0.04} & \textbf{5.54 $\pm$ 0.03} & \textbf{5.52 $\pm$ 0.04} & \textbf{5.52 $\pm$ 0.02} \\
\bottomrule
\end{tabular}
\end{table}

\begin{table}[ht]\centering
\caption{Test Loss under different malicious-client fractions for LLaMA / StackOverflow with Non-IID partition. Values are mean $\pm$ standard deviation over three seeds. For each seed, all metrics are taken from the checkpoint with the lowest test loss. The best case is highlighted in bold.}
\label{tab:llama_noniid_test_loss_attack_fraction}
\begin{tabular}{lccccc}
\toprule
Method & 0\% & 10\% & 20\% & 30\% & 40\% \\
\midrule
FedAvg & \textbf{5.37 $\pm$ 0.03} & 5.41 $\pm$ 0.03 & 5.47 $\pm$ 0.04 & 5.54 $\pm$ 0.03 & 5.67 $\pm$ 0.03 \\
Norm-Bound & 5.51 $\pm$ 0.03 & 5.57 $\pm$ 0.03 & 5.65 $\pm$ 0.04 & 5.79 $\pm$ 0.02 & 5.83 $\pm$ 0.05 \\
Median & 5.40 $\pm$ 0.03 & 5.42 $\pm$ 0.03 & 5.44 $\pm$ 0.03 & 5.48 $\pm$ 0.03 & 5.56 $\pm$ 0.03 \\
Trimmed & 5.37 $\pm$ 0.03 & 5.40 $\pm$ 0.03 & 5.44 $\pm$ 0.04 & 5.49 $\pm$ 0.03 & 5.56 $\pm$ 0.04 \\
Multi-Krum & 5.37 $\pm$ 0.03 & 5.38 $\pm$ 0.03 & 5.39 $\pm$ 0.03 & 5.41 $\pm$ 0.03 & 5.45 $\pm$ 0.04 \\
FLAME & 5.49 $\pm$ 0.03 & 5.49 $\pm$ 0.03 & 5.49 $\pm$ 0.03 & 5.50 $\pm$ 0.03 & 5.59 $\pm$ 0.04 \\
FedLNS & \textbf{5.37 $\pm$ 0.03} & \textbf{5.38 $\pm$ 0.03} & \textbf{5.38 $\pm$ 0.03} & \textbf{5.38 $\pm$ 0.03} & \textbf{5.39 $\pm$ 0.03} \\
\bottomrule
\end{tabular}
\end{table}

\begin{table}[ht]\centering
\caption{Test PPL under different malicious-client fractions for LLaMA / StackOverflow with Non-IID partition. Values are mean $\pm$ standard deviation over three seeds. For each seed, all metrics are taken from the checkpoint with the lowest test loss. The best case is highlighted in bold.}
\label{tab:llama_noniid_test_ppl_attack_fraction}
\begin{tabular}{lccccc}
\toprule
Method & 0\% & 10\% & 20\% & 30\% & 40\% \\
\midrule
FedAvg & \textbf{215.38 $\pm$ 6.80} & 224.27 $\pm$ 6.68 & 237.35 $\pm$ 9.12 & 255.84 $\pm$ 7.17 & 288.74 $\pm$ 8.31 \\
Norm-Bound & 246.30 $\pm$ 6.76 & 261.73 $\pm$ 6.66 & 283.49 $\pm$ 10.99 & 326.42 $\pm$ 6.26 & 340.88 $\pm$ 18.09 \\
Median & 221.84 $\pm$ 7.19 & 225.51 $\pm$ 7.38 & 230.73 $\pm$ 7.84 & 239.51 $\pm$ 6.99 & 260.33 $\pm$ 7.99 \\
Trimmed & 215.55 $\pm$ 6.92 & 222.49 $\pm$ 7.43 & 230.96 $\pm$ 8.21 & 241.46 $\pm$ 8.09 & 259.56 $\pm$ 9.62 \\
Multi-Krum & 215.55 $\pm$ 6.92 & 218.10 $\pm$ 7.52 & 219.94 $\pm$ 7.17 & 222.74 $\pm$ 7.63 & 233.42 $\pm$ 10.15 \\
FLAME & 243.12 $\pm$ 6.82 & 242.99 $\pm$ 6.85 & 243.25 $\pm$ 6.41 & 245.58 $\pm$ 7.14 & 269.08 $\pm$ 11.90 \\
FedLNS & 215.55 $\pm$ 6.85 & \textbf{216.13 $\pm$ 6.70} & \textbf{216.97 $\pm$ 6.91} & \textbf{218.02 $\pm$ 7.07} & \textbf{220.18 $\pm$ 6.45} \\
\bottomrule
\end{tabular}
\end{table}

\begin{table}[ht]\centering
\caption{Semantic Entropy under different malicious-client fractions for LLaMA / StackOverflow with Non-IID partition. Values are mean $\pm$ standard deviation over three seeds. For each seed, all metrics are taken from the checkpoint with the lowest test loss. The best case is highlighted in bold.}
\label{tab:llama_noniid_test_sem_ent_attack_fraction}
\begin{tabular}{lccccc}
\toprule
Method & 0\% & 10\% & 20\% & 30\% & 40\% \\
\midrule
FedAvg & \textbf{5.53 $\pm$ 0.03} & 5.80 $\pm$ 0.01 & 5.94 $\pm$ 0.11 & 6.22 $\pm$ 0.11 & 6.61 $\pm$ 0.27 \\
Norm-Bound & 5.70 $\pm$ 0.03 & 5.98 $\pm$ 0.01 & 6.44 $\pm$ 0.21 & 6.90 $\pm$ 0.09 & 6.56 $\pm$ 0.42 \\
Median & 5.56 $\pm$ 0.03 & 5.76 $\pm$ 0.02 & 5.84 $\pm$ 0.09 & 5.91 $\pm$ 0.04 & 5.99 $\pm$ 0.12 \\
Trimmed & 5.53 $\pm$ 0.03 & 5.75 $\pm$ 0.03 & 5.84 $\pm$ 0.10 & 5.92 $\pm$ 0.06 & 5.96 $\pm$ 0.10 \\
Multi-Krum & 5.53 $\pm$ 0.03 & 5.55 $\pm$ 0.04 & 5.56 $\pm$ 0.03 & 5.56 $\pm$ 0.02 & 5.63 $\pm$ 0.06 \\
FLAME & 5.66 $\pm$ 0.06 & 5.63 $\pm$ 0.06 & 5.62 $\pm$ 0.01 & 5.66 $\pm$ 0.07 & 5.74 $\pm$ 0.06 \\
FedLNS & 5.54 $\pm$ 0.03 & \textbf{5.54 $\pm$ 0.03} & \textbf{5.53 $\pm$ 0.03} & \textbf{5.53 $\pm$ 0.02} & \textbf{5.52 $\pm$ 0.03} \\
\bottomrule
\end{tabular}
\end{table}

\clearpage
\section{Bank Coverage at Activation}
\label{app:bank_size_ablation}

This appendix studies the amount of bank coverage available when \FedLNS{} first activates screening and global aggregation. 
Appendix~\ref{app:min_bank_probability} analyzes the benign-majority probability of a partially populated bank; here, we evaluate how broader activation coverage affects the early active training trajectory.

\subsection{Reporting Protocol}

The bank activation threshold determines when the server begins screening and global aggregation. We therefore focus on the early active period immediately after each threshold is reached, where differences in bank coverage and activation timing are most visible.

For each seed and activation coverage, the activation round is the first communication round in which the bank reaches $M_{\mathrm{act}}$. We then consider the first ten active aggregation rounds. Within this window, the checkpoint with the lowest test loss is selected, and test loss, perplexity, and token-level semantic entropy are reported from the same checkpoint.

\subsection{Activation Rounds}

Tables~\ref{tab:bank_ablation_activation_rounds_iid_mal0}--\ref{tab:bank_ablation_activation_rounds_noniid_mal40} report the activation rounds for all population-level malicious-client fractions and partitions. The observed timing follows the pattern analyzed in Appendix~\ref{app:min_bank_probability}: $M_{\mathrm{act}}=0$ activates immediately, while thresholds covering $25\%$, $50\%$, and $75\%$ of the client population activate after approximately $3$, $7$, and $13$--$14$ rounds, respectively.

\begin{table}[ht] \centering
\caption{Estimated activation rounds for the bank-size ablation under IID partition with $0\%$ malicious clients. The activation round is the first communication round where the bank reaches $M_{\mathrm{act}}$.}
\label{tab:bank_ablation_activation_rounds_iid_mal0}
\begin{tabular}{lccc}
\toprule
$M_{\mathrm{act}}/K$ & GPT / WikiText & BERT / Shakespeare & LLaMA / StackOverflow \\
\midrule
0\% & 1.0 $\pm$ 0.0 & 1.0 $\pm$ 0.0 & 1.0 $\pm$ 0.0 \\
25\% & 3.0 $\pm$ 0.0 & 3.0 $\pm$ 0.0 & 3.0 $\pm$ 0.0 \\
50\% & 6.7 $\pm$ 0.6 & 6.7 $\pm$ 0.6 & 6.7 $\pm$ 0.6 \\
75\% & 13.3 $\pm$ 0.6 & 13.3 $\pm$ 0.6 & 13.3 $\pm$ 0.6 \\
\bottomrule
\end{tabular}
\end{table}

\begin{table}[ht] \centering
\caption{Estimated activation rounds for the bank-size ablation under IID partition with $10\%$ malicious clients. The activation round is the first communication round where the bank reaches $M_{\mathrm{act}}$.}
\label{tab:bank_ablation_activation_rounds_iid_mal10}
\begin{tabular}{lccc}
\toprule
$M_{\mathrm{act}}/K$ & GPT / WikiText & BERT / Shakespeare & LLaMA / StackOverflow \\
\midrule
0\% & 1.0 $\pm$ 0.0 & 1.0 $\pm$ 0.0 & 1.0 $\pm$ 0.0 \\
25\% & 3.0 $\pm$ 0.0 & 3.0 $\pm$ 0.0 & 3.0 $\pm$ 0.0 \\
50\% & 7.0 $\pm$ 0.0 & 7.0 $\pm$ 0.0 & 7.0 $\pm$ 0.0 \\
75\% & 13.0 $\pm$ 0.0 & 13.0 $\pm$ 0.0 & 13.0 $\pm$ 0.0 \\
\bottomrule
\end{tabular}
\end{table}

\begin{table}[ht] \centering
\caption{Estimated activation rounds for the bank-size ablation under IID partition with $20\%$ malicious clients. The activation round is the first communication round where the bank reaches $M_{\mathrm{act}}$.}
\label{tab:bank_ablation_activation_rounds_iid_mal20}
\begin{tabular}{lccc}
\toprule
$M_{\mathrm{act}}/K$ & GPT / WikiText & BERT / Shakespeare & LLaMA / StackOverflow \\
\midrule
0\% & 1.0 $\pm$ 0.0 & 1.0 $\pm$ 0.0 & 1.0 $\pm$ 0.0 \\
25\% & 3.0 $\pm$ 0.0 & 3.0 $\pm$ 0.0 & 3.0 $\pm$ 0.0 \\
50\% & 7.0 $\pm$ 0.0 & 7.0 $\pm$ 0.0 & 7.0 $\pm$ 0.0 \\
75\% & 13.3 $\pm$ 0.6 & 13.3 $\pm$ 0.6 & 13.3 $\pm$ 0.6 \\
\bottomrule
\end{tabular}
\end{table}

\begin{table}[ht] \centering
\caption{Estimated activation rounds for the bank-size ablation under IID partition with $30\%$ malicious clients. The activation round is the first communication round where the bank reaches $M_{\mathrm{act}}$.}
\label{tab:bank_ablation_activation_rounds_iid_mal30}
\begin{tabular}{lccc}
\toprule
$M_{\mathrm{act}}/K$ & GPT / WikiText & BERT / Shakespeare & LLaMA / StackOverflow \\
\midrule
0\% & 1.0 $\pm$ 0.0 & 1.0 $\pm$ 0.0 & 1.0 $\pm$ 0.0 \\
25\% & 3.0 $\pm$ 0.0 & 3.0 $\pm$ 0.0 & 3.0 $\pm$ 0.0 \\
50\% & 7.3 $\pm$ 0.6 & 7.3 $\pm$ 0.6 & 7.3 $\pm$ 0.6 \\
75\% & 14.0 $\pm$ 1.0 & 14.0 $\pm$ 1.0 & 14.0 $\pm$ 1.0 \\
\bottomrule
\end{tabular}
\end{table}

\begin{table}[ht] \centering
\caption{Estimated activation rounds for the bank-size ablation under IID partition with $40\%$ malicious clients. The activation round is the first communication round where the bank reaches $M_{\mathrm{act}}$.}
\label{tab:bank_ablation_activation_rounds_iid_mal40}
\begin{tabular}{lccc}
\toprule
$M_{\mathrm{act}}/K$ & GPT / WikiText & BERT / Shakespeare & LLaMA / StackOverflow \\
\midrule
0\% & 1.0 $\pm$ 0.0 & 1.0 $\pm$ 0.0 & 1.0 $\pm$ 0.0 \\
25\% & 3.0 $\pm$ 0.0 & 3.0 $\pm$ 0.0 & 3.0 $\pm$ 0.0 \\
50\% & 7.0 $\pm$ 0.0 & 7.0 $\pm$ 0.0 & 7.0 $\pm$ 0.0 \\
75\% & 14.0 $\pm$ 1.0 & 14.0 $\pm$ 1.0 & 14.0 $\pm$ 1.0 \\
\bottomrule
\end{tabular}
\end{table}

\begin{table}[ht] \centering
\caption{Estimated activation rounds for the bank-size ablation under Non-IID partition with $0\%$ malicious clients. The activation round is the first communication round where the bank reaches $M_{\mathrm{act}}$.}
\label{tab:bank_ablation_activation_rounds_noniid_mal0}
\begin{tabular}{lccc}
\toprule
$M_{\mathrm{act}}/K$ & GPT / WikiText & BERT / Shakespeare & LLaMA / StackOverflow \\
\midrule
0\% & 1.0 $\pm$ 0.0 & 1.0 $\pm$ 0.0 & 1.0 $\pm$ 0.0 \\
25\% & 3.0 $\pm$ 0.0 & 3.0 $\pm$ 0.0 & 3.0 $\pm$ 0.0 \\
50\% & 6.7 $\pm$ 0.6 & 6.7 $\pm$ 0.6 & 6.7 $\pm$ 0.6 \\
75\% & 13.3 $\pm$ 0.6 & 13.3 $\pm$ 0.6 & 13.3 $\pm$ 0.6 \\
\bottomrule
\end{tabular}
\end{table}

\begin{table}[ht] \centering
\caption{Estimated activation rounds for the bank-size ablation under Non-IID partition with $10\%$ malicious clients. The activation round is the first communication round where the bank reaches $M_{\mathrm{act}}$.}
\label{tab:bank_ablation_activation_rounds_noniid_mal10}
\begin{tabular}{lccc}
\toprule
$M_{\mathrm{act}}/K$ & GPT / WikiText & BERT / Shakespeare & LLaMA / StackOverflow \\
\midrule
0\% & 1.0 $\pm$ 0.0 & 1.0 $\pm$ 0.0 & 1.0 $\pm$ 0.0 \\
25\% & 3.0 $\pm$ 0.0 & 3.0 $\pm$ 0.0 & 3.0 $\pm$ 0.0 \\
50\% & 7.0 $\pm$ 0.0 & 7.0 $\pm$ 0.0 & 7.0 $\pm$ 0.0 \\
75\% & 13.0 $\pm$ 0.0 & 13.0 $\pm$ 0.0 & 13.0 $\pm$ 0.0 \\
\bottomrule
\end{tabular}
\end{table}

\begin{table}[ht] \centering
\caption{Estimated activation rounds for the bank-size ablation under Non-IID partition with $20\%$ malicious clients. The activation round is the first communication round where the bank reaches $M_{\mathrm{act}}$.}
\label{tab:bank_ablation_activation_rounds_noniid_mal20}
\begin{tabular}{lccc}
\toprule
$M_{\mathrm{act}}/K$ & GPT / WikiText & BERT / Shakespeare & LLaMA / StackOverflow \\
\midrule
0\% & 1.0 $\pm$ 0.0 & 1.0 $\pm$ 0.0 & 1.0 $\pm$ 0.0 \\
25\% & 3.0 $\pm$ 0.0 & 3.0 $\pm$ 0.0 & 3.0 $\pm$ 0.0 \\
50\% & 7.0 $\pm$ 0.0 & 7.0 $\pm$ 0.0 & 7.0 $\pm$ 0.0 \\
75\% & 13.3 $\pm$ 0.6 & 13.3 $\pm$ 0.6 & 13.3 $\pm$ 0.6 \\
\bottomrule
\end{tabular}
\end{table}

\begin{table}[ht] \centering
\caption{Estimated activation rounds for the bank-size ablation under Non-IID partition with $30\%$ malicious clients. The activation round is the first communication round where the bank reaches $M_{\mathrm{act}}$.}
\label{tab:bank_ablation_activation_rounds_noniid_mal30}
\begin{tabular}{lccc}
\toprule
$M_{\mathrm{act}}/K$ & GPT / WikiText & BERT / Shakespeare & LLaMA / StackOverflow \\
\midrule
0\% & 1.0 $\pm$ 0.0 & 1.0 $\pm$ 0.0 & 1.0 $\pm$ 0.0 \\
25\% & 3.0 $\pm$ 0.0 & 3.0 $\pm$ 0.0 & 3.0 $\pm$ 0.0 \\
50\% & 7.3 $\pm$ 0.6 & 7.3 $\pm$ 0.6 & 7.3 $\pm$ 0.6 \\
75\% & 14.0 $\pm$ 1.0 & 14.0 $\pm$ 1.0 & 14.0 $\pm$ 1.0 \\
\bottomrule
\end{tabular}
\end{table}

\begin{table}[ht] \centering
\caption{Estimated activation rounds for the bank-size ablation under Non-IID partition with $40\%$ malicious clients. The activation round is the first communication round where the bank reaches $M_{\mathrm{act}}$.}
\label{tab:bank_ablation_activation_rounds_noniid_mal40}
\begin{tabular}{lccc}
\toprule
$M_{\mathrm{act}}/K$ & GPT / WikiText & BERT / Shakespeare & LLaMA / StackOverflow \\
\midrule
0\% & 1.0 $\pm$ 0.0 & 1.0 $\pm$ 0.0 & 1.0 $\pm$ 0.0 \\
25\% & 3.0 $\pm$ 0.0 & 3.0 $\pm$ 0.0 & 3.0 $\pm$ 0.0 \\
50\% & 7.0 $\pm$ 0.0 & 7.0 $\pm$ 0.0 & 7.0 $\pm$ 0.0 \\
75\% & 14.0 $\pm$ 1.0 & 14.0 $\pm$ 1.0 & 14.0 $\pm$ 1.0 \\
\bottomrule
\end{tabular}
\end{table}

\clearpage
\subsection{Results Under IID Partitions}

Tables~\ref{tab:bank_ablation_all_models_iid_mal0_test_loss_active10}--\ref{tab:bank_ablation_all_models_iid_mal40_test_sem_ent_active10} report the IID bank-threshold results. Across model families and attack fractions, larger thresholds generally improve the early active-window values; these comparisons reflect both broader client coverage in the bank and later activation of global aggregation.

At the strongest $40\%$ malicious-client setting, the $75\%$ activation coverage gives the lowest test loss and perplexity for all three model families under IID partitioning:
perplexity falls from $1820.21\pm4.44$ to $1626.87\pm9.50$ for GPT/WikiText, from $1532.76\pm114.78$ to $1043.96\pm22.20$ for BERT/Tiny Shakespeare, and from $2408.15\pm162.25$ to $609.87\pm24.44$ for LLaMA/StackOverflow. Token-level semantic entropy also decreases substantially for BERT and LLaMA, indicating that broader bank coverage can stabilize the early active trajectory.

\begin{table}[ht] \centering
\caption{Bank-size threshold ablation under IID partition with $0\%$ malicious clients. Metric: Test Loss. For each seed and threshold, we select the checkpoint with the lowest test loss within the first 10 communication rounds after the bank reaches $M_{\mathrm{act}}$. The best case is highlighted in bold.}
\label{tab:bank_ablation_all_models_iid_mal0_test_loss_active10}
\begin{tabular}{lccc}
\toprule
$M_{\mathrm{act}}/K$ & GPT / WikiText & BERT / Shakespeare & LLaMA / StackOverflow \\
\midrule
0\% & 7.50 $\pm$ 0.00 & 7.33 $\pm$ 0.01 & 7.72 $\pm$ 0.04 \\
25\% & 7.48 $\pm$ 0.01 & 7.20 $\pm$ 0.03 & 7.37 $\pm$ 0.04 \\
50\% & 7.44 $\pm$ 0.01 & 7.04 $\pm$ 0.04 & 6.87 $\pm$ 0.09 \\
75\% & \textbf{7.38 $\pm$ 0.01} & \textbf{6.94 $\pm$ 0.02} & \textbf{6.42 $\pm$ 0.03} \\
\bottomrule
\end{tabular}
\end{table}

\begin{table}[ht] \centering \centering
\caption{Bank-size threshold ablation under IID partition with $0\%$ malicious clients. Metric: Test PPL. For each seed and threshold, we select the checkpoint with the lowest test loss within the first 10 communication rounds after the bank reaches $M_{\mathrm{act}}$. The best case is highlighted in bold.}
\label{tab:bank_ablation_all_models_iid_mal0_test_ppl_active10}
\begin{tabular}{lccc}
\toprule
$M_{\mathrm{act}}/K$ & GPT / WikiText & BERT / Shakespeare & LLaMA / StackOverflow \\
\midrule
0\% & 1802.40 $\pm$ 2.97 & 1519.44 $\pm$ 14.38 & 2258.01 $\pm$ 88.68 \\
25\% & 1764.16 $\pm$ 11.60 & 1341.88 $\pm$ 42.96 & 1584.71 $\pm$ 60.01 \\
50\% & 1695.11 $\pm$ 18.37 & 1143.14 $\pm$ 48.58 & 969.72 $\pm$ 83.50 \\
75\% & \textbf{1601.83 $\pm$ 17.25} & \textbf{1028.47 $\pm$ 24.40} & \textbf{616.50 $\pm$ 19.83} \\
\bottomrule
\end{tabular}
\end{table}

\begin{table}[ht] \centering
\caption{Bank-size threshold ablation under IID partition with $0\%$ malicious clients. Metric: Semantic Entropy. For each seed and threshold, we select the checkpoint with the lowest test loss within the first 10 communication rounds after the bank reaches $M_{\mathrm{act}}$. The best case is highlighted in bold.}
\label{tab:bank_ablation_all_models_iid_mal0_test_sem_ent_active10}
\begin{tabular}{lccc}
\toprule
$M_{\mathrm{act}}/K$ & GPT / WikiText & BERT / Shakespeare & LLaMA / StackOverflow \\
\midrule
0\% & 7.01 $\pm$ 0.03 & 7.46 $\pm$ 0.02 & 9.95 $\pm$ 0.02 \\
25\% & 6.96 $\pm$ 0.09 & 7.26 $\pm$ 0.38 & 9.67 $\pm$ 0.03 \\
50\% & 6.90 $\pm$ 0.00 & 6.39 $\pm$ 0.16 & 8.88 $\pm$ 0.18 \\
75\% & \textbf{6.90 $\pm$ 0.06} & \textbf{6.37 $\pm$ 0.12} & \textbf{7.49 $\pm$ 0.04} \\
\bottomrule
\end{tabular}
\end{table}

\begin{table}[ht] \centering
\caption{Bank-size threshold ablation under IID partition with $10\%$ malicious clients. Metric: Test Loss. For each seed and threshold, we select the checkpoint with the lowest test loss within the first 10 communication rounds after the bank reaches $M_{\mathrm{act}}$. The best case is highlighted in bold.}
\label{tab:bank_ablation_all_models_iid_mal10_test_loss_active10}
\begin{tabular}{lccc}
\toprule
$M_{\mathrm{act}}/K$ & GPT / WikiText & BERT / Shakespeare & LLaMA / StackOverflow \\
\midrule
0\% & 7.50 $\pm$ 0.00 & 7.32 $\pm$ 0.01 & 7.72 $\pm$ 0.04 \\
25\% & 7.48 $\pm$ 0.01 & 7.18 $\pm$ 0.05 & 7.37 $\pm$ 0.04 \\
50\% & 7.44 $\pm$ 0.01 & 7.06 $\pm$ 0.04 & 6.84 $\pm$ 0.03 \\
75\% & \textbf{7.39 $\pm$ 0.00} & \textbf{6.92 $\pm$ 0.02} & \textbf{6.44 $\pm$ 0.03} \\
\bottomrule
\end{tabular}
\end{table}

\begin{table}[ht] \centering
\caption{Bank-size threshold ablation under IID partition with $10\%$ malicious clients. Metric: Test PPL. For each seed and threshold, we select the checkpoint with the lowest test loss within the first 10 communication rounds after the bank reaches $M_{\mathrm{act}}$. The best case is highlighted in bold.}
\label{tab:bank_ablation_all_models_iid_mal10_test_ppl_active10}
\begin{tabular}{lccc}
\toprule
$M_{\mathrm{act}}/K$ & GPT / WikiText & BERT / Shakespeare & LLaMA / StackOverflow \\
\midrule
0\% & 1808.57 $\pm$ 3.95 & 1512.51 $\pm$ 12.15 & 2256.49 $\pm$ 90.48 \\
25\% & 1772.48 $\pm$ 13.23 & 1310.31 $\pm$ 71.29 & 1585.07 $\pm$ 62.27 \\
50\% & 1701.03 $\pm$ 11.07 & 1163.34 $\pm$ 42.36 & 933.03 $\pm$ 30.89 \\
75\% & \textbf{1615.82 $\pm$ 5.52} & \textbf{1011.33 $\pm$ 20.24} & \textbf{625.74 $\pm$ 17.90} \\
\bottomrule
\end{tabular}
\end{table}

\begin{table}[ht] \centering
\caption{Bank-size threshold ablation under IID partition with $10\%$ malicious clients. Metric: Semantic Entropy. For each seed and threshold, we select the checkpoint with the lowest test loss within the first 10 communication rounds after the bank reaches $M_{\mathrm{act}}$. The best case is highlighted in bold.}
\label{tab:bank_ablation_all_models_iid_mal10_test_sem_ent_active10}
\begin{tabular}{lccc}
\toprule
$M_{\mathrm{act}}/K$ & GPT / WikiText & BERT / Shakespeare & LLaMA / StackOverflow \\
\midrule
0\% & 7.01 $\pm$ 0.10 & 7.51 $\pm$ 0.04 & 9.95 $\pm$ 0.02 \\
25\% & 6.98 $\pm$ 0.06 & 6.88 $\pm$ 0.12 & 9.67 $\pm$ 0.03 \\
50\% & 6.92 $\pm$ 0.05 & 6.45 $\pm$ 0.01 & 8.79 $\pm$ 0.06 \\
75\% & \textbf{6.91 $\pm$ 0.06} & \textbf{6.33 $\pm$ 0.03} & \textbf{7.54 $\pm$ 0.04} \\
\bottomrule
\end{tabular}
\end{table}

\begin{table}[ht] \centering
\caption{Bank-size threshold ablation under IID partition with $20\%$ malicious clients. Metric: Test Loss. For each seed and threshold, we select the checkpoint with the lowest test loss within the first 10 communication rounds after the bank reaches $M_{\mathrm{act}}$. The best case is highlighted in bold.}
\label{tab:bank_ablation_all_models_iid_mal20_test_loss_active10}
\begin{tabular}{lccc}
\toprule
$M_{\mathrm{act}}/K$ & GPT / WikiText & BERT / Shakespeare & LLaMA / StackOverflow \\
\midrule
0\% & 7.50 $\pm$ 0.01 & 7.32 $\pm$ 0.01 & 7.72 $\pm$ 0.04 \\
25\% & 7.48 $\pm$ 0.01 & 7.19 $\pm$ 0.05 & 7.37 $\pm$ 0.04 \\
50\% & 7.43 $\pm$ 0.01 & 7.05 $\pm$ 0.01 & 6.84 $\pm$ 0.03 \\
75\% & \textbf{7.39 $\pm$ 0.01} & \textbf{6.93 $\pm$ 0.04} & \textbf{6.43 $\pm$ 0.02} \\
\bottomrule
\end{tabular}
\end{table}

\begin{table}[ht] \centering
\caption{Bank-size threshold ablation under IID partition with $20\%$ malicious clients. Metric: Test PPL. For each seed and threshold, we select the checkpoint with the lowest test loss within the first 10 communication rounds after the bank reaches $M_{\mathrm{act}}$. The best case is highlighted in bold.}
\label{tab:bank_ablation_all_models_iid_mal20_test_ppl_active10}
\begin{tabular}{lccc}
\toprule
$M_{\mathrm{act}}/K$ & GPT / WikiText & BERT / Shakespeare & LLaMA / StackOverflow \\
\midrule
0\% & 1810.07 $\pm$ 10.64 & 1505.92 $\pm$ 11.61 & 2256.98 $\pm$ 91.12 \\
25\% & 1774.75 $\pm$ 21.70 & 1328.38 $\pm$ 66.28 & 1585.63 $\pm$ 62.41 \\
50\% & 1693.93 $\pm$ 23.24 & 1157.90 $\pm$ 17.29 & 933.10 $\pm$ 31.26 \\
75\% & \textbf{1612.54 $\pm$ 12.60} & \textbf{1027.72 $\pm$ 37.62} & \textbf{617.43 $\pm$ 15.49} \\
\bottomrule
\end{tabular}
\end{table}

\begin{table}[ht] \centering
\caption{Bank-size threshold ablation under IID partition with $20\%$ malicious clients. Metric: Semantic Entropy. For each seed and threshold, we select the checkpoint with the lowest test loss within the first 10 communication rounds after the bank reaches $M_{\mathrm{act}}$. The best case is highlighted in bold.}
\label{tab:bank_ablation_all_models_iid_mal20_test_sem_ent_active10}
\begin{tabular}{lccc}
\toprule
$M_{\mathrm{act}}/K$ & GPT / WikiText & BERT / Shakespeare & LLaMA / StackOverflow \\
\midrule
0\% & 7.02 $\pm$ 0.07 & 7.52 $\pm$ 0.06 & 9.95 $\pm$ 0.02 \\
25\% & 6.92 $\pm$ 0.03 & 6.91 $\pm$ 0.13 & 9.67 $\pm$ 0.03 \\
50\% & 6.95 $\pm$ 0.04 & 6.48 $\pm$ 0.11 & 8.79 $\pm$ 0.06 \\
75\% & \textbf{6.88 $\pm$ 0.07} & \textbf{6.29 $\pm$ 0.05} & \textbf{7.50 $\pm$ 0.07} \\
\bottomrule
\end{tabular}
\end{table}

\begin{table}[ht] \centering
\caption{Bank-size threshold ablation under IID partition with $30\%$ malicious clients. Metric: Test Loss. For each seed and threshold, we select the checkpoint with the lowest test loss within the first 10 communication rounds after the bank reaches $M_{\mathrm{act}}$. The best case is highlighted in bold.}
\label{tab:bank_ablation_all_models_iid_mal30_test_loss_active10}
\begin{tabular}{lccc}
\toprule
$M_{\mathrm{act}}/K$ & GPT / WikiText & BERT / Shakespeare & LLaMA / StackOverflow \\
\midrule
0\% & 7.50 $\pm$ 0.01 & 7.32 $\pm$ 0.07 & 7.72 $\pm$ 0.04 \\
25\% & 7.48 $\pm$ 0.01 & 7.18 $\pm$ 0.04 & 7.37 $\pm$ 0.04 \\
50\% & 7.44 $\pm$ 0.01 & 7.06 $\pm$ 0.05 & 6.81 $\pm$ 0.05 \\
75\% & \textbf{7.39 $\pm$ 0.01} & \textbf{6.92 $\pm$ 0.03} & \textbf{6.40 $\pm$ 0.03} \\
\bottomrule
\end{tabular}
\end{table}

\begin{table}[ht] \centering
\caption{Bank-size threshold ablation under IID partition with $30\%$ malicious clients. Metric: Test PPL. For each seed and threshold, we select the checkpoint with the lowest test loss within the first 10 communication rounds after the bank reaches $M_{\mathrm{act}}$. The best case is highlighted in bold.}
\label{tab:bank_ablation_all_models_iid_mal30_test_ppl_active10}
\begin{tabular}{lccc}
\toprule
$M_{\mathrm{act}}/K$ & GPT / WikiText & BERT / Shakespeare & LLaMA / StackOverflow \\
\midrule
0\% & 1813.92 $\pm$ 12.17 & 1508.04 $\pm$ 110.34 & 2257.69 $\pm$ 91.15 \\
25\% & 1778.81 $\pm$ 10.18 & 1312.65 $\pm$ 46.41 & 1585.24 $\pm$ 62.65 \\
50\% & 1694.88 $\pm$ 16.38 & 1169.29 $\pm$ 61.52 & 903.36 $\pm$ 43.16 \\
75\% & \textbf{1619.16 $\pm$ 11.42} & \textbf{1011.10 $\pm$ 35.22} & \textbf{600.06 $\pm$ 15.82} \\
\bottomrule
\end{tabular}
\end{table}

\begin{table}[ht] \centering
\caption{Bank-size threshold ablation under IID partition with $30\%$ malicious clients. Metric: Semantic Entropy. For each seed and threshold, we select the checkpoint with the lowest test loss within the first 10 communication rounds after the bank reaches $M_{\mathrm{act}}$. The best case is highlighted in bold.}
\label{tab:bank_ablation_all_models_iid_mal30_test_sem_ent_active10}
\begin{tabular}{lccc}
\toprule
$M_{\mathrm{act}}/K$ & GPT / WikiText & BERT / Shakespeare & LLaMA / StackOverflow \\
\midrule
0\% & 7.03 $\pm$ 0.08 & 7.51 $\pm$ 0.07 & 9.95 $\pm$ 0.02 \\
25\% & 6.96 $\pm$ 0.02 & 6.94 $\pm$ 0.18 & 9.67 $\pm$ 0.03 \\
50\% & 6.96 $\pm$ 0.04 & 6.43 $\pm$ 0.04 & 8.71 $\pm$ 0.12 \\
75\% & \textbf{6.88 $\pm$ 0.07} & \textbf{6.40 $\pm$ 0.07} & \textbf{7.42 $\pm$ 0.08} \\
\bottomrule
\end{tabular}
\end{table}

\begin{table}[ht] \centering
\caption{Bank-size threshold ablation under IID partition with $40\%$ malicious clients. Metric: Test Loss. For each seed and threshold, we select the checkpoint with the lowest test loss within the first 10 communication rounds after the bank reaches $M_{\mathrm{act}}$. The best case is highlighted in bold.}
\label{tab:bank_ablation_all_models_iid_mal40_test_loss_active10}
\begin{tabular}{lccc}
\toprule
$M_{\mathrm{act}}/K$ & GPT / WikiText & BERT / Shakespeare & LLaMA / StackOverflow \\
\midrule
0\% & 7.51 $\pm$ 0.00 & 7.33 $\pm$ 0.07 & 7.79 $\pm$ 0.07 \\
25\% & 7.49 $\pm$ 0.01 & 7.24 $\pm$ 0.04 & 7.42 $\pm$ 0.05 \\
50\% & 7.45 $\pm$ 0.00 & 7.07 $\pm$ 0.05 & 6.87 $\pm$ 0.02 \\
75\% & \textbf{7.39 $\pm$ 0.01} & \textbf{6.95 $\pm$ 0.02} & \textbf{6.41 $\pm$ 0.04} \\
\bottomrule
\end{tabular}
\end{table}

\begin{table}[ht] \centering
\caption{Bank-size threshold ablation under IID partition with $40\%$ malicious clients. Metric: Test PPL. For each seed and threshold, we select the checkpoint with the lowest test loss within the first 10 communication rounds after the bank reaches $M_{\mathrm{act}}$. The best case is highlighted in bold.}
\label{tab:bank_ablation_all_models_iid_mal40_test_ppl_active10}
\begin{tabular}{lccc}
\toprule
$M_{\mathrm{act}}/K$ & GPT / WikiText & BERT / Shakespeare & LLaMA / StackOverflow \\
\midrule
0\% & 1820.21 $\pm$ 4.44 & 1532.76 $\pm$ 114.78 & 2408.15 $\pm$ 162.25 \\
25\% & 1782.72 $\pm$ 10.82 & 1388.21 $\pm$ 62.55 & 1674.97 $\pm$ 86.91 \\
50\% & 1718.82 $\pm$ 4.65 & 1171.92 $\pm$ 53.02 & 964.01 $\pm$ 22.51 \\
75\% & \textbf{1626.87 $\pm$ 9.50} & \textbf{1043.96 $\pm$ 22.20} & \textbf{609.87 $\pm$ 24.44} \\
\bottomrule
\end{tabular}
\end{table}

\begin{table}[ht] \centering
\caption{Bank-size threshold ablation under IID partition with $40\%$ malicious clients. Metric: Semantic Entropy. For each seed and threshold, we select the checkpoint with the lowest test loss within the first 10 communication rounds after the bank reaches $M_{\mathrm{act}}$. The best case is highlighted in bold.}
\label{tab:bank_ablation_all_models_iid_mal40_test_sem_ent_active10}
\begin{tabular}{lccc}
\toprule
$M_{\mathrm{act}}/K$ & GPT / WikiText & BERT / Shakespeare & LLaMA / StackOverflow \\
\midrule
0\% & 7.11 $\pm$ 0.06 & 7.92 $\pm$ 0.48 & 9.99 $\pm$ 0.04 \\
25\% & 7.02 $\pm$ 0.04 & 7.13 $\pm$ 0.22 & 9.73 $\pm$ 0.06 \\
50\% & 6.97 $\pm$ 0.07 & 6.53 $\pm$ 0.06 & 8.88 $\pm$ 0.09 \\
75\% & \textbf{6.93 $\pm$ 0.03} & \textbf{6.32 $\pm$ 0.10} & \textbf{7.46 $\pm$ 0.14} \\
\bottomrule
\end{tabular}
\end{table}

\clearpage
\subsection{Results Under Non-IID Partitions}

Tables~\ref{tab:bank_ablation_all_models_noniid_mal0_test_loss_active10}--\ref{tab:bank_ablation_all_models_noniid_mal40_test_sem_ent_active10} report the non-IID results. Larger thresholds again generally improve the early active-window values, while also postponing the first global aggregation rounds.

At the strongest $40\%$ malicious-client setting, the $75\%$ activation coverage gives the lowest test loss and perplexity for all three model families under Non-IID partitioning:
perplexity falls from $1920.35\pm56.71$ to $1698.50\pm44.56$ for GPT/WikiText and from $1514.80\pm139.13$ to $1078.90\pm33.23$ for BERT/Tiny Shakespeare, while for LLaMA/StackOverflow the $0\%$ threshold is highly unstable ($5266.58\pm4615.45$) and the $75\%$ threshold reduces this to $759.15\pm249.94$, which shows that delaying screening until the bank has broader coverage can substantially stabilize early active training under severe attack and non-IID heterogeneity.

For token-level semantic entropy, the $75\%$ threshold is again best for BERT and LLaMA at the strongest attack setting, while for GPT/non-IID the $50\%$ threshold gives the lowest value with $75\%$ close behind, indicating that the best threshold can vary slightly by metric, though the overall pattern still favors larger bank coverage for stable early training.

\begin{table}[ht] \centering
\caption{Bank-size threshold ablation under Non-IID partition with $0\%$ malicious clients. Metric: Test Loss. For each seed and threshold, we select the checkpoint with the lowest test loss within the first 10 communication rounds after the bank reaches $M_{\mathrm{act}}$. The best case is highlighted in bold.}
\label{tab:bank_ablation_all_models_noniid_mal0_test_loss_active10}
\begin{tabular}{lccc}
\toprule
$M_{\mathrm{act}}/K$ & GPT / WikiText & BERT / Shakespeare & LLaMA / StackOverflow \\
\midrule
0\% & 7.58 $\pm$ 0.01 & 7.35 $\pm$ 0.03 & 7.73 $\pm$ 0.05 \\
25\% & 7.55 $\pm$ 0.02 & 7.25 $\pm$ 0.01 & 7.37 $\pm$ 0.05 \\
50\% & 7.51 $\pm$ 0.03 & 7.11 $\pm$ 0.06 & 6.88 $\pm$ 0.09 \\
75\% & \textbf{7.44 $\pm$ 0.02} & \textbf{6.96 $\pm$ 0.04} & \textbf{6.43 $\pm$ 0.04} \\
\bottomrule
\end{tabular}
\end{table}

\begin{table}[ht] \centering
\caption{Bank-size threshold ablation under Non-IID partition with $0\%$ malicious clients. Metric: Test PPL. For each seed and threshold, we select the checkpoint with the lowest test loss within the first 10 communication rounds after the bank reaches $M_{\mathrm{act}}$. The best case is highlighted in bold.}
\label{tab:bank_ablation_all_models_noniid_mal0_test_ppl_active10}
\begin{tabular}{lccc}
\toprule
$M_{\mathrm{act}}/K$ & GPT / WikiText & BERT / Shakespeare & LLaMA / StackOverflow \\
\midrule
0\% & 1959.50 $\pm$ 22.90 & 1550.67 $\pm$ 47.69 & 2273.70 $\pm$ 102.86 \\
25\% & 1896.90 $\pm$ 42.06 & 1403.02 $\pm$ 11.73 & 1596.15 $\pm$ 71.84 \\
50\% & 1823.97 $\pm$ 46.01 & 1221.73 $\pm$ 76.35 & 976.65 $\pm$ 92.33 \\
75\% & \textbf{1710.88 $\pm$ 32.52} & \textbf{1050.53 $\pm$ 37.04} & \textbf{622.54 $\pm$ 22.03} \\
\bottomrule
\end{tabular}
\end{table}

\begin{table}[ht] \centering
\caption{Bank-size threshold ablation under Non-IID partition with $0\%$ malicious clients. Metric: Semantic Entropy. For each seed and threshold, we select the checkpoint with the lowest test loss within the first 10 communication rounds after the bank reaches $M_{\mathrm{act}}$. The best case is highlighted in bold.}
\label{tab:bank_ablation_all_models_noniid_mal0_test_sem_ent_active10}
\begin{tabular}{lccc}
\toprule
$M_{\mathrm{act}}/K$ & GPT / WikiText & BERT / Shakespeare & LLaMA / StackOverflow \\
\midrule
0\% & 6.86 $\pm$ 0.10 & 7.82 $\pm$ 0.38 & 9.95 $\pm$ 0.02 \\
25\% & 6.82 $\pm$ 0.03 & 6.95 $\pm$ 0.18 & 9.68 $\pm$ 0.04 \\
50\% & 6.72 $\pm$ 0.03 & 6.60 $\pm$ 0.15 & 8.88 $\pm$ 0.20 \\
75\% & \textbf{6.71 $\pm$ 0.03} & \textbf{6.33 $\pm$ 0.07} & \textbf{7.48 $\pm$ 0.06} \\
\bottomrule
\end{tabular}
\end{table}

\begin{table}[ht] \centering
\caption{Bank-size threshold ablation under Non-IID partition with $10\%$ malicious clients. Metric: Test Loss. For each seed and threshold, we select the checkpoint with the lowest test loss within the first 10 communication rounds after the bank reaches $M_{\mathrm{act}}$. The best case is highlighted in bold.}
\label{tab:bank_ablation_all_models_noniid_mal10_test_loss_active10}
\begin{tabular}{lccc}
\toprule
$M_{\mathrm{act}}/K$ & GPT / WikiText & BERT / Shakespeare & LLaMA / StackOverflow \\
\midrule
0\% & 7.55 $\pm$ 0.02 & 7.33 $\pm$ 0.01 & 7.73 $\pm$ 0.05 \\
25\% & 7.52 $\pm$ 0.02 & 7.25 $\pm$ 0.01 & 7.37 $\pm$ 0.05 \\
50\% & 7.48 $\pm$ 0.03 & 7.10 $\pm$ 0.07 & 6.84 $\pm$ 0.04 \\
75\% & \textbf{7.42 $\pm$ 0.02} & \textbf{6.95 $\pm$ 0.05} & \textbf{6.44 $\pm$ 0.03} \\
\bottomrule
\end{tabular}
\end{table}

\begin{table}[ht] \centering
\caption{Bank-size threshold ablation under Non-IID partition with $10\%$ malicious clients. Metric: Test PPL. For each seed and threshold, we select the checkpoint with the lowest test loss within the first 10 communication rounds after the bank reaches $M_{\mathrm{act}}$. The best case is highlighted in bold.}
\label{tab:bank_ablation_all_models_noniid_mal10_test_ppl_active10}
\begin{tabular}{lccc}
\toprule
$M_{\mathrm{act}}/K$ & GPT / WikiText & BERT / Shakespeare & LLaMA / StackOverflow \\
\midrule
0\% & 1899.89 $\pm$ 44.69 & 1531.88 $\pm$ 13.60 & 2265.83 $\pm$ 100.94 \\
25\% & 1853.08 $\pm$ 43.58 & 1402.99 $\pm$ 7.86 & 1592.05 $\pm$ 71.15 \\
50\% & 1768.44 $\pm$ 51.11 & 1213.39 $\pm$ 83.14 & 937.25 $\pm$ 35.56 \\
75\% & \textbf{1672.14 $\pm$ 37.85} & \textbf{1047.93 $\pm$ 54.20} & \textbf{628.96 $\pm$ 18.82} \\
\bottomrule
\end{tabular}
\end{table}

\begin{table}[ht] \centering
\caption{Bank-size threshold ablation under Non-IID partition with $10\%$ malicious clients. Metric: Semantic Entropy. For each seed and threshold, we select the checkpoint with the lowest test loss within the first 10 communication rounds after the bank reaches $M_{\mathrm{act}}$. The best case is highlighted in bold.}
\label{tab:bank_ablation_all_models_noniid_mal10_test_sem_ent_active10}
\begin{tabular}{lccc}
\toprule
$M_{\mathrm{act}}/K$ & GPT / WikiText & BERT / Shakespeare & LLaMA / StackOverflow \\
\midrule
0\% & 6.76 $\pm$ 0.01 & 7.63 $\pm$ 0.05 & 9.96 $\pm$ 0.02 \\
25\% & 6.71 $\pm$ 0.01 & 6.97 $\pm$ 0.15 & 9.68 $\pm$ 0.04 \\
50\% & \textbf{6.68 $\pm$ 0.05} & 6.61 $\pm$ 0.19 & 8.80 $\pm$ 0.08 \\
75\% & 6.71 $\pm$ 0.02 & \textbf{6.42 $\pm$ 0.02} & \textbf{7.53 $\pm$ 0.05} \\
\bottomrule
\end{tabular}
\end{table}

\begin{table}[ht] \centering
\caption{Bank-size threshold ablation under Non-IID partition with $20\%$ malicious clients. Metric: Test Loss. For each seed and threshold, we select the checkpoint with the lowest test loss within the first 10 communication rounds after the bank reaches $M_{\mathrm{act}}$. The best case is highlighted in bold.}
\label{tab:bank_ablation_all_models_noniid_mal20_test_loss_active10}
\begin{tabular}{lccc}
\toprule
$M_{\mathrm{act}}/K$ & GPT / WikiText & BERT / Shakespeare & LLaMA / StackOverflow \\
\midrule
0\% & 7.55 $\pm$ 0.02 & 7.27 $\pm$ 0.12 & 7.73 $\pm$ 0.05 \\
25\% & 7.52 $\pm$ 0.02 & 7.14 $\pm$ 0.17 & 7.38 $\pm$ 0.05 \\
50\% & 7.47 $\pm$ 0.04 & 7.08 $\pm$ 0.16 & 6.85 $\pm$ 0.04 \\
75\% & \textbf{7.42 $\pm$ 0.03} & \textbf{6.95 $\pm$ 0.11} & \textbf{6.43 $\pm$ 0.02} \\
\bottomrule
\end{tabular}
\end{table}

\begin{table}[ht] \centering
\caption{Bank-size threshold ablation under Non-IID partition with $20\%$ malicious clients. Metric: Test PPL. For each seed and threshold, we select the checkpoint with the lowest test loss within the first 10 communication rounds after the bank reaches $M_{\mathrm{act}}$. The best case is highlighted in bold.}
\label{tab:bank_ablation_all_models_noniid_mal20_test_ppl_active10}
\begin{tabular}{lccc}
\toprule
$M_{\mathrm{act}}/K$ & GPT / WikiText & BERT / Shakespeare & LLaMA / StackOverflow \\
\midrule
0\% & 1892.83 $\pm$ 45.80 & 1439.58 $\pm$ 163.65 & 2284.63 $\pm$ 122.32 \\
25\% & 1843.09 $\pm$ 43.59 & 1270.95 $\pm$ 202.01 & 1603.79 $\pm$ 84.10 \\
50\% & 1759.53 $\pm$ 65.05 & 1199.79 $\pm$ 182.48 & 941.83 $\pm$ 41.03 \\
75\% & \textbf{1673.23 $\pm$ 42.10} & \textbf{1048.68 $\pm$ 111.65} & \textbf{622.50 $\pm$ 15.12} \\
\bottomrule
\end{tabular}
\end{table}

\begin{table}[ht] \centering
\caption{Bank-size threshold ablation under Non-IID partition with $20\%$ malicious clients. Metric: Semantic Entropy. For each seed and threshold, we select the checkpoint with the lowest test loss within the first 10 communication rounds after the bank reaches $M_{\mathrm{act}}$. The best case is highlighted in bold.}
\label{tab:bank_ablation_all_models_noniid_mal20_test_sem_ent_active10}
\begin{tabular}{lccc}
\toprule
$M_{\mathrm{act}}/K$ & GPT / WikiText & BERT / Shakespeare & LLaMA / StackOverflow \\
\midrule
0\% & 6.78 $\pm$ 0.08 & 7.47 $\pm$ 0.31 & 9.96 $\pm$ 0.03 \\
25\% & 6.84 $\pm$ 0.10 & 6.84 $\pm$ 0.45 & 9.69 $\pm$ 0.05 \\
50\% & 6.75 $\pm$ 0.06 & 6.59 $\pm$ 0.16 & 8.81 $\pm$ 0.08 \\
75\% & \textbf{6.74 $\pm$ 0.02} & \textbf{6.48 $\pm$ 0.10} & \textbf{7.50 $\pm$ 0.06} \\
\bottomrule
\end{tabular}
\end{table}

\begin{table}[ht] \centering
\caption{Bank-size threshold ablation under Non-IID partition with $30\%$ malicious clients. Metric: Test Loss. For each seed and threshold, we select the checkpoint with the lowest test loss within the first 10 communication rounds after the bank reaches $M_{\mathrm{act}}$. The best case is highlighted in bold.}
\label{tab:bank_ablation_all_models_noniid_mal30_test_loss_active10}
\begin{tabular}{lccc}
\toprule
$M_{\mathrm{act}}/K$ & GPT / WikiText & BERT / Shakespeare & LLaMA / StackOverflow \\
\midrule
0\% & 7.56 $\pm$ 0.02 & 7.30 $\pm$ 0.06 & 7.73 $\pm$ 0.05 \\
25\% & 7.52 $\pm$ 0.04 & 7.20 $\pm$ 0.09 & 7.38 $\pm$ 0.05 \\
50\% & 7.47 $\pm$ 0.03 & 7.09 $\pm$ 0.07 & 6.81 $\pm$ 0.06 \\
75\% & \textbf{7.43 $\pm$ 0.02} & \textbf{6.97 $\pm$ 0.03} & \textbf{6.41 $\pm$ 0.03} \\
\bottomrule
\end{tabular}
\end{table}

\begin{table}[ht] \centering
\caption{Bank-size threshold ablation under Non-IID partition with $30\%$ malicious clients. Metric: Test PPL. For each seed and threshold, we select the checkpoint with the lowest test loss within the first 10 communication rounds after the bank reaches $M_{\mathrm{act}}$. The best case is highlighted in bold.}
\label{tab:bank_ablation_all_models_noniid_mal30_test_ppl_active10}
\begin{tabular}{lccc}
\toprule
$M_{\mathrm{act}}/K$ & GPT / WikiText & BERT / Shakespeare & LLaMA / StackOverflow \\
\midrule
0\% & 1914.41 $\pm$ 35.81 & 1484.60 $\pm$ 89.71 & 2279.76 $\pm$ 116.94 \\
25\% & 1849.76 $\pm$ 65.54 & 1339.75 $\pm$ 121.30 & 1600.19 $\pm$ 79.79 \\
50\% & 1760.25 $\pm$ 53.04 & 1206.50 $\pm$ 80.07 & 911.93 $\pm$ 52.64 \\
75\% & \textbf{1677.76 $\pm$ 37.72} & \textbf{1063.79 $\pm$ 26.84} & \textbf{605.22 $\pm$ 18.26} \\
\bottomrule
\end{tabular}
\end{table}

\begin{table}[ht] \centering
\caption{Bank-size threshold ablation under Non-IID partition with $30\%$ malicious clients. Metric: Semantic Entropy. For each seed and threshold, we select the checkpoint with the lowest test loss within the first 10 communication rounds after the bank reaches $M_{\mathrm{act}}$. The best case is highlighted in bold.}
\label{tab:bank_ablation_all_models_noniid_mal30_test_sem_ent_active10}
\begin{tabular}{lccc}
\toprule
$M_{\mathrm{act}}/K$ & GPT / WikiText & BERT / Shakespeare & LLaMA / StackOverflow \\
\midrule
0\% & \textbf{6.74 $\pm$ 0.07} & 7.62 $\pm$ 0.11 & 9.96 $\pm$ 0.03 \\
25\% & 6.79 $\pm$ 0.01 & 6.92 $\pm$ 0.14 & 9.68 $\pm$ 0.05 \\
50\% & 6.75 $\pm$ 0.05 & 6.53 $\pm$ 0.10 & 8.72 $\pm$ 0.15 \\
75\% & 6.78 $\pm$ 0.05 & \textbf{6.40 $\pm$ 0.04} & \textbf{7.43 $\pm$ 0.08} \\
\bottomrule
\end{tabular}
\end{table}

\begin{table}[ht] \centering
\caption{Bank-size threshold ablation under Non-IID partition with $40\%$ malicious clients. Metric: Test Loss. For each seed and threshold, we select the checkpoint with the lowest test loss within the first 10 communication rounds after the bank reaches $M_{\mathrm{act}}$. The best case is highlighted in bold.}
\label{tab:bank_ablation_all_models_noniid_mal40_test_loss_active10}
\begin{tabular}{lccc}
\toprule
$M_{\mathrm{act}}/K$ & GPT / WikiText & BERT / Shakespeare & LLaMA / StackOverflow \\
\midrule
0\% & 7.56 $\pm$ 0.03 & 7.32 $\pm$ 0.09 & 8.33 $\pm$ 0.82 \\
25\% & 7.53 $\pm$ 0.03 & 7.21 $\pm$ 0.11 & 7.96 $\pm$ 0.82 \\
50\% & 7.49 $\pm$ 0.04 & 7.13 $\pm$ 0.08 & 7.28 $\pm$ 0.61 \\
75\% & \textbf{7.44 $\pm$ 0.03} & \textbf{6.98 $\pm$ 0.03} & \textbf{6.60 $\pm$ 0.31} \\
\bottomrule
\end{tabular}
\end{table}

\begin{table}[ht] \centering
\caption{Bank-size threshold ablation under Non-IID partition with $40\%$ malicious clients. Metric: Test PPL. For each seed and threshold, we select the checkpoint with the lowest test loss within the first 10 communication rounds after the bank reaches $M_{\mathrm{act}}$. The best case is highlighted in bold.}
\label{tab:bank_ablation_all_models_noniid_mal40_test_ppl_active10}
\begin{tabular}{lccc}
\toprule
$M_{\mathrm{act}}/K$ & GPT / WikiText & BERT / Shakespeare & LLaMA / StackOverflow \\
\midrule
0\% & 1920.35 $\pm$ 56.71 & 1514.80 $\pm$ 139.13 & 5266.58 $\pm$ 4615.45 \\
25\% & 1860.94 $\pm$ 60.73 & 1355.84 $\pm$ 138.76 & 3651.86 $\pm$ 3205.90 \\
50\% & 1783.26 $\pm$ 65.14 & 1248.37 $\pm$ 101.92 & 1649.52 $\pm$ 1092.41 \\
75\% & \textbf{1698.50 $\pm$ 44.56} & \textbf{1078.90 $\pm$ 33.23} & \textbf{759.15 $\pm$ 249.94} \\
\bottomrule
\end{tabular}
\end{table}

\begin{table}[ht] \centering
\caption{Bank-size threshold ablation under Non-IID partition with $40\%$ malicious clients. Metric: Semantic Entropy. For each seed and threshold, we select the checkpoint with the lowest test loss within the first 10 communication rounds after the bank reaches $M_{\mathrm{act}}$. The best case is highlighted in bold.}
\label{tab:bank_ablation_all_models_noniid_mal40_test_sem_ent_active10}
\begin{tabular}{lccc}
\toprule
$M_{\mathrm{act}}/K$ & GPT / WikiText & BERT / Shakespeare & LLaMA / StackOverflow \\
\midrule
0\% & 6.80 $\pm$ 0.12 & 7.85 $\pm$ 0.34 & 10.12 $\pm$ 0.17 \\
25\% & 6.81 $\pm$ 0.10 & 6.90 $\pm$ 0.19 & 9.96 $\pm$ 0.31 \\
50\% & \textbf{6.73 $\pm$ 0.04} & 6.62 $\pm$ 0.15 & 9.36 $\pm$ 0.66 \\
75\% & 6.76 $\pm$ 0.06 & \textbf{6.29 $\pm$ 0.18} & \textbf{8.00 $\pm$ 0.97} \\
\bottomrule
\end{tabular}
\end{table}

\subsection{Interpretation}

The study characterizes $M_{\mathrm{act}}$ as the bank coverage at which screening and global aggregation begin, rather than as a desired final bank size. The bank continues to grow after activation. Across the evaluated settings, the dominant trend favors broader activation coverage: perplexity improves consistently as coverage increases, and token-level semantic entropy generally follows the same pattern. This supports maximizing the bank coverage available at activation whenever the resulting delay is acceptable.

We use $50\%$ population coverage as the default practical activation point, not because it is statistically preferable to broader coverage, but because it begins aggregation earlier than the largest evaluated threshold of $75\%$. In deployments with intermittent or nonuniform participation, waiting for all client identities may be impractically slow. A smaller activation bank can therefore provide an operational starting point, while the probability analysis in Appendix~\ref{app:min_bank_probability} quantifies its benign-majority reliability under uniform random participation. The ablation does not evaluate $100\%$ activation coverage and therefore does not establish full coverage as an empirical optimum.


\clearpage
\section{GMM-Selection Ablation}
\label{app:gmm_ablation}

This appendix reports the ablation study for the GMM selection rule in \FedLNS{}. 
The default \FedLNS{} rule fits both a one-component and a two-component diagonal Gaussian mixture model to the current bank-standardized normalization-signature features, computes the Bayesian Information Criterion (BIC) for each model, and selects the lower-BIC model. 
This adaptive rule is compared with a forced two-component variant.

\paragraph{BIC-selected 1-vs-2 GMM.}
This is the default \FedLNS{} rule. 
It selects
\begin{equation}
    \widehat{c}_t
    =
    \arg\min_{c\in\{1,2\}}
    \mathrm{BIC}(c).
\end{equation}
When $\widehat{c}_t=1$, all clients are retained. 
When $\widehat{c}_t=2$, \FedLNS{} retains the component with the smaller median standardized deviation from the bank reference.

\paragraph{Forced 2-GMM.}
This variant always fits a two-component GMM after bank activation and retains the component with the smaller median absolute standardized deviation from the bank reference. 
It is more aggressive than the default method because it always tries to split the current round into two groups, even when the current client signatures may be better explained by a single population.

All GMM ablation runs use the same model, dataset, partition, attack, optimizer, seeds, and bank threshold as the corresponding main experiment. 
The only changed variable is the GMM decision rule. 
For each seed, we select the checkpoint with the lowest test loss and report test loss, perplexity, and token-level semantic entropy from that same checkpoint. Thus, token-level semantic entropy is not independently optimized; it is measured at the checkpoint selected by test loss.
All entries are reported as mean $\pm$ sample standard deviation over seeds $\{100,200,300\}$.

\subsection{Results Under IID Partitions}

Tables~\ref{tab:gmm_ablation_all_models_iid_mal0_test_loss}--\ref{tab:gmm_ablation_all_models_iid_mal40_test_sem_ent} report the IID GMM-selection ablation results across all malicious-client fractions. 
The BIC-selected and forced two-component rules produce nearly identical selected evaluation-checkpoint values across the three model families. 
At the strongest $40\%$ malicious-client setting, both rules obtain the same values up to the reported precision for BERT/Tiny Shakespeare and LLaMA/StackOverflow. 
For GPT/WikiText, the two rules are also essentially identical: the BIC-selected rule gives test loss $7.16\pm0.00$, perplexity $1281.63\pm1.74$, and token-level semantic entropy $6.80\pm0.06$, while forced 2-GMM gives test loss $7.16\pm0.00$, perplexity $1281.62\pm1.73$, and token-level semantic entropy $6.80\pm0.06$.

In the IID results, the default BIC-selected rule and the forced two-component rule have nearly identical endpoint values. The ablation does not establish that the current-round feature distribution is statistically identifiable as two components in every attacked setting.

\begin{table}[ht] \centering
\caption{GMM-selection ablation under IID partition with $0\%$ malicious clients. Metric: Test Loss. Values are mean $\pm$ standard deviation over three seeds. For each seed, all metrics are taken from the checkpoint with the lowest test loss. The best case is highlighted in bold.}
\label{tab:gmm_ablation_all_models_iid_mal0_test_loss}
\begin{tabular}{lccc}
\toprule
GMM rule & GPT / WikiText & BERT / Shakespeare & LLaMA / StackOverflow \\
\midrule
BIC 1/2-GMM & \textbf{7.13 $\pm$ 0.00} & \textbf{6.50 $\pm$ 0.05} & \textbf{5.36 $\pm$ 0.03} \\
Forced 2-GMM & \textbf{7.13 $\pm$ 0.00} & \textbf{6.50 $\pm$ 0.05} & \textbf{5.36 $\pm$ 0.03} \\
\bottomrule
\end{tabular}
\end{table}

\begin{table}[ht]\centering
\caption{GMM-selection ablation under IID partition with $0\%$ malicious clients. Metric: Test PPL. Values are mean $\pm$ standard deviation over three seeds. For each seed, all metrics are taken from the checkpoint with the lowest test loss. The best case is highlighted in bold.}
\label{tab:gmm_ablation_all_models_iid_mal0_test_ppl}
\begin{tabular}{lccc}
\toprule
GMM rule & GPT / WikiText & BERT / Shakespeare & LLaMA / StackOverflow \\
\midrule
BIC 1/2-GMM & \textbf{1246.57 $\pm$ 2.52} & \textbf{663.68 $\pm$ 29.76} & \textbf{213.42 $\pm$ 6.57} \\
Forced 2-GMM & \textbf{1246.57 $\pm$ 2.52} & \textbf{663.68 $\pm$ 29.76} & \textbf{213.42 $\pm$ 6.57} \\
\bottomrule
\end{tabular}
\end{table}

\begin{table}[ht]\centering
\caption{GMM-selection ablation under IID partition with $0\%$ malicious clients. Metric: Semantic Entropy. Values are mean $\pm$ standard deviation over three seeds. For each seed, all metrics are taken from the checkpoint with the lowest test loss. The best case is highlighted in bold.}
\label{tab:gmm_ablation_all_models_iid_mal0_test_sem_ent}
\begin{tabular}{lccc}
\toprule
GMM rule & GPT / WikiText & BERT / Shakespeare & LLaMA / StackOverflow \\
\midrule
BIC 1/2-GMM & \textbf{6.73 $\pm$ 0.03} & \textbf{6.28 $\pm$ 0.09} & \textbf{5.52 $\pm$ 0.02} \\
Forced 2-GMM & \textbf{6.73 $\pm$ 0.03} & \textbf{6.28 $\pm$ 0.09} & \textbf{5.52 $\pm$ 0.02} \\
\bottomrule
\end{tabular}
\end{table}

\begin{table}[ht]\centering
\caption{GMM-selection ablation under IID partition with $10\%$ malicious clients. Metric: Test Loss. Values are mean $\pm$ standard deviation over three seeds. For each seed, all metrics are taken from the checkpoint with the lowest test loss. The best case is highlighted in bold.}
\label{tab:gmm_ablation_all_models_iid_mal10_test_loss}
\begin{tabular}{lccc}
\toprule
GMM rule & GPT / WikiText & BERT / Shakespeare & LLaMA / StackOverflow \\
\midrule
BIC 1/2-GMM & \textbf{7.13 $\pm$ 0.00} & \textbf{6.49 $\pm$ 0.02} & \textbf{5.37 $\pm$ 0.03} \\
Forced 2-GMM & \textbf{7.13 $\pm$ 0.00} & \textbf{6.49 $\pm$ 0.02} & \textbf{5.37 $\pm$ 0.03} \\
\bottomrule
\end{tabular}
\end{table}

\begin{table}[ht]\centering
\caption{GMM-selection ablation under IID partition with $10\%$ malicious clients. Metric: Test PPL. Values are mean $\pm$ standard deviation over three seeds. For each seed, all metrics are taken from the checkpoint with the lowest test loss. The best case is highlighted in bold.}
\label{tab:gmm_ablation_all_models_iid_mal10_test_ppl}
\begin{tabular}{lccc}
\toprule
GMM rule & GPT / WikiText & BERT / Shakespeare & LLaMA / StackOverflow \\
\midrule
BIC 1/2-GMM & \textbf{1247.28 $\pm$ 5.97} & \textbf{655.97 $\pm$ 11.66} & \textbf{213.93 $\pm$ 6.56} \\
Forced 2-GMM & \textbf{1247.28 $\pm$ 5.97} & \textbf{655.97 $\pm$ 11.66} & \textbf{213.93 $\pm$ 6.56} \\
\bottomrule
\end{tabular}
\end{table}

\begin{table}[ht]\centering
\caption{GMM-selection ablation under IID partition with $10\%$ malicious clients. Metric: Semantic Entropy. Values are mean $\pm$ standard deviation over three seeds. For each seed, all metrics are taken from the checkpoint with the lowest test loss. The best case is highlighted in bold.}
\label{tab:gmm_ablation_all_models_iid_mal10_test_sem_ent}
\begin{tabular}{lccc}
\toprule
GMM rule & GPT / WikiText & BERT / Shakespeare & LLaMA / StackOverflow \\
\midrule
BIC 1/2-GMM & \textbf{6.81 $\pm$ 0.01} & \textbf{6.38 $\pm$ 0.16} & \textbf{5.52 $\pm$ 0.04} \\
Forced 2-GMM & \textbf{6.81 $\pm$ 0.01} & \textbf{6.38 $\pm$ 0.16} & \textbf{5.52 $\pm$ 0.04} \\
\bottomrule
\end{tabular}
\end{table}

\begin{table}[ht]\centering
\caption{GMM-selection ablation under IID partition with $20\%$ malicious clients. Metric: Test Loss. Values are mean $\pm$ standard deviation over three seeds. For each seed, all metrics are taken from the checkpoint with the lowest test loss. The best case is highlighted in bold.}
\label{tab:gmm_ablation_all_models_iid_mal20_test_loss}
\begin{tabular}{lccc}
\toprule
GMM rule & GPT / WikiText & BERT / Shakespeare & LLaMA / StackOverflow \\
\midrule
BIC 1/2-GMM & \textbf{7.13 $\pm$ 0.00} & \textbf{6.49 $\pm$ 0.02} & \textbf{5.37 $\pm$ 0.03} \\
Forced 2-GMM & \textbf{7.13 $\pm$ 0.00} & \textbf{6.49 $\pm$ 0.02} & \textbf{5.37 $\pm$ 0.03} \\
\bottomrule
\end{tabular}
\end{table}

\begin{table}[ht]\centering
\caption{GMM-selection ablation under IID partition with $20\%$ malicious clients. Metric: Test PPL. Values are mean $\pm$ standard deviation over three seeds. For each seed, all metrics are taken from the checkpoint with the lowest test loss. The best case is highlighted in bold.}
\label{tab:gmm_ablation_all_models_iid_mal20_test_ppl}
\begin{tabular}{lccc}
\toprule
GMM rule & GPT / WikiText & BERT / Shakespeare & LLaMA / StackOverflow \\
\midrule
BIC 1/2-GMM & \textbf{1249.11 $\pm$ 1.78} & \textbf{661.38 $\pm$ 11.94} & \textbf{214.57 $\pm$ 6.75} \\
Forced 2-GMM & \textbf{1249.11 $\pm$ 1.78} & \textbf{661.38 $\pm$ 11.94} & \textbf{214.57 $\pm$ 6.75} \\
\bottomrule
\end{tabular}
\end{table}

\begin{table}[ht]\centering
\caption{GMM-selection ablation under IID partition with $20\%$ malicious clients. Metric: Semantic Entropy. Values are mean $\pm$ standard deviation over three seeds. For each seed, all metrics are taken from the checkpoint with the lowest test loss. The best case is highlighted in bold.}
\label{tab:gmm_ablation_all_models_iid_mal20_test_sem_ent}
\begin{tabular}{lccc}
\toprule
GMM rule & GPT / WikiText & BERT / Shakespeare & LLaMA / StackOverflow \\
\midrule
BIC 1/2-GMM & \textbf{6.75 $\pm$ 0.02} & \textbf{6.19 $\pm$ 0.09} & \textbf{5.54 $\pm$ 0.03} \\
Forced 2-GMM & \textbf{6.75 $\pm$ 0.02} & \textbf{6.19 $\pm$ 0.09} & \textbf{5.54 $\pm$ 0.03} \\
\bottomrule
\end{tabular}
\end{table}

\begin{table}[ht]\centering
\caption{GMM-selection ablation under IID partition with $30\%$ malicious clients. Metric: Test Loss. Values are mean $\pm$ standard deviation over three seeds. For each seed, all metrics are taken from the checkpoint with the lowest test loss. The best case is highlighted in bold.}
\label{tab:gmm_ablation_all_models_iid_mal30_test_loss}
\begin{tabular}{lccc}
\toprule
GMM rule & GPT / WikiText & BERT / Shakespeare & LLaMA / StackOverflow \\
\midrule
BIC 1/2-GMM & \textbf{7.14 $\pm$ 0.00} & \textbf{6.51 $\pm$ 0.01} & \textbf{5.37 $\pm$ 0.03} \\
Forced 2-GMM & \textbf{7.14 $\pm$ 0.00} & \textbf{6.51 $\pm$ 0.01} & \textbf{5.37 $\pm$ 0.03} \\
\bottomrule
\end{tabular}
\end{table}

\begin{table}[ht]\centering
\caption{GMM-selection ablation under IID partition with $30\%$ malicious clients. Metric: Test PPL. Values are mean $\pm$ standard deviation over three seeds. For each seed, all metrics are taken from the checkpoint with the lowest test loss. The best case is highlighted in bold.}
\label{tab:gmm_ablation_all_models_iid_mal30_test_ppl}
\begin{tabular}{lccc}
\toprule
GMM rule & GPT / WikiText & BERT / Shakespeare & LLaMA / StackOverflow \\
\midrule
BIC 1/2-GMM & \textbf{1258.21 $\pm$ 4.70} & \textbf{669.37 $\pm$ 5.86} & 215.61 $\pm$ 6.70 \\
Forced 2-GMM & \textbf{1258.21 $\pm$ 4.70} & \textbf{669.37 $\pm$ 5.86} & \textbf{215.61 $\pm$ 6.70} \\
\bottomrule
\end{tabular}
\end{table}

\begin{table}[ht]\centering
\caption{GMM-selection ablation under IID partition with $30\%$ malicious clients. Metric: Semantic Entropy. Values are mean $\pm$ standard deviation over three seeds. For each seed, all metrics are taken from the checkpoint with the lowest test loss. The best case is highlighted in bold.}
\label{tab:gmm_ablation_all_models_iid_mal30_test_sem_ent}
\begin{tabular}{lccc}
\toprule
GMM rule & GPT / WikiText & BERT / Shakespeare & LLaMA / StackOverflow \\
\midrule
BIC 1/2-GMM & \textbf{6.74 $\pm$ 0.04} & \textbf{6.24 $\pm$ 0.10} & \textbf{5.52 $\pm$ 0.04} \\
Forced 2-GMM & \textbf{6.74 $\pm$ 0.04} & \textbf{6.24 $\pm$ 0.10} & \textbf{5.52 $\pm$ 0.04} \\
\bottomrule
\end{tabular}
\end{table}

\begin{table}[ht]\centering
\caption{GMM-selection ablation under IID partition with $40\%$ malicious clients. Metric: Test Loss. Values are mean $\pm$ standard deviation over three seeds. For each seed, all metrics are taken from the checkpoint with the lowest test loss. The best case is highlighted in bold.}
\label{tab:gmm_ablation_all_models_iid_mal40_test_loss}
\begin{tabular}{lccc}
\toprule
GMM rule & GPT / WikiText & BERT / Shakespeare & LLaMA / StackOverflow \\
\midrule
BIC 1/2-GMM & 7.16 $\pm$ 0.00 & \textbf{6.55 $\pm$ 0.02} & \textbf{5.38 $\pm$ 0.03} \\
Forced 2-GMM & \textbf{7.16 $\pm$ 0.00} & \textbf{6.55 $\pm$ 0.02} & \textbf{5.38 $\pm$ 0.03} \\
\bottomrule
\end{tabular}
\end{table}

\begin{table}[ht]\centering
\caption{GMM-selection ablation under IID partition with $40\%$ malicious clients. Metric: Test PPL. Values are mean $\pm$ standard deviation over three seeds. For each seed, all metrics are taken from the checkpoint with the lowest test loss. The best case is highlighted in bold.}
\label{tab:gmm_ablation_all_models_iid_mal40_test_ppl}
\begin{tabular}{lccc}
\toprule
GMM rule & GPT / WikiText & BERT / Shakespeare & LLaMA / StackOverflow \\
\midrule
BIC 1/2-GMM & 1281.63 $\pm$ 1.74 & \textbf{697.62 $\pm$ 13.62} & \textbf{217.14 $\pm$ 6.85} \\
Forced 2-GMM & \textbf{1281.62 $\pm$ 1.73} & \textbf{697.62 $\pm$ 13.62} & \textbf{217.14 $\pm$ 6.85} \\
\bottomrule
\end{tabular}
\end{table}

\begin{table}[ht]\centering
\caption{GMM-selection ablation under IID partition with $40\%$ malicious clients. Metric: Semantic Entropy. Values are mean $\pm$ standard deviation over three seeds. For each seed, all metrics are taken from the checkpoint with the lowest test loss. The best case is highlighted in bold.}
\label{tab:gmm_ablation_all_models_iid_mal40_test_sem_ent}
\begin{tabular}{lccc}
\toprule
GMM rule & GPT / WikiText & BERT / Shakespeare & LLaMA / StackOverflow \\
\midrule
BIC 1/2-GMM & \textbf{6.80 $\pm$ 0.06} & \textbf{6.24 $\pm$ 0.06} & \textbf{5.52 $\pm$ 0.02} \\
Forced 2-GMM & 6.80 $\pm$ 0.06 & \textbf{6.24 $\pm$ 0.06} & \textbf{5.52 $\pm$ 0.02} \\
\bottomrule
\end{tabular}
\end{table}

\clearpage
\subsection{Results Under Non-IID Partitions}

Tables~\ref{tab:gmm_ablation_all_models_noniid_mal0_test_loss}--\ref{tab:gmm_ablation_all_models_noniid_mal40_test_sem_ent} report the corresponding non-IID results. 
The non-IID results follow the same pattern: BIC-selected one-vs-two GMM and forced two-component GMM have nearly indistinguishable endpoint values in most settings. 
At the strongest $40\%$ malicious-client setting, the two methods produce identical reported values for all three model families. 
For GPT/WikiText, both rules obtain test loss $7.16\pm0.01$, perplexity $1286.50\pm15.34$, and token-level semantic entropy $6.92\pm0.24$. 
For BERT/Tiny Shakespeare, both obtain test loss $6.48\pm0.13$, perplexity $658.21\pm91.66$, and token-level semantic entropy $6.19\pm0.13$. 
For LLaMA/StackOverflow, both obtain test loss $5.39\pm0.03$, perplexity $220.18\pm6.45$, and token-level semantic entropy $5.52\pm0.03$.

The main visible difference appears only in a small number of intermediate settings. 
For example, under GPT/WikiText with non-IID partitioning and $30\%$ malicious clients, BIC-selected GMM gives perplexity $1277.28\pm23.52$, while forced 2-GMM gives $1277.81\pm22.93$. 
For token-level semantic entropy in the same setting, forced 2-GMM gives $6.87\pm0.06$, while BIC-selected GMM gives $6.90\pm0.02$. 
These differences are small and do not indicate a systematic advantage of forced splitting.

\begin{table}[ht]\centering
\caption{GMM-selection ablation under Non-IID partition with $0\%$ malicious clients. Metric: Test Loss. Values are mean $\pm$ standard deviation over three seeds. For each seed, all metrics are taken from the checkpoint with the lowest test loss. The best case is highlighted in bold.}
\label{tab:gmm_ablation_all_models_noniid_mal0_test_loss}
\begin{tabular}{lccc}
\toprule
GMM rule & GPT / WikiText & BERT / Shakespeare & LLaMA / StackOverflow \\
\midrule
BIC 1/2-GMM & 7.11 $\pm$ 0.01 & \textbf{6.47 $\pm$ 0.05} & \textbf{5.37 $\pm$ 0.03} \\
Forced 2-GMM & \textbf{7.11 $\pm$ 0.01} & \textbf{6.47 $\pm$ 0.05} & \textbf{5.37 $\pm$ 0.03} \\
\bottomrule
\end{tabular}
\end{table}

\begin{table}[ht]\centering
\caption{GMM-selection ablation under Non-IID partition with $0\%$ malicious clients. Metric: Test PPL. Values are mean $\pm$ standard deviation over three seeds. For each seed, all metrics are taken from the checkpoint with the lowest test loss. The best case is highlighted in bold.}
\label{tab:gmm_ablation_all_models_noniid_mal0_test_ppl}
\begin{tabular}{lccc}
\toprule
GMM rule & GPT / WikiText & BERT / Shakespeare & LLaMA / StackOverflow \\
\midrule
BIC 1/2-GMM & 1225.90 $\pm$ 12.06 & \textbf{643.56 $\pm$ 33.46} & \textbf{215.55 $\pm$ 6.85} \\
Forced 2-GMM & \textbf{1225.89 $\pm$ 12.06} & \textbf{643.56 $\pm$ 33.46} & \textbf{215.55 $\pm$ 6.85} \\
\bottomrule
\end{tabular}
\end{table}

\begin{table}[ht]\centering
\caption{GMM-selection ablation under Non-IID partition with $0\%$ malicious clients. Metric: Semantic Entropy. Values are mean $\pm$ standard deviation over three seeds. For each seed, all metrics are taken from the checkpoint with the lowest test loss. The best case is highlighted in bold.}
\label{tab:gmm_ablation_all_models_noniid_mal0_test_sem_ent}
\begin{tabular}{lccc}
\toprule
GMM rule & GPT / WikiText & BERT / Shakespeare & LLaMA / StackOverflow \\
\midrule
BIC 1/2-GMM & \textbf{6.80 $\pm$ 0.09} & \textbf{6.12 $\pm$ 0.05} & \textbf{5.54 $\pm$ 0.03} \\
Forced 2-GMM & 6.80 $\pm$ 0.09 & \textbf{6.12 $\pm$ 0.05} & \textbf{5.54 $\pm$ 0.03} \\
\bottomrule
\end{tabular}
\end{table}

\begin{table}[ht]\centering
\caption{GMM-selection ablation under Non-IID partition with $10\%$ malicious clients. Metric: Test Loss. Values are mean $\pm$ standard deviation over three seeds. For each seed, all metrics are taken from the checkpoint with the lowest test loss. The best case is highlighted in bold.}
\label{tab:gmm_ablation_all_models_noniid_mal10_test_loss}
\begin{tabular}{lccc}
\toprule
GMM rule & GPT / WikiText & BERT / Shakespeare & LLaMA / StackOverflow \\
\midrule
BIC 1/2-GMM & \textbf{7.12 $\pm$ 0.00} & \textbf{6.36 $\pm$ 0.06} & \textbf{5.38 $\pm$ 0.03} \\
Forced 2-GMM & \textbf{7.12 $\pm$ 0.00} & \textbf{6.36 $\pm$ 0.06} & \textbf{5.38 $\pm$ 0.03} \\
\bottomrule
\end{tabular}
\end{table}

\begin{table}[ht]\centering
\caption{GMM-selection ablation under Non-IID partition with $10\%$ malicious clients. Metric: Test PPL. Values are mean $\pm$ standard deviation over three seeds. For each seed, all metrics are taken from the checkpoint with the lowest test loss. The best case is highlighted in bold.}
\label{tab:gmm_ablation_all_models_noniid_mal10_test_ppl}
\begin{tabular}{lccc}
\toprule
GMM rule & GPT / WikiText & BERT / Shakespeare & LLaMA / StackOverflow \\
\midrule
BIC 1/2-GMM & \textbf{1238.52 $\pm$ 4.97} & \textbf{576.63 $\pm$ 36.58} & \textbf{216.13 $\pm$ 6.70} \\
Forced 2-GMM & \textbf{1238.52 $\pm$ 4.97} & \textbf{576.63 $\pm$ 36.58} & \textbf{216.13 $\pm$ 6.70} \\
\bottomrule
\end{tabular}
\end{table}

\begin{table}[ht]\centering
\caption{GMM-selection ablation under Non-IID partition with $10\%$ malicious clients. Metric: Semantic Entropy. Values are mean $\pm$ standard deviation over three seeds. For each seed, all metrics are taken from the checkpoint with the lowest test loss. The best case is highlighted in bold.}
\label{tab:gmm_ablation_all_models_noniid_mal10_test_sem_ent}
\begin{tabular}{lccc}
\toprule
GMM rule & GPT / WikiText & BERT / Shakespeare & LLaMA / StackOverflow \\
\midrule
BIC 1/2-GMM & \textbf{6.79 $\pm$ 0.05} & \textbf{6.06 $\pm$ 0.18} & \textbf{5.54 $\pm$ 0.03} \\
Forced 2-GMM & \textbf{6.79 $\pm$ 0.05} & \textbf{6.06 $\pm$ 0.18} & \textbf{5.54 $\pm$ 0.03} \\
\bottomrule
\end{tabular}
\end{table}

\begin{table}[ht]\centering
\caption{GMM-selection ablation under Non-IID partition with $20\%$ malicious clients. Metric: Test Loss. Values are mean $\pm$ standard deviation over three seeds. For each seed, all metrics are taken from the checkpoint with the lowest test loss. The best case is highlighted in bold.}
\label{tab:gmm_ablation_all_models_noniid_mal20_test_loss}
\begin{tabular}{lccc}
\toprule
GMM rule & GPT / WikiText & BERT / Shakespeare & LLaMA / StackOverflow \\
\midrule
BIC 1/2-GMM & \textbf{7.13 $\pm$ 0.01} & \textbf{6.42 $\pm$ 0.14} & \textbf{5.38 $\pm$ 0.03} \\
Forced 2-GMM & \textbf{7.13 $\pm$ 0.01} & \textbf{6.42 $\pm$ 0.14} & \textbf{5.38 $\pm$ 0.03} \\
\bottomrule
\end{tabular}
\end{table}

\begin{table}[ht]\centering
\caption{GMM-selection ablation under Non-IID partition with $20\%$ malicious clients. Metric: Test PPL. Values are mean $\pm$ standard deviation over three seeds. For each seed, all metrics are taken from the checkpoint with the lowest test loss. The best case is highlighted in bold.}
\label{tab:gmm_ablation_all_models_noniid_mal20_test_ppl}
\begin{tabular}{lccc}
\toprule
GMM rule & GPT / WikiText & BERT / Shakespeare & LLaMA / StackOverflow \\
\midrule
BIC 1/2-GMM & \textbf{1251.65 $\pm$ 11.37} & \textbf{618.30 $\pm$ 91.12} & \textbf{216.97 $\pm$ 6.91} \\
Forced 2-GMM & \textbf{1251.65 $\pm$ 11.37} & \textbf{618.30 $\pm$ 91.12} & \textbf{216.97 $\pm$ 6.91} \\
\bottomrule
\end{tabular}
\end{table}

\begin{table}[ht]\centering
\caption{GMM-selection ablation under Non-IID partition with $20\%$ malicious clients. Metric: Semantic Entropy. Values are mean $\pm$ standard deviation over three seeds. For each seed, all metrics are taken from the checkpoint with the lowest test loss. The best case is highlighted in bold.}
\label{tab:gmm_ablation_all_models_noniid_mal20_test_sem_ent}
\begin{tabular}{lccc}
\toprule
GMM rule & GPT / WikiText & BERT / Shakespeare & LLaMA / StackOverflow \\
\midrule
BIC 1/2-GMM & \textbf{7.04 $\pm$ 0.40} & \textbf{6.33 $\pm$ 0.09} & \textbf{5.53 $\pm$ 0.03} \\
Forced 2-GMM & \textbf{7.04 $\pm$ 0.40} & \textbf{6.33 $\pm$ 0.09} & \textbf{5.53 $\pm$ 0.03} \\
\bottomrule
\end{tabular}
\end{table}

\begin{table}[ht]\centering
\caption{GMM-selection ablation under Non-IID partition with $30\%$ malicious clients. Metric: Test Loss. Values are mean $\pm$ standard deviation over three seeds. For each seed, all metrics are taken from the checkpoint with the lowest test loss. The best case is highlighted in bold.}
\label{tab:gmm_ablation_all_models_noniid_mal30_test_loss}
\begin{tabular}{lccc}
\toprule
GMM rule & GPT / WikiText & BERT / Shakespeare & LLaMA / StackOverflow \\
\midrule
BIC 1/2-GMM & \textbf{7.15 $\pm$ 0.02} & \textbf{6.38 $\pm$ 0.18} & \textbf{5.38 $\pm$ 0.03} \\
Forced 2-GMM & 7.15 $\pm$ 0.02 & 6.38 $\pm$ 0.18 & \textbf{5.38 $\pm$ 0.03} \\
\bottomrule
\end{tabular}
\end{table}

\begin{table}[ht]\centering
\caption{GMM-selection ablation under Non-IID partition with $30\%$ malicious clients. Metric: Test PPL. Values are mean $\pm$ standard deviation over three seeds. For each seed, all metrics are taken from the checkpoint with the lowest test loss. The best case is highlighted in bold.}
\label{tab:gmm_ablation_all_models_noniid_mal30_test_ppl}
\begin{tabular}{lccc}
\toprule
GMM rule & GPT / WikiText & BERT / Shakespeare & LLaMA / StackOverflow \\
\midrule
BIC 1/2-GMM & \textbf{1277.28 $\pm$ 23.52} & \textbf{594.61 $\pm$ 108.75} & \textbf{218.02 $\pm$ 7.07} \\
Forced 2-GMM & 1277.81 $\pm$ 22.93 & 594.63 $\pm$ 108.74 & \textbf{218.02 $\pm$ 7.07} \\
\bottomrule
\end{tabular}
\end{table}

\begin{table}[ht]\centering
\caption{GMM-selection ablation under Non-IID partition with $30\%$ malicious clients. Metric: Semantic Entropy. Values are mean $\pm$ standard deviation over three seeds. For each seed, all metrics are taken from the checkpoint with the lowest test loss. The best case is highlighted in bold.}
\label{tab:gmm_ablation_all_models_noniid_mal30_test_sem_ent}
\begin{tabular}{lccc}
\toprule
GMM rule & GPT / WikiText & BERT / Shakespeare & LLaMA / StackOverflow \\
\midrule
BIC 1/2-GMM & 6.90 $\pm$ 0.02 & 6.00 $\pm$ 0.38 & \textbf{5.53 $\pm$ 0.02} \\
Forced 2-GMM & \textbf{6.87 $\pm$ 0.06} & \textbf{6.00 $\pm$ 0.38} & \textbf{5.53 $\pm$ 0.02} \\
\bottomrule
\end{tabular}
\end{table}

\begin{table}[ht]\centering
\caption{GMM-selection ablation under Non-IID partition with $40\%$ malicious clients. Metric: Test Loss. Values are mean $\pm$ standard deviation over three seeds. For each seed, all metrics are taken from the checkpoint with the lowest test loss. The best case is highlighted in bold.}
\label{tab:gmm_ablation_all_models_noniid_mal40_test_loss}
\begin{tabular}{lccc}
\toprule
GMM rule & GPT / WikiText & BERT / Shakespeare & LLaMA / StackOverflow \\
\midrule
BIC 1/2-GMM & \textbf{7.16 $\pm$ 0.01} & \textbf{6.48 $\pm$ 0.13} & \textbf{5.39 $\pm$ 0.03} \\
Forced 2-GMM & \textbf{7.16 $\pm$ 0.01} & \textbf{6.48 $\pm$ 0.13} & \textbf{5.39 $\pm$ 0.03} \\
\bottomrule
\end{tabular}
\end{table}

\begin{table}[ht]\centering
\caption{GMM-selection ablation under Non-IID partition with $40\%$ malicious clients. Metric: Test PPL. Values are mean $\pm$ standard deviation over three seeds. For each seed, all metrics are taken from the checkpoint with the lowest test loss. The best case is highlighted in bold.}
\label{tab:gmm_ablation_all_models_noniid_mal40_test_ppl}
\begin{tabular}{lccc}
\toprule
GMM rule & GPT / WikiText & BERT / Shakespeare & LLaMA / StackOverflow \\
\midrule
BIC 1/2-GMM & \textbf{1286.50 $\pm$ 15.34} & \textbf{658.21 $\pm$ 91.66} & \textbf{220.18 $\pm$ 6.45} \\
Forced 2-GMM & \textbf{1286.50 $\pm$ 15.34} & \textbf{658.21 $\pm$ 91.66} & \textbf{220.18 $\pm$ 6.45} \\
\bottomrule
\end{tabular}
\end{table}

\begin{table}[ht]\centering
\caption{GMM-selection ablation under Non-IID partition with $40\%$ malicious clients. Metric: Semantic Entropy. Values are mean $\pm$ standard deviation over three seeds. For each seed, all metrics are taken from the checkpoint with the lowest test loss. The best case is highlighted in bold.}
\label{tab:gmm_ablation_all_models_noniid_mal40_test_sem_ent}
\begin{tabular}{lccc}
\toprule
GMM rule & GPT / WikiText & BERT / Shakespeare & LLaMA / StackOverflow \\
\midrule
BIC 1/2-GMM & \textbf{6.92 $\pm$ 0.24} & \textbf{6.19 $\pm$ 0.13} & \textbf{5.52 $\pm$ 0.03} \\
Forced 2-GMM & \textbf{6.92 $\pm$ 0.24} & \textbf{6.19 $\pm$ 0.13} & \textbf{5.52 $\pm$ 0.03} \\
\bottomrule
\end{tabular}
\end{table}

\subsection{Interpretation}

The BIC-selected and forced two-component rules yield nearly identical endpoint values in the evaluated runs. The forced rule always partitions the current participants, whereas BIC preserves the option to retain all clients when the one-component model receives the lower criterion value. This ablation therefore characterizes endpoint sensitivity to the model-selection rule; it does not establish statistical identifiability of the fitted mixtures.

\end{document}